\documentclass{article}

\usepackage[nonatbib,final]{neurips_2025}

\usepackage[final]{microtype}
\usepackage[utf8]{inputenc} %
\usepackage[T1]{fontenc}    %
\usepackage{url}            %
\usepackage{booktabs}       %
\usepackage{amsfonts}       %
\usepackage{nicefrac}       %
\usepackage{xcolor}         %
\usepackage{graphicx}
\usepackage{subcaption}
\usepackage{titletoc}

\contentsmargin{2.55em}
\dottedcontents{lsection}[3.8em]{\small \itshape}{2.3em}{1pc}
\dottedcontents{lsubsection}[6.1em]{\small \itshape}{2.3em}{1pc}
\usepackage{pgfplots}
\usepackage[edges]{forest}
\pgfplotsset{compat=newest}
\usetikzlibrary{trees,positioning,calc}

\usepackage[backref=true,backend=biber,style=authoryear-comp,sortcites,sorting=ynt,maxbibnames=99,maxcitenames=3,uniquelist=true,uniquename=init,giveninits]{biblatex}
\DeclareNameAlias{sortname}{family-given}

\usepackage{algorithm}
\usepackage{algpseudocode}

\usepackage[textsize=tiny]{todonotes} %

\usepackage{amsmath}
\usepackage{amsthm}
\usepackage{amssymb}
\usepackage{mathtools}
\usepackage{multirow}
\usepackage{tabularx}
\usepackage{enumitem}
\usepackage{bm}

\usepackage{pifont}%
\newcommand{\cmark}{\ding{51}}%
\newcommand{\xmark}{\ding{55}}%

\newcommand*{\R}{\mathbb{R}}

\newcommand*{\C}{\mathcal{C}}

\newcommand*{\bfx}{{\bm{x}}}
\newcommand*{\bfX}{{\bm{X}}}

\newcommand*{\bfy}{{\bm{y}}}

\newcommand*{\bfz}{{\bm{z}}}

\newcommand*{\bfa}{{\bm{a}}}
\newcommand*{\bfA}{{\bm{A}}}
\newcommand*{\bfu}{{\bm{u}}}

\newcommand*{\bsf}{{\bm{f}}}

\newcommand*{\rmd}{\mathrm{d}}

\newcommand*{\var}{\mathrm{Var}}
\newcommand*{\ex}{\mathbb{E}}

\DeclareFontFamily{U}{mathx}{}
\DeclareFontShape{U}{mathx}{m}{n}{<-> mathx10}{}
\DeclareSymbolFont{mathx}{U}{mathx}{m}{n}
\DeclareMathAccent{\widecheck}{0}{mathx}{"71}

\usepackage{thmtools}
\declaretheorem[name=Theorem,parent=section]{theorem}

\declaretheorem[name=Proposition,sibling=theorem]{proposition}
\declaretheorem[name=Lemma,sibling=theorem]{lemma}
\declaretheorem[name=Corollary,parent=theorem]{corollary}
\declaretheorem[name=Assumption,parent=section]{assumption}
\declaretheorem[name=Definition,parent=section,style=definition]{definition}

\declaretheorem[name=Remark,sibling=definition,style=definition]{remark}
\declaretheorem[name={Key insight},style=definition]{insight}

\definecolor{g-red}{HTML}{cc241d}
\definecolor{g-green}{HTML}{98971a}
\definecolor{g-yellow}{HTML}{d79921}
\definecolor{g-blue}{HTML}{458588}
\definecolor{g-purple}{HTML}{b16286}
\definecolor{g-aqua}{HTML}{689d6a}
\definecolor{g-orange}{HTML}{d65d0e}

\definecolor{g-red2}{HTML}{9d0006}
\definecolor{g-green2}{HTML}{78740e}
\definecolor{g-yellow2}{HTML}{b57614}
\definecolor{g-blue2}{HTML}{076678}
\definecolor{g-purple2}{HTML}{8f3f71}
\definecolor{g-aqua2}{HTML}{427b58}
\definecolor{g-orange2}{HTML}{af3a03}

\definecolor{bg0-s}{HTML}{32302f}
\definecolor{bg0}{HTML}{282828}
\definecolor{bg1}{HTML}{3c3836}
\definecolor{bg2}{HTML}{504945}
\definecolor{bg3}{HTML}{665c54}
\definecolor{fg0}{HTML}{fbf1c7}
\definecolor{fg1}{HTML}{ebdbb2}

\PassOptionsToPackage{hyphens}{url}\usepackage[breaklinks=true,bookmarks=false]{hyperref}
\hypersetup{
    colorlinks,
    linkcolor=g-red2,
    citecolor=g-red2,
    urlcolor=g-purple2
}

\usepackage[capitalise, noabbrev]{cleveref}
\crefname{assumption}{Assumption}{Assumptions}

\usepackage{xspace}
\makeatletter
\DeclareRobustCommand\onedot{\futurelet\@let@token\@onedot}
\def\@onedot{\ifx\@let@token.\else.\null\fi\xspace}

\def\eg{\emph{e.g}\onedot} \def\Eg{\emph{E.g}\onedot}
\def\ie{\emph{i.e}\onedot} \def\Ie{\emph{I.e}\onedot}
\def\NB{\emph{N.B}\onedot} 
\def\cf{\emph{cf}\onedot} 
\def\etc{\emph{\&c}\onedot} 
\def\wrt{w.r.t\onedot}
\makeatother

\usepackage[many]{tcolorbox}

\newtcolorbox{theorembox}{
    enhanced,
    sharp corners,
    frame hidden,
    boxrule=0pt,
    boxsep=0pt,
    notitle,
    borderline west={4pt}{0pt}{g-orange},
    colback=gray!5,
}

\newtcolorbox{standoutbox}{
    enhanced,
    sharp corners,
    breakable,
    frame hidden,
    boxrule=0pt,
    left=2mm, right=2mm,
    notitle,
    colback=g-blue!5,
    colframe=g-blue!5
}

\newtcolorbox{defbox}{
    enhanced,
    sharp corners,
    frame hidden,
    boxrule=0pt,
    boxsep=0pt,
    notitle,
    borderline west={4pt}{0pt}{g-red},
    colback=gray!5,
}

\title{Greed is Good:\\A Unifying Perspective on Guided Generation}

\author{%
Zander W.~Blasingame\\
Clarkson University\\
\texttt{blasinzw@clarkson.edu} \\
\And
Chen Liu\\
Clarkson University\\
\texttt{cliu@clarkson.edu}
}

\begin{document}

\maketitle

\begin{abstract}
    Training-free guided generation is a widely used and powerful technique that allows the end user to exert further control over the generative process of flow/diffusion models.
    Generally speaking, two families of techniques have emerged for solving this problem for \textit{gradient-based guidance}: namely, \textit{posterior guidance} (\ie, guidance by projecting the current sample to the target distribution via the target prediction model) and \textit{end-to-end guidance} (\ie, guidance by performing backpropagation throughout the entire ODE solve).
    In this work, we show that these two seemingly separate families can actually be \textit{unified} by looking at the posterior guidance as a \textit{greedy strategy} of \textit{end-to-end guidance}.
    We explore the theoretical connections between these two families and provide an in-depth theoretical understanding of these two techniques relative to the \textit{continuous ideal gradients}.
    Motivated by this analysis, we then show a method for \textit{interpolating} between these two families enabling a trade-off between compute and accuracy of the guidance gradients.
    We then validate this work on several inverse image problems and property-guided molecular generation.

\end{abstract}

\section{Introduction}
Guided generation greatly extends the utility of state-of-the-art generative models by allowing the end user to exert greater control over the generative process, ultimately making the tool more useful in a wide variety of applications ranging from conditional generation, editing of samples, inverse problems \etc
We focus particularly on a subset of neural differential equations that model \textit{affine probability paths}, in other words, diffusion and flow models due to their widespread adoption in a large variety of practical tasks.
\Eg, audio \parencite{liu2023audioldm,schneider2024mousai}, images \parencite{rombach2022high,flux2024}, biometrics \parencite{blasingame2024leveraging}, molecules \parencite{hoogeboom2022equivariant,ben-hamu2024dflow}, proteins \parencite{watson2023novo,skreta2024superposition}, \etc

We can divide the guided generation techniques into two broad categories: conditional training and training-free methods.
The former of these two requires the training of the underlying diffusion/flow model on additional conditional information, either as a part of the training or at a later time as additional fine-tuning \parencite{song2021denoising,ho2021classifier,hulora}.
The latter category instead makes use of some known guidance function defined on the data distribution and incorporates this information back to the model to influence the generative process.
These training-free techniques can be further broken down into two sub-categories, \ie, posterior and end-to-end guidance.
The former class of techniques uses a simple estimation of the posterior distribution that can be easily found in diffusion models \parencite{chung2023diffusion} and \textit{some} flow models \parencite[\cf][Section 4.8]{lipman2024flow-guide}.
This simple posterior estimate can then be fed into a guidance function to construct a gradient \wrt to the current timestep.
We refer to this category as \textit{posterior guidance} as they use this posterior estimate to perform the guidance process.
This can then be used to update the ODE solve as a form of classifier guidance \parencite{chung2023diffusion,yu2023freedom}.
The latter class of techniques, in contrast, performs backpropagation throughout the entire sampling process of the flow/diffusion model \parencite{blasingame2024adjointdeis,ben-hamu2024dflow}.
We refer to this category as \textit{end-to-end guidance} as it performs backpropagation throughout the \textit{entire} sampling trajectory.

The aim of this work is to bring these two seemingly disparate family of techniques together into a \textit{single unified view}. 
\begin{standoutbox}
Our key insight is that we can \textit{bridge} between techniques that use posterior sampling and techniques that use end-to-end optimization for guidance by viewing the former as a \textit{greedy strategy} on the latter.
\end{standoutbox}

\paragraph{Contributions.}
In light of this insight, we compare several state-of-the-art techniques from this perspective, showing how this perspective yields a unified and flexible framework for viewing guided generation with flow/diffusion models.
We perform a detailed analysis of this greedy strategy, showing that it is not only a unifying view, but that it actually makes \textit{good} decisions under certain scenarios.
We then show a perspective which allows one to move between these two classes of guided generation techniques, opening up an exciting and novel design space.
Lastly, we conduct some numerical experiments on inverse image problems and molecule generation.

\section{Preliminaries}
Flow models \parencite{lipman2023flow} are a highly popular class of generative models that model the generative process as a neural \textit{ordinary differential equation} (ODE) \parencite{chen2018neural}.
Consider two $\R^d$-valued random variables: $\bfX_0 \sim p(\bfx)$ and $\bfX_1 \sim q(\bfx)$, denoting the \textit{source} (noise) and \textit{target} (data) distributions, respectively.
Then consider a time-dependent vector field $\bfu \in \mathcal{C}^{1,r}([0,1] \times \R^d;\R^d)$\footnote{For notational simplicity, we let $\mathcal{C}^{k_1,k_2,\ldots,k_n}(X_1 \times X_2 \times \cdots \times X_n ; Y)$ denote the set of continuous functions that are $k_i$-times differentiable in the $i$-th argument mapping from $(X_1 \times X_2 \times \cdots \times X_n)$ to $Y$, if $Y$ is omitted, then $Y = \R$.
}
with $r \geq 1$ which determines
a time-dependent flow $\Phi_t \in \mathcal{C}^{1,r}([0,1]\times\R^d;\R^d)$ which satisfies the ODE
\begin{equation}
    \Phi_0(\bfx) = \bfx, \quad \frac{\rmd}{\rmd t} \Phi_t(\bfx) = \bfu(t, \Phi_t(\bfx)).
\end{equation}
This is known as a $\mathcal{C}^{r}$-flow and this flow is diffeomorphism in its second argument for all $t \in [0,1]$.
For notational simplicity let $\bfu_t(\bfx) \mapsto \bfu(t, \bfx)$.
A special case of flow models are known as \textit{affine probability paths} and are defined as $\bfX_t = \alpha_t \bfX_0 + \sigma_t \bfX_1$ with schedule $(\alpha_t,\sigma_t)$. 
We provide more details on flow models in \cref{app:more_flow}.\footnote{Without loss of generality we consider flow models which subsume the ODE formulation of diffusion models.}

\begin{figure*}[t]
    \centering

    \tikzset{
        basic/.style  = {draw, text width=20mm, font=\tiny, rectangle},
        root/.style = {basic, thin, align=center},
        tnode/.style = {basic, thin, align=center, font=\tiny\bfseries},
        xnode/.style = {basic, thin, align=left, text width=30mm},
        xnnode/.style = {basic, thin, align=left, text width=20mm},
        arrow/.style = {thick, dotted, shorten >=3, shorten <=3, ->}
    }

    \begin{forest} for tree={
        grow=east,
        growth parent anchor=west,
        parent anchor=east,
        child anchor=west,
        fork sep=6mm,
        l sep=12mm,
    },
    forked edges,
    [Training-free\\guided generation, root
        [Posterior\\guidance, tnode, name=posterior
            [\parencite{chung2023diffusion}, xnnode]
            [\parencite{yu2023freedom}, xnnode]
        ]
        [End-to-end\\guidance, tnode, name=e2e
            [State\\optimization, tnode
                [\parencite{blasingame2024adjointdeis}, xnode]
                [\parencite{ben-hamu2024dflow}, xnode]
            ]
            [Control signal\\ optimization, tnode
                [\parencite{liu2023flowgrad}, xnode]
                [\parencite{wang2024training}, xnode]
            ]
        ]
    ]
    \draw[arrow] (e2e) to node[font=\scriptsize, anchor=west,align=center, yshift=0mm]{A greedy strategy} (posterior); 
    \end{forest}

    \caption{The greedy perspective as a unification of separate families in the taxonomy of training-free guided generation. We provide a more detailed version of this in \cref{fig:app:taxonomy_of_guided}.}
    \label{fig:taxonomy_of_guided}
\end{figure*}

\section{An overview of training-free guidance with gradients}
We explore techniques for solving \textit{training-free} guidance problems---this is in contrast with techniques like classifier \parencite{diff_beat_gan,song2021scorebased} and classifier-free \parencite{ho2021classifier} guidance---which use some off-the-shelf guidance function $\mathcal L \in \C^1(\R^d)$ defined on the output of the flow model.
\Ie, we wish to optimize the ODE solve such that the output $\bfx_1$ minimizes $\mathcal L$.
Suppose we have numerical scheme (Euler, RK4, DPM-Solver, \etc) denoted
\begin{equation}
    \label{eq:numerical_solver}
    \begin{aligned}
        &\bm\Phi : \R \times \R \times \R^d \times \C(\R \times \R^d; \R^d) \to \R^d,\\
        &\bm\Phi(t_n, t_{n+1}, \bfx_n, \bfu_{t}^\theta) \mapsto \bfx_{n+1}.
    \end{aligned}
\end{equation}
For simplicity we will omit the explicit dependency of the numerical scheme on $\bfu_t^\theta$ and assume it implicitly; likewise, let $\bm \Phi_{h}(t_n, \cdot, \cdot) \mapsto \bm\Phi(t_n, t_{n+1}, \cdot, \cdot)$ where $h = t_{n+1} - t_n$.
We write this objective more formally below in \cref{eq:problem_stmt}.
\begin{defbox}
    \textbf{Problem statement.}
    Given some $t_1 \in [0, 1)$ and step size regime $\{t_1 < t_2 < \ldots < t_N = 1\}$ solve:
    \begin{equation}
        \label{eq:problem_stmt}
        \begin{array}{ll@{}r@{}}
            \textrm{Find a sequence} & \{\bfx_n\}_{n=1}^N & \textrm{which minimizes}\; \mathcal L(\bfx_N),\\
            \textrm{subject to} & \bfx_{n+1} = \bm \Phi(t_{n+1}, t_n, \bfx_n).
        \end{array}
    \end{equation}
\end{defbox}
Next, we will detail two popular families of techniques for solving the problem mentioned above.
We illustrate the relationships between these different families in \cref{fig:taxonomy_of_guided}, a taxonomy of training-free guidance methods.
We note that these two seemingly separate branches can be unified back into a single branch, by the viewing posterior guidance techniques as a greedy strategy of the later.
Likewise, we provide a visual overview of the guidance mechanisms in \cref{fig:greedy-overview}.

\begin{figure}[t]
    \centering
    \includegraphics[width=\textwidth, trim={0 0 0 2.5cm}, clip]{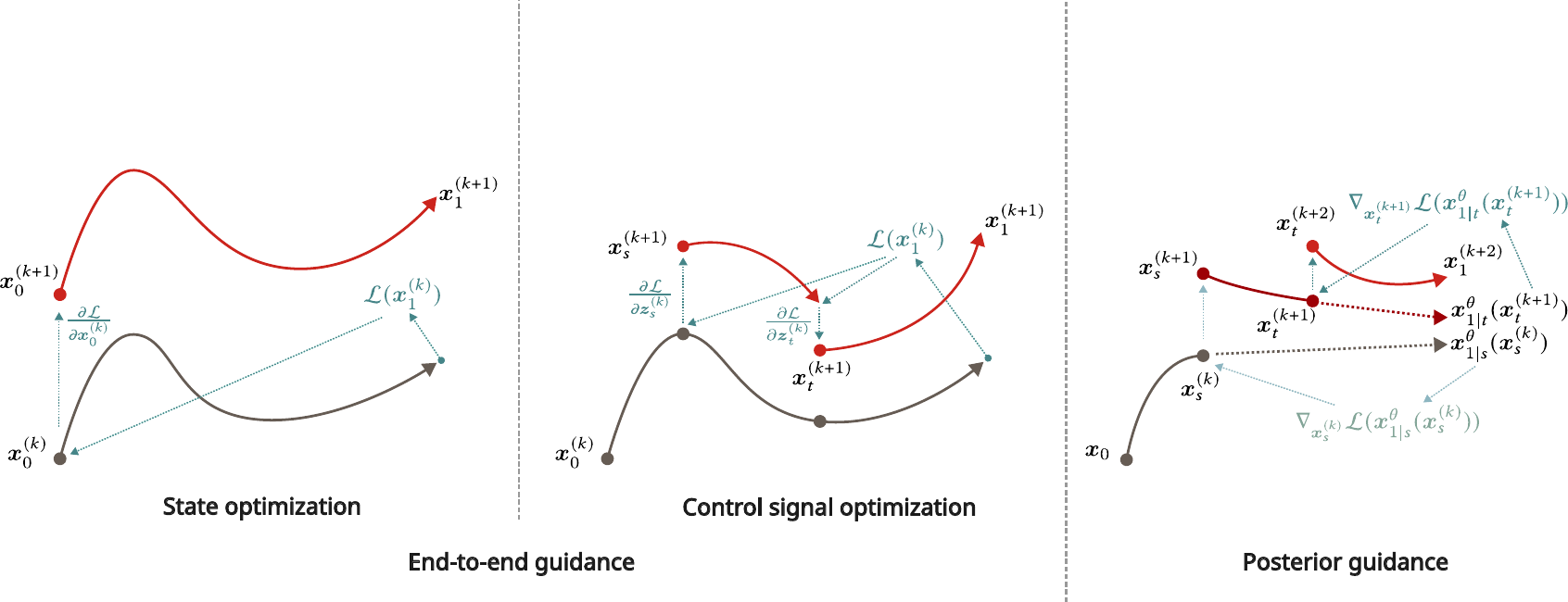}
    \caption[
        Visual comparison of different training-free guided generation techniques.
    ]{
        Visual comparison of different training-free guided generation techniques.
    }
    \label{fig:greedy-overview}
\end{figure}

\subsection{Posterior guidance}
A popular technique for \textit{training-free} guidance is what we will term \textit{posterior guidance} \parencite{chung2023diffusion,yu2023freedom}.
The key idea behind this strategy is to use the parameterized target prediction model $\bfx_{1|t}^\theta(\bfx)$, \ie, the expected value of the posterior distribution given $\bfX_t = \bfx$, to provide a guidance gradient of the form $\nabla_\bfx \mathcal L(\bfx_{1|t}^\theta(\bfx))$ for some guidance function $\mathcal L \in \C^1(\R^d)$.
For literature working with score-based generative models \parencite{song2021scorebased}, this interpretation arose from the famous Tweedie's formula \parencite{stein1981estimation,efron2011false}.
Thus, for each $\bfx_n$ in the ODE solve, we add guidance to it in the form of posterior guidance gradient.

\subsection{End-to-end optimization for guidance}
Another popular class of techniques is what we will term \textit{end-to-end guidance} \parencite{blasingame2024adjointdeis,ben-hamu2024dflow}, \ie, techniques which perform guidance by optimizing the initial condition $\bfx_0$ \wrt the guidance function $\mathcal L$; such techniques require performing backpropagation through a neural ODE.
Fittingly, we will import notations and terminology from the study of \textit{neural differential equations} \parencite{kidger_thesis} to discuss these techniques.
The first technique for performing this kind of guidance is known as \textit{discretize-then-optimize} (DTO) where the numerical scheme (\cf \cref{eq:numerical_solver}) is part of the computation graph of the model reverse-mode automatic differentiation \parencite{linnainmaa1976backprop} is applied, \ie, \textit{vanilla backpropagation}.
The memory cost of such techniques, however, is $\mathcal O(n)$, prompting researchers to explore the second method known as \textit{optimize-then-discretize} (OTD) which instead solves \textit{another} ODE in \textit{reverse-time} which models the continuous-time dynamics of reverse-mode differentiation, this is called the \textit{continuous adjoint method} \parencites{chen2018neural}[\cf][Section 5.1.2]{kidger_thesis}. 

Given a flow model $\bfu_\theta \in \mathcal{C}^{1,1}([0,1] \times \R^d;\R^d)$ that is Lipschitz continuous in its second argument and the solution $\bfx: [0,1] \to \R^d, \bfx_t \mapsto \bfx(t)$, let $\bfa_\bfx \coloneq \partial\mathcal{L}/\partial \bfx_t$ denote the \textit{adjoint state}.
Then $\bfa_\bfx(t)$ can be found by solving the continuous adjoint equation:
\begin{equation}
    \bfa_\bfx(1) = \frac{\partial \mathcal{L}}{\partial \bfx_1}, \quad \frac{\rmd \bfa_\bfx}{\rmd t}(t) = -\bfa_\bfx(t)^\top \frac{\partial \bfu_t^\theta}{\partial \bfx}(\bfx_t).
    \label{eq:continuous_adjoint_eqs}
\end{equation}
\NB, this technique was first proposed by \textcite{pontryagin1963} and popularized for neural differential equations by \textcite{chen2018neural}.
This approach has a constant memory cost $\mathcal O(1)$; however, this comes with the cost of several drawbacks related to the numerical scheme.
While these issues are not particularly relevant to our theoretical analyses, we note them in \cref{app:otd_issues} for the ML practitioner.

\section{A greedy perspective on guidance}
\label{sec:greedy_perspective}
Now returning back to our problem statement from \cref{eq:problem_stmt}, the end-to-end guidance techniques amount to optimizing the initial condition $\bfx_0$ in light of the entire solution trajectory admitted by the numerical scheme.
A natural question we consider for problems of this form is that rather than finding the full sequence $\{\bfx_n\}$, can we make use of local information instead?
\Ie,
\begin{standoutbox}
    \begin{insight}
        Rather than solving the full ODE from $\bfx_t$, what if we greedily took a locally optimal step at each $\bfx_t$ instead?
    \end{insight}
\end{standoutbox}

Formally, we define a greedy strategy is the following augmentation to the numerical scheme from \cref{eq:numerical_solver} as
\begin{align}
    \label{eq:greedy_action}
    \bfx_n^{\mathcal G} &= \mathcal G (t_n, \bfx_n, \bfu_{t_n}^\theta),\\
    \bfx_{n+1} &= \bm\Phi(t_n, t_{n+1}, \bfx_n^{\mathcal G}),
\end{align}
where $\mathcal G$ is the \textit{greedy action} which makes its decision from only information available at time $t_n$.

Now in particular we are interested in a specific greedy action, \ie, posterior guidance.
We define this greedy action as the solution to the following iterative process with initial value $\bfx_n^{(0)} = \bfx_n$ which solves
\begin{equation}
    \label{eq:action}
    \bfx_n^{(k+1)} = \bfx_n^{(k)} - \eta \nabla \mathcal L\left(\bfx_{1|t_n}^\theta(\bfx_n^{(k)})\right),
\end{equation}
for some sufficient number $k > 0$ and learning rate $\eta > 0$.

By construction this greedy action is the popular strategy of posterior guidance.
The rest of this section is then devoted to exploring the connections between this greedy action and end-to-end guidance schemes.
More, succinctly we state our insight below:
\begin{standoutbox}
    \begin{insight}
        Posterior guidance can be viewed as Euler schemes within the DTO or OTD backpropagation schemes.
    \end{insight}
\end{standoutbox}

To make our analysis simpler, let us write the flow from $s$ to $t$ in terms of the target prediction model.
The flow from time $s$ to time $t$ can then be expressed as the integral of the right-hand side of \cref{eq:vector_field_denoiser} over time.
Thus, the flow is now expressed as a semi-linear integral equation with linear term $a_t\bfx$ and non-linear term $b_t\bfx_{1|t}^\theta(\bfx)$.
Due to this semi-linear structure, we apply the same technique of \textit{exponential integrators} \parencite{hochbruck2010exponential} that has been successfully used to simplify numerical solvers for diffusion models \parencite{lu2022dpmsolver,zhangfast,gonzalez2024seeds}.
\NB, the full derivations and proofs for this section can be found in \cref{app:greedy_perspective}.

Let $\gamma_t \coloneq \alpha_t/\sigma_t$ denote the signal-to-noise ratio (SNR), then $\gamma_t$ is a monotonically increasing sequence in $t$, due to the properties of $(\alpha_t, \sigma_t)$ (\cf \cref{eq:schedule_boundaries}) and thus has an inverse $t_\gamma$ such that $t_\gamma(\gamma(t)) = t$.
With abuse of notation, we let $\bfx_\gamma \coloneq \bfx_{t_\gamma(\gamma)}$ and $\bfx_{1|\gamma}^\theta(\cdot) = \bfx_{1|t_\gamma(\gamma)}^\theta(\cdot)$.
As such, we can rewrite the solution to the flow model in terms of $\gamma$ by making use of exponential integrators, which we show in \cref{prop:exact_sol_flow} with the full proof provided in \cref{proof:exact_sol_flow}.
\begin{theorembox}
\begin{restatable}[Exact solution of affine probability paths]{proposition}{exactsolflow}
    \label{prop:exact_sol_flow}
    Given an initial value of $\bfx_s$ at time $s \in [0, 1]$ the solution $\bfx_t$ at time $t \in [0,1]$ of an ODE governed by the vector field in \cref{eq:marginal_vec} is:
    \begin{equation}
        \label{eq:exact_sol_flow}
         \bfx_t = \frac{\sigma_t}{\sigma_s} \bfx_s + \sigma_t \int_{\gamma_s}^{\gamma_t} \bfx_{1|\gamma}^\theta(\bfx_\gamma)\;\rmd \gamma.
    \end{equation}
\end{restatable}
\end{theorembox}

\begin{remark}
    This result bears some similarity to \textcite[Propostion 5.1]{lu2022dpm++}; however, they integrate \wrt the log-SNR; their result can be recovered, \textit{mutatis mutandis}, with the identity $\lambda_t = \log \gamma_t$.
\end{remark}

\subsection{Greedy guidance as an Euler scheme}
Now equipped with this simplified form, we begin to draw connections between end-to-end guidance and our greedy strategy.
In \cref{prop:greedy_is_explicit_euler} we show that the greedy action in \cref{eq:action} can be interpreted as backpropagation via a DTO scheme with an Euler step of size $h = \gamma_1 - \gamma_t$.
\begin{theorembox}
    \begin{restatable}[Greedy as an explicit Euler scheme within DTO]{proposition}{greedyasexplicit}
        \label{prop:greedy_is_explicit_euler}
        For some trajectory state $\bfx_t$ at time $t$, the greedy gradient given by $\nabla_{\bfx} \mathcal{L}(\bfx_{1|t}^\theta(\bfx))$ is the DTO scheme with an explicit Euler discretization with step size $h = \gamma_1 - \gamma_t$.
    \end{restatable}
\end{theorembox}

Now we examine greedy action from the perspective of an OTD scheme.
In \cref{thm:greedy_is_implicit_euler} we show that a greedy strategy can be viewed as the first iteration of a fixed-point method of an implicit Euler discretization of the continuous adjoint equations.
\begin{theorembox}
    \begin{restatable}[Greedy as an implicit Euler scheme within OTD]{proposition}{greedyasimplicit}
        \label{thm:greedy_is_implicit_euler}
        For some trajectory state $\bfx_t$ at time $t$, the greedy gradient given by $\nabla_{\bfx_t} \mathcal{L}(\bfx_{1|t}^\theta(\bfx_t))$ is an implicit Euler discretization of the continuous adjoint equations for the true gradients with step size $h = \gamma_1 - \gamma_t$.
    \end{restatable}
\end{theorembox}
\begin{proof}[Proof sketch]
First, we use the technique of exponential integrators to simplify the continuous adjoint equations. Then we perform a first-order Taylor expansion around $\gamma_t$, which is equivalent to an implicit Euler scheme, as we calculate the gradient flow from $1$ to $t$. The full proof is provided in \cref{proof:greedy_is_implicit_euler}.
\end{proof}

\section{Is greed good?}
\label{sec:theory_guidance}
A natural question to ask in light of this discussion on taking this greedy action is why even bother backpropagating through the ODE solve at all for guidance?
After all, we could simply run the optimization process directly in the data space (\cf \cref{eq:action}).
So why perform end-to-end guidance or this greedy action at all?
\NB, the full derivations and proofs for this section may be found in \cref{app:dynamics}.

We begin by examining the structure of the gradient $\nabla_{\bfx} \mathcal L (\Phi_{t,1}^\theta(\bfx))$.
By the chain rule we observe the following:\footnote{Let $\nabla_{\bfx_1}$ be shorthand for the gradient \wrt the output $\Phi_{t,1}^\theta(\bfx)$.}
\begin{equation}
    \label{eq:chain_rule_gradient_flow}
    \nabla_\bfx \mathcal L \left(\Phi_{t,1}^\theta(\bfx) \right) = \nabla_\bfx \Phi_{t,1}^\theta(\bfx)^\top \nabla_{\bfx_1} \mathcal L\left(\Phi_{t,1}^\theta(\bfx_1)\right).
\end{equation}
The question then is what is the behavior of $\nabla_\bfx \Phi_{t,1}^\theta(\bfx)$?
We answer this in \cref{thm:jacobians_aggp} below, providing an integral equation for $\nabla_\bfx \Phi_{s,t}^\theta(\bfx)$.
\begin{theorembox}
    \begin{restatable}[Jacobian matrices of affine Gaussian probability paths]{theorem}{jacobiansaggp}
        \label{thm:jacobians_aggp}
        For the standard affine Gaussian probability path with flow model $\Phi_{s,t}^\theta(\bfx)$, the Jacobian matrix $\nabla_\bfx \Phi_{s,t}(\bfx)$ as function of $\bfx$ is given as the solution to
        \begin{equation}
            \label{eq:jacobians_aggp}
            \nabla_\bfx \Phi_{s,t}^\theta(\bfx) = \frac{\sigma_t}{\sigma_s}\bm I + \sigma_t\int_s^t \dot\gamma_u \frac{\gamma_u}{\sigma_u} \var_{1|u}(\Phi_{s,u}^\theta(\bfx))\nabla_{\bfx} \Phi_{s,u}^\theta(\bfx)\;\rmd u,
        \end{equation}
        where
        \begin{equation}
            \var_{1|t}(\bfx) = \ex_{p_{1|t}(\bfx_1|\bfx)} \left[(\bfx_1 - \bfx_{1|t}^\theta(\bfx)) (\bfx_1 - \bfx_{1|t}^\theta(\bfx))^\top\right].
        \end{equation}
    \end{restatable}
\end{theorembox}

\begin{remark}
    From \cref{thm:jacobians_aggp} we observe the Jacobian-vector product $\nabla_\bfx \Phi_{s,t}^\theta(\bfx)^\top \bm v$ corresponds to an integral of covariance projections applied to $\bm v$.\footnote{Readers familiar with the work of \textcite{ben-hamu2024dflow} may notice some similarities between our result \cref{thm:jacobians_aggp} and \textcite[Theorem 4.2]{ben-hamu2024dflow}. We discuss this more in \cref{remark:diff_in_jacobian_thms}.}
\end{remark}

Thus, we see that the continuous-time backpropagation process through the flow model is a projection of the loss by a covariance matrix into the directions of highest variance, \ie, the guidance encourages the state to evolve within states on the data manifold.
We elaborate on this more in \cref{app:dyn_grad_guidance}.
While this is a nice observation we cannot solve such an integral in practice.
What about our greedy strategy, how does it impact the loss function?

\subsection{Dynamics of gradient guidance}
We now consider how the output of the flow model will change under greedy guidance.
In particular, we are interested in how $\Phi_{t,1}^\theta(\bfx)$ changes under the following gradient step
\begin{equation}
    \bfx' = \bfx - \eta \nabla_\bfx \mathcal L\left(\bfx_{1|t}^\theta(\bfx)\right).
\end{equation}
To do this, we make use of the Gateaux differential \parencite{gateaux1913fonctionnelles} which allows us to define the differential that describes how the output of the flow model $\bfx_1$ evolves with changes to $\bfx$ at time $t$.
We present the result to this question in \cref{prop:dyn_greedy_guidance} below.

\begin{theorembox}
    \begin{restatable}[Dynamics of greedy gradient guidance]{proposition}{dyngreedyguidance}
        \label{prop:dyn_greedy_guidance}
        Consider the standard affine Gaussian probability paths model trained to zero loss.
        The Gateaux differential of $\bfx$ at some time $t \in [0, 1]$ in the direction of the gradient $\nabla_\bfx \mathcal L\left(\bfx_{1|t}^\theta(\bfx)\right)$ is given by
        \begin{equation}
            \delta_\bfx^{\mathcal G} \Phi_{t,1}^\theta(\bfx) = -\nabla_\bfx\Phi_{t,1}^\theta(\bfx) \nabla_\bfx\bfx_{1|t}^\theta(\bfx)^\top \nabla_{\bfx_1} \mathcal L(\bfx_1).
        \end{equation}
    \end{restatable}
\end{theorembox}

\begin{remark}
    Recall that from \cref{thm:jacobians_aggp} and \parencite[Proposition 4.1]{ben-hamu2024dflow} we know that both $\nabla_\bfx\Phi_{t,1}^\theta(\bfx)$ and $\nabla_\bfx\bfx_{1|t}^\theta(\bfx)$ consist of covariance matrices, thus the dynamics of greedy gradient guidance are governed by this covariance projection of the loss.
\end{remark}

Next, we ask what is the difference between the \textit{idealized} gradient $\nabla_\bfx \Phi_{t,1}^\theta(\bfx)$ and the greedy gradient $\nabla_\bfx \bfx_{1|t}^\theta(\bfx)$?
Intuitively, we find that it is bound by the local truncation error, \ie, $\mathcal O(h^2)$ which we show below.

\begin{theorembox}
    \begin{restatable}[Dynamics of gradient vs greedy guidance]{theorem}{gradvsgreed}
        \label{thm:grad_vs_greed}
        The difference between the dynamics of gradient guidance in \cref{prop:dyn_grad_guidance} and greedy gradient guidance in \cref{prop:dyn_greedy_guidance} for a point $\bfx$ at time $t$ with guidance function $\mathcal L \in \C^1(\R^d)$ is bounded by $\mathcal O(h^2)$ where $h \coloneq \gamma_1 - \gamma_t$, \ie,
        \begin{equation}
            \left\|\nabla_\bfx \Phi_{t,1}^\theta(\bfx) - \nabla_\bfx \bfx_{1|t}^\theta(\bfx)\right\| = \mathcal O (h^2).
        \end{equation}
    \end{restatable}
\end{theorembox}

An important question is whether a greedy strategy makes \textit{good} decisions at each timestep.
\Ie, if we make a good decision at time $t$, does that ensure that an optimal solution was made in the sense of $\Phi_{1|t}^\theta(\bfx_t)$.
A natural way to examine this question is to consider whether convergence in the local case implies convergence of the whole solution trajectory.
We find that up to a bound dependent on the step size, convergence in the greedy solution implies convergence in the flow, which we state more formally in \cref{thm:convergence}.
\begin{theorembox}
    \begin{restatable}[Greedy convergence]{theorem}{greedyconverge}
        \label{thm:convergence}
        For affine probability paths, if there exists a sequence of states $\bfx_t^{(n)}$ at time $t$ such that it converges to the locally optimal solution $\bfx_{1|t}^\theta(\bfx_t^{(n)}) \to \bfx_{1}^*$.
        Then the solution, $\Phi_{1|t}^\theta(\bfx_t^{(n)})$, converges to a neighborhood of size $\mathcal O(h^2)$ centered at $\bfx_1^*$.
    \end{restatable}
\end{theorembox}

\section{Beyond Euler}
\label{sec:beyond_euler}
Motivated by this connection between the powerful, but expensive, end-to-end guidance techniques and posterior guidance techniques, we ask is there a middle-ground between them?
A natural extension would be to consider something beyond the Euler scheme from the previous section, \eg, applying the midpoint method or two Euler steps.
To motivate this discussion more rigorously we present \cref{thm:grad_vs_single_step}, which shows that for any explicit single-step Runge-Kutta solver, the error between the \textit{ideal} gradient and this estimated gradient is on the order of the local truncation error of the underlying numerical solver.

\begin{theorembox}
    \begin{restatable}[Truncation error of single-step gradients]{theorem}{gradsinglestepdto}
        \label{thm:grad_vs_single_step}
        Let $\bm \Phi$ be an explict Runge-Kutta solver of order $\alpha > 0$ of a flow model with flow $\Phi_{s,t}^\theta(\bfx)$.
        Then for any $t \in [0, 1]$,
        \begin{equation}
            \left\|\nabla_\bfx \Phi_{t,1}^\theta(\bfx) - \nabla_\bfx \bm \Phi_{t, 1}(\bfx) \right\| = \mathcal O (h^{\alpha + 1}),
        \end{equation}
        where $h = 1 - t$.
    \end{restatable}
\end{theorembox}

\begin{standoutbox}
    \begin{insight}
        We can use a higher-order solver to move between posterior and end-to-end guidance exchanging compute and gradient accuracy.
    \end{insight}
\end{standoutbox}

This theoretical tool enables us to move between posterior and full end-to-end guidance choosing whichever point between compute and accuracy happens to be most suitable, hopefully opening a larger design space for solving interesting problems.
Additional discussions and the full derivations are found in \cref{app:beyond_euler}.

\section{Experiments}
\label{sec:experiments}
Motivated by the theoretical connections from the previous sections we apply the greedy posterior strategy (Euler) to several problems using flow/diffusion models, as well as several methods lying in the in between space of end-to-end guidance and posterior guidance, namely, a single-step midpoint scheme and 2-step Euler scheme.

\begin{figure}[t]
    \centering
    \includegraphics[width=\textwidth]{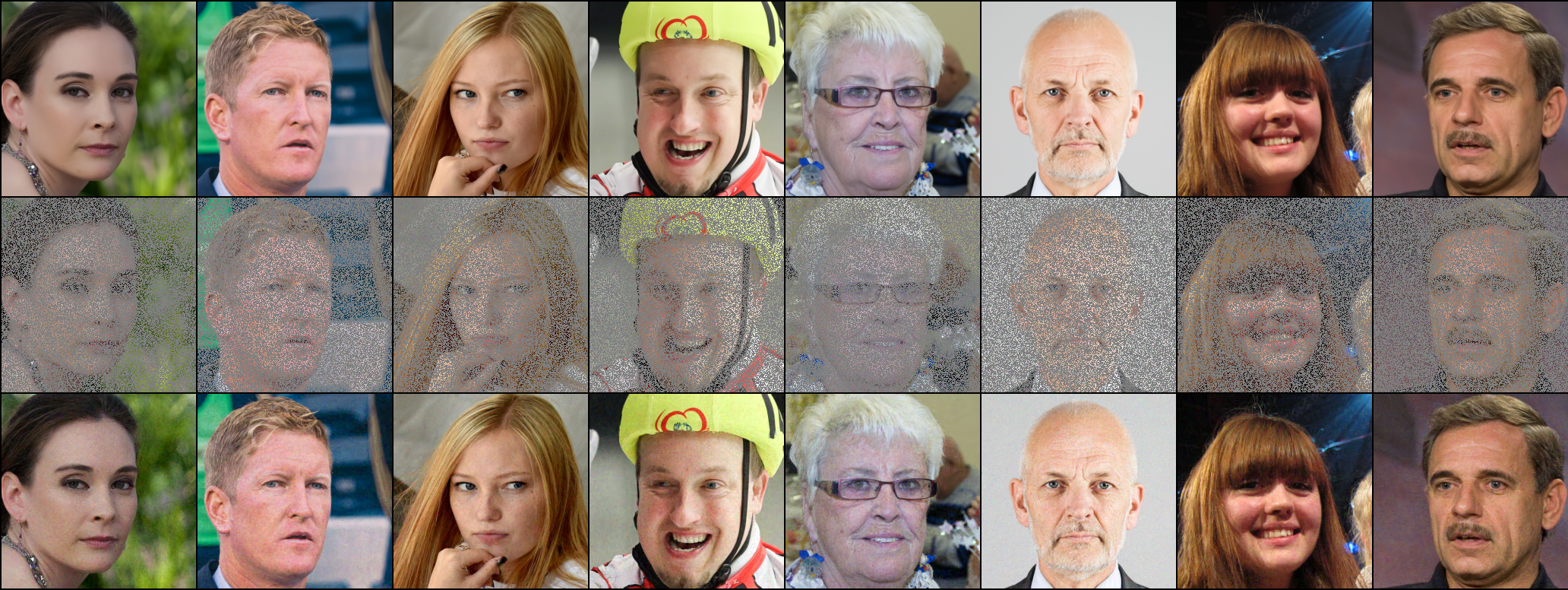}
    \label{fig:inv_image_ex}
    \caption{Qualitative visualization of using posterior guidance to solve an inverse problem on the task of inpainting with a 70\% random mask. Top row is the ground truth, middle row is the measurement, and the bottom row is the reconstruction.}
\end{figure}

\subsection{Inverse problems for images}
\label{sec:inv_problems}
A common application of posterior guidance has been in solving inverse problems \parencite{song2021scorebased,chung2022come} (\cf \cref{app:inverse_problems}).
As such, we explore several inverse problems in the image domain.
In particular, we explore a set of inverse image problems on a subset of 100 images from the FFHQ \parencite{stylegan} $256\times 256$ dataset.
We make use of the pre-trained diffusion model from \textcite{chung2023diffusion} trained on the FFHQ dataset.

\paragraph{Inverse problems and metrics.}
Following \parencite{zhang2024improvingdiffusioninverseproblem} we conduct experiments on the following linear tasks: super resolution, Gaussian deblurring, motion deblurring, inpaintining (with a box mask), and inpainting (with a 70\% random mask); along with three non-linear problems: phase retrieval, high dynamic range (HDR) reconstruction, and non-linear deblurring.
We use the standard evaluation metrics of \textit{peak signal-to-noise-ratio} (PSNR), \textit{structural similarity index measure} (SSIM), \textit{Learned Perceptual Image Patch Similarity} (LPIPS) \parencite{zhang2018unreasonable}, and \textit{Fr\'echet Inception Distance} (FID) \parencite{heusel2017fid}.
We solve the probability flow ODE with the midpoint scheme and 20 discretization steps; further configuration details are reported in \cref{app:inv_images}.

\begin{table}[t]
    \centering
    \caption{ A snapshot of the quantitative results for solving inverse image problems on FFHQ. We report the mean performance (PSNR, SSIM, and LPIPS) across 100 validation images along with the FID. All tasks are using a noisy measurement with noise level $\beta_\bfy = 0.05$. The full results are found in \cref{tab:add_inv_images}.}
    \label{tab:inv_images}
    \small
    \begin{tabular}{ll cccc}
        \toprule
        \textbf{Task} & \textbf{Method} & \textbf{PSNR} ($\uparrow$) & \textbf{SSIM} ($\uparrow$) & \textbf{LPIPS} ($\downarrow$) & \textbf{FID} ($\downarrow$)\\

        \midrule
        \multirow{4}{*}{Inpaint (random)} & Greedy (Euler)  & 30.87 & 0.823 & 0.141 & 40.73\\
        & Greedy (midpoint)  & 31.03 & 0.816 & 0.139 & 38.80\\
        & Greedy (2-step Euler)  & 30.80 & 0.811 & 0.144 & 39.23\\
        
        & DAPS & 31.12 & 0.844 & 0.098 & 32.17\\
        & DPS & 25.46 & 0.823 & 0.203 & 69.20\\
        
        \midrule
        \multirow{4}{*}{Gaussian deblurring} & Greedy (Euler)  & 28.01 & 0.766 & 0.182 & 57.04\\
        & Greedy (midpoint)  & 28.36 & 0.776 & 0.185 & 58.55\\
        & Greedy (2-step Euler)  & 28.18 & 0.774 & 0.181 & 57.18\\

        & DAPS & 29.19 & 0.817 & 0.165 & 53.33\\
        & DPS & 25.87 & 0.764 & 0.219 & 79.75\\
        
        \bottomrule
    \end{tabular}
\end{table}

\paragraph{Results.}
We present some qualitative results on reconstructing images from a random mask in \cref{fig:inv_image_ex}.
Quantitatively, we present a snapshot of our full results (\cf \cref{tab:add_inv_images}) on the inpainting with random mask and Gaussian deblurring tasks.
For reference we include the standard DPS \parencite{chung2023diffusion} and the recent state-of-the-art DAPS \parencite{zhang2024improvingdiffusioninverseproblem}.
We observe that the posterior guidance strategy works well performing closer to DAPS than DPS.
Interestingly, on these tasks the extra compute and smaller truncation error of the midpoint and 2-step Euler did not lead to any noticeable performance gains.
We report further results in \cref{app:add_inv} along with additional analysis and discussion.

\begin{table}[t]
    \centering
    \caption{Further ablations on the number of discretization steps on the non-linear HDR inverse problem.}
    \label{tab:ablation_hdr}
    \small
    \begin{tabular}{l cccc}
        \toprule
        \textbf{Method} & \textbf{PSNR} ($\uparrow$) & \textbf{SSIM} ($\uparrow$) & \textbf{LPIPS} ($\downarrow$) & \textbf{FID} ($\downarrow$)\\
        \midrule
        DAPS & ${2 7 . 1 2}_{ \pm 3.53}$ & ${0 . 7 5 2}_{ \pm 0.041}$ & ${0 . 1 6 2}_{ \pm 0.072}$ & 42.97\\
        DPS & ${22.73}_{ \pm 6.07}$ & ${0.591}_{ \pm 0.141}$ & $0.264_{ \pm 0.156}$ & 112.82\\
        RED-diff & $22.16_{ \pm 3.41}$ & $0.512_{ \pm 0.083}$ & ${0.258}_{ \pm 0.089}$ & ${108.32}$\\
        \hfill\\
        Greedy (Euler) & $25.07_{\pm 4.25}$ & $0.776_{\pm 0.126}$ & $0.173_{\pm 0.070}$ & 43.25\\
        Greedy (2-step Euler) & $26.32_{\pm 4.34}$ & $0.802_{\pm 0.111}$ & $0.173_{\pm 0.065}$ & 38.64\\
        Greedy (3-step Euler) & $27.17_{\pm 4.21}$ & $0.820_{\pm 0.096}$ & $0.154_{\pm 0.062}$ & 36.07\\
        Greedy (4-step Euler) & $27.89_{\pm 4.10}$ & $0.828_{\pm 0.092}$ & $0.151_{\pm 0.061}$ & 36.94\\
        Greedy (5-step Euler) & $\textbf{28.27} _{\pm 4.01}$ & $\textbf{0.831}_{\pm 0.088}$ & $\textbf{0.149}_{\pm 0.059}$ & \textbf{35.35}\\

        \hfill\\
        DTO (1-step) & $13.16 _{\pm 1.15}$ & $0.372_{\pm 0.083}$ & $0.521_{\pm 0.059}$ & 108.39\\
        DTO (2-step) & $14.91 _{\pm 1.23}$ & $0.372_{\pm 0.080}$ & $0.483_{\pm 0.061}$ & 98.93\\
        DTO (4-step) & $16.37 _{\pm 1.38}$ & $0.455_{\pm 0.082}$ & $0.457_{\pm 0.066}$ & 93.52\\
        DTO (8-step) & $16.37 _{\pm 1.38}$ & $0.455_{\pm 0.082}$ & $0.457_{\pm 0.066}$ & 93.52\\
        \bottomrule
    \end{tabular}
\end{table}

\paragraph{Ablations on discretization steps.}
As we discussed in \cref{sec:beyond_euler} we can improve performance by taking more step sizes.
We preform a more involved ablation of this design axis on the \textit{high-dynamic range} (HRD) reconstruction experiment detailed in \cref{tab:ablation_hdr}.
Additionally, we also report the results from the RED-diff \parencite{mardani2024a} algorithm.
Following \textcite{zhang2024improvingdiffusioninverseproblem} for the \textit{non-linear} inverse problem of HDR reconstruction we perform 4 runs per algorithm and report the mean and standard deviation.
We notice that the Greedy (5-step Euler) performs very well beating the SOTA DAPS algorithm, even the 2-step and 3-step perform variants as well or better than DAPs on this problem.
Increasing the number of discretization steps leads to better performance (\cf \cref{thm:grad_vs_single_step}).
Interestingly, the standard deviation decreases as well with the results becoming more consistent.
We also compared to a full DTO run with vary step sizes, \ie, end-to-end optimization with a vary number of steps used in calculating the gradient (the full 20 are used for the final sampling).
We do observe a similar trend of increasing performance as we increase the number of steps, however, it far under-performs the greedy strategy for a similar compute budget.

\begin{figure}[h]
    \centering
    \begin{subfigure}{0.125\textwidth}
        \centering\includegraphics[width=\textwidth]{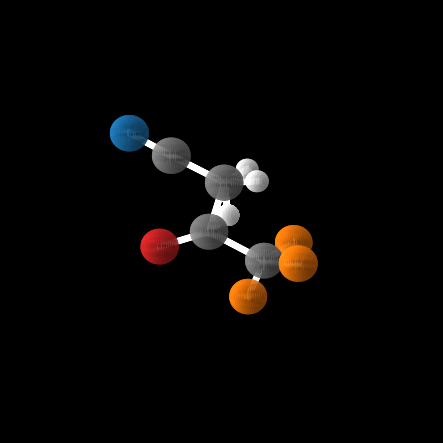}
    \end{subfigure}%
    \begin{subfigure}{0.125\textwidth}
        \centering\includegraphics[width=\textwidth]{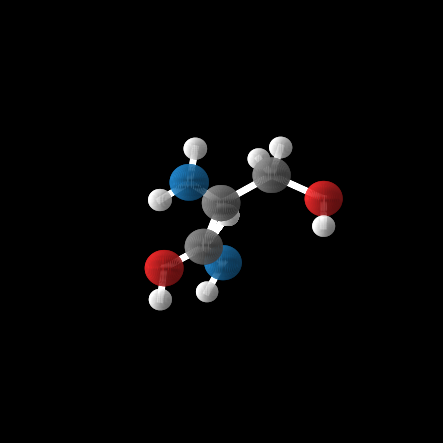}
    \end{subfigure}%
    \begin{subfigure}{0.125\textwidth}
        \centering\includegraphics[width=\textwidth]{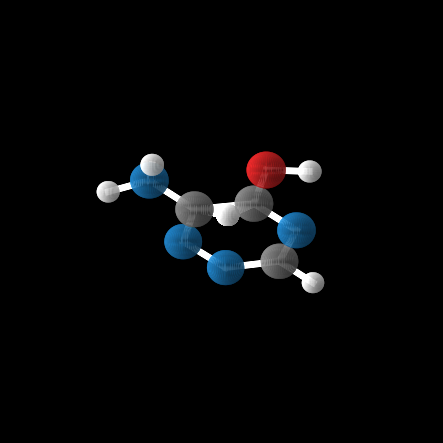}
    \end{subfigure}%
    \begin{subfigure}{0.125\textwidth}
        \centering\includegraphics[width=\textwidth]{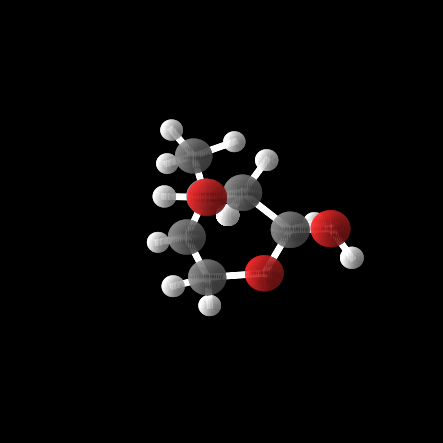}
    \end{subfigure}%
    \begin{subfigure}{0.125\textwidth}
        \centering\includegraphics[width=\textwidth]{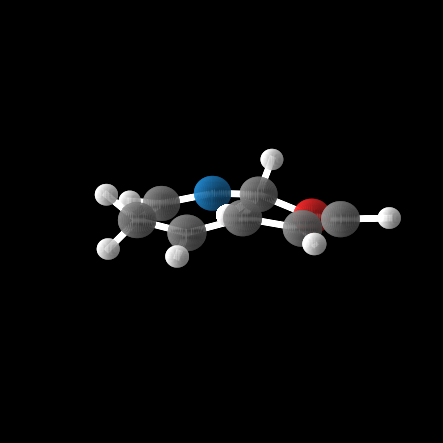}
    \end{subfigure}%
    \begin{subfigure}{0.125\textwidth}
        \centering\includegraphics[width=\textwidth]{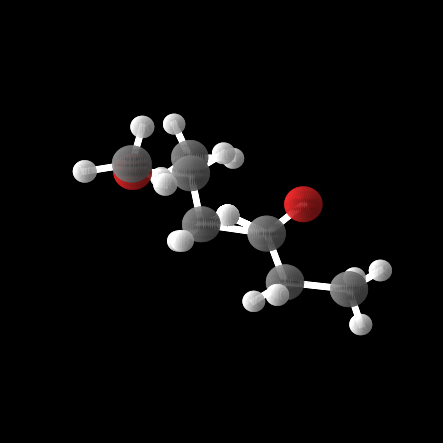}
    \end{subfigure}%
    \begin{subfigure}{0.125\textwidth}
        \centering\includegraphics[width=\textwidth]{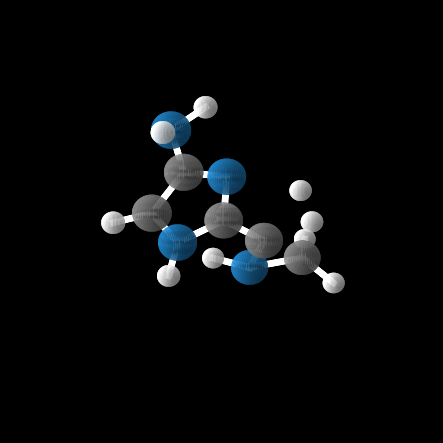}
    \end{subfigure}%
    \begin{subfigure}{0.125\textwidth}
        \centering\includegraphics[width=\textwidth]{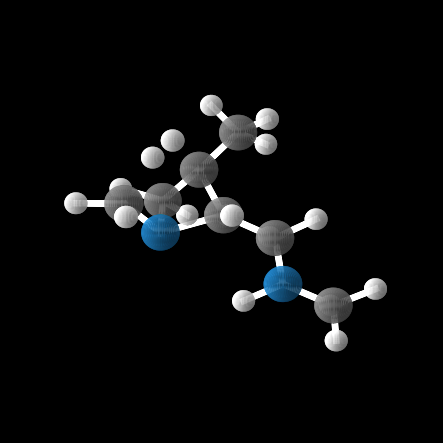}
    \end{subfigure}

    \begin{subfigure}{0.125\textwidth}
        \centering\includegraphics[width=\textwidth]{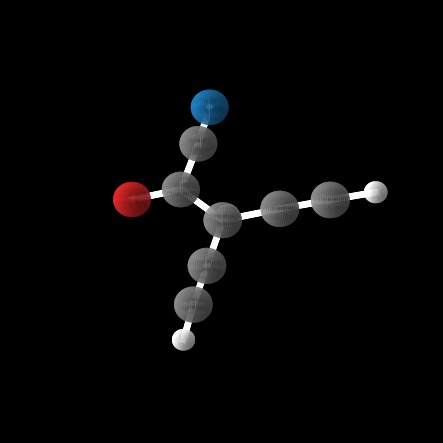}
        \caption*{48.43}
    \end{subfigure}%
    \begin{subfigure}{0.125\textwidth}
        \centering\includegraphics[width=\textwidth]{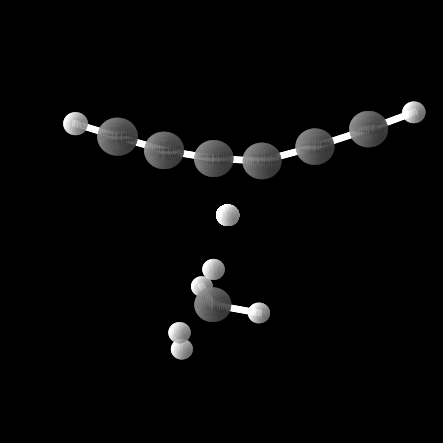}
        \caption*{53.64}
    \end{subfigure}%
    \begin{subfigure}{0.125\textwidth}
        \centering\includegraphics[width=\textwidth]{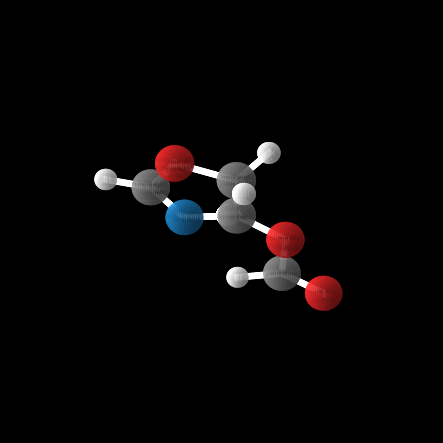}
        \caption*{60.05}
    \end{subfigure}%
    \begin{subfigure}{0.125\textwidth}
        \centering\includegraphics[width=\textwidth]{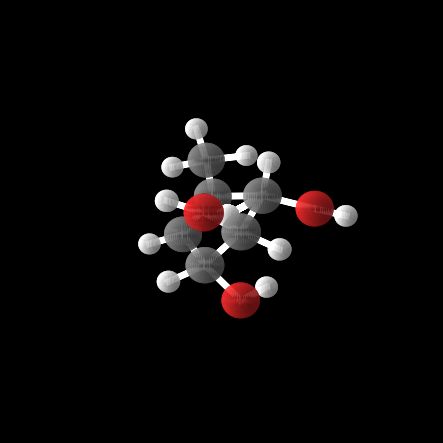}
        \caption*{73.95}
    \end{subfigure}%
    \begin{subfigure}{0.125\textwidth}
        \centering\includegraphics[width=\textwidth]{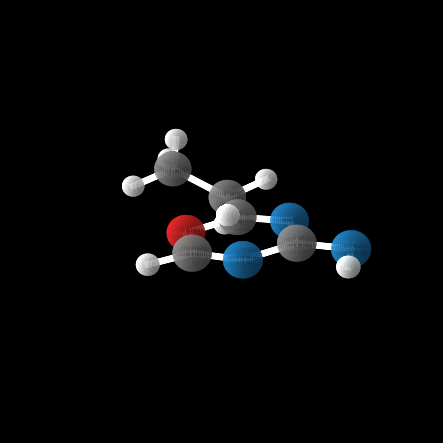}
        \caption*{78.02}
    \end{subfigure}%
    \begin{subfigure}{0.125\textwidth}
        \centering\includegraphics[width=\textwidth]{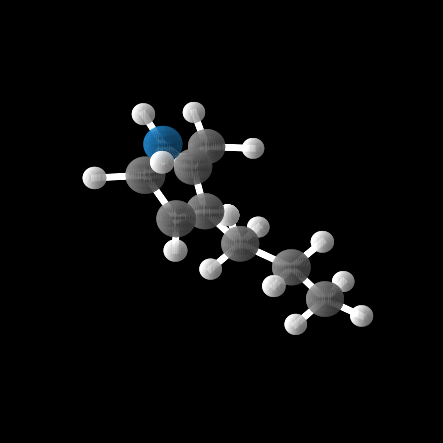}
        \caption*{82.75}
    \end{subfigure}%
    \begin{subfigure}{0.125\textwidth}
        \centering\includegraphics[width=\textwidth]{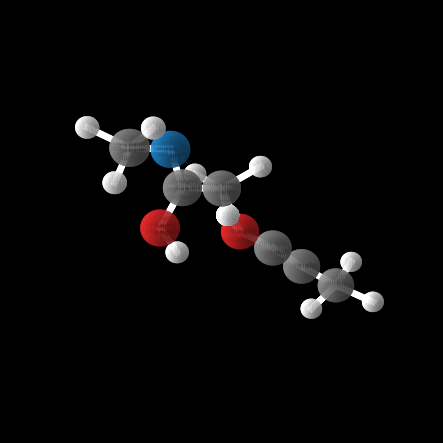}
        \caption*{88.45}
    \end{subfigure}%
    \begin{subfigure}{0.125\textwidth}
        \centering\includegraphics[width=\textwidth]{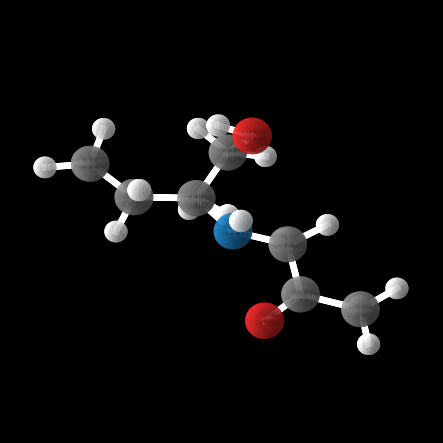}
        \caption*{91.30}
    \end{subfigure}
    \label{fig:molecules_vis}
    \caption{Qualitative visualization of controlled generated molecules for various polarizability $(\alpha)$ levels. Top row is generated using end-to-end guidance with a DTO scheme and the bottom row is generated using greedy guidance.}
\end{figure}

\subsection{Molecule generation for QM9}
We also illustrate the core ideas with some experiments in controllable molecule generation on the QM9 dataset \parencite{ruddigkeit2012enumeration}, a popular molecular dataset containing small molecules with up to 29 atoms.
Following \textcite{hoogeboom2022equivariant,ben-hamu2024dflow}, we perform the conditional generation of molecules with specified quantum chemical property values.
In particular, we target the following properties: polarizability $\alpha$, orbital energies $\varepsilon_{\mathrm{HOMO}}, \varepsilon_{\mathrm{LUMO}}$ and their gap $\Delta \varepsilon$, dipole moment $\mu$, and heat capacity $C_v$.
The property classifiers were trained following the methodology outlined in \textcite{hoogeboom2022equivariant}.
The underlying flow model is an unconditional equivariant flow matching model with \textit{conditional optimal transport} path \parencites[Section 4.7]{lipman2024flow-guide}[\cf][]{tong2023conditional,tong2024improving}, \ie, the EquiFM \parencite{song2023equivariant} model.
We solve the ODE with Euler's method and 50 discretization steps; further configuration details are reported in \cref{app:inv_images}.
Further details are provided in \cref{app:qm9}.

\paragraph{Metrics.}
To evaluate the guided generation we calculate the \textit{mean absolute error} (MAE) between the predicted property value of the generated molecule by the property classifier and the target property value \parencite{satorras2021en}.
Additionally in \cref{app:add_mol} we report the quality of the generated molecules by evaluating the atom stability (the percentage of atoms with correct valency) and molecule stability (the percentage of molecules where all atoms are stable).

\begin{table}[htpb]
    \centering
    \caption{Quantitative evaluation of conditional molecule generation. The MAE is reported for each molecule property (lower is better).}
    \label{tab:mol_results}
    \small

    \begin{tabular}{l ccc ccc}
        \toprule
        Property & $\alpha$ & $\Delta \varepsilon$ & $\varepsilon_{\mathrm{HOMO}}$ & $\varepsilon_{\mathrm{LUMO}}$ & $\mu$ & $C_v$\\
        Unit & $\mathrm{Bohr}^2$ & meV & meV & meV & D & $\frac{\mathrm{cal}}{\mathrm{K} \cdot \mathrm{mol}}$\\
        \midrule
        Greedy (Euler) & 11.282 & 1265 & 725 & 1092 & 1.559 & 6.469\\
        Greedy (midpoint) & 5.313 & 1196 & 599 & 1057 & 1.417 & 2.967\\
        Greedy (2-step Euler) & 5.667 & 1205 & 695 & 1222 & 1.491 & 2.767 \\
        Greedy (3-step Euler) & 5.098 & 1152 & 600 & 1152 & 1.384 & 3.229 \\
        Greedy (5-step Euler) & 4.177 & 1083 & 571 & 939 & 1.328 & 2.332 \\
        \hfill\\
        DTO (1-step) & 13.049 & $989 \times 10^{12}$ & 681 & 86.512 & 1.666 & 15.144\\
        DTO (2-step) & 6.113 & 1359 & 666 & 1199 & 1.533 & 3.757\\
        DTO (4-step) & 6.115 & 1294 & 668 & 1190 & 1.406 & 2.829\\
        DTO (8-step) & 4.549 & 1070 & 608 & 1078 & 1.247 & 2.594\\
        DTO (16-step) & 3.454 & 817 & 608 & 939 & 1.177 & 2.003\\
        DTO (32-step) & 2.912 & 750 & 410 & 666 & 0.721 & 1.566\\
        DTO (40-step) & 2.384 & 625 & 372 & 556 & 0.719 & 1.425\\
        DTO (50-step) & 1.404 & 401 & 176 & 373 & 0.372 & 0.866\\
        
        \midrule
        EquiFM & 9.525 & 1494 & 622 & 1523 & 1.628 & 6.689\\
        Lower bound & 0.10 & 64 & 39 & 46 & 0.043 & 0.040\\
        \bottomrule
    \end{tabular}
\end{table}

\paragraph{Results.}
In \cref{fig:molecules_vis} we present a visual comparison between molecules generated targeting different polarizability $\alpha$ values using a DTO end-to-end guidance scheme (essentially D-Flow) and the posterior guidance scheme.
Notice that as $\alpha$ increases the compactness of the molecules generated by a DTO scheme decreases.
This trend is less noticeable for the posterior guided samples.
We report quantitative results in \cref{tab:mol_results}.
We report the unguided EquiFM generated molecules as an upper bound and include the theoretical lower bounds from \textcite{ben-hamu2024dflow}.
It is here that we notice a sharp decrease in performance from using posterior guidance.
In particular the greedy (Euler) strategy is is highly unstable even performing worse than the unguided model on the $\alpha$ property.
The introduction of an additional step in the form of either midpoint or 2-step Euler does seem to improve performance; although the significance varies property to property.
We observe that the midpoint method seems to perform slightly better than the 2-step Euler.
We performed experiments with Ralston's third-order method and a fourth-order Runge-Kutta scheme but noticed significant instability in comparison to just taking more steps, we posit this is due to the large step size and that a hybrid scheme like that employed by \textcite{moufad2025variational} might be a reasonable solution to such problems, but ultimately we leave that question up to future work.
Moreover, we observe that increasing the number of steps generally improves performance with greedy (5-step Euler) performing the best among all the greedy guidance strategies.
The gradients of DTO (50-step) strategy are the \textit{ideal} and perfect gradients \wrt the flow model as the numerical solver takes 50 discretization steps and thus form the \textit{upper bound of performance}.
We see that as the DTO strategy incorporates more discretization steps the results converge to the upper bound along this particular design axis; in particular, we notice that the greedy strategy does well in the the regime of less discretization steps.

\section{Conclusion}
In this paper we present a unifying view of two different families of guided generation: end-to-end guidance and posterior guidance from the lens of a greedy algorithm.
We present numerous theoretical connections tying these two families together.
Our theoretical analysis shows that there might be some reason to believe that such a cheap approximation of the gradient can be reasonable for \textit{certain} tasks.
By exploiting the theoretical connections we created, we investigate guidance techniques which lie in between these two families giving rise to an exciting novel design space.
We then conduct several experiments on inverse image problems and on controlled molecule generation to illustrate this new design space.
We hope that our findings can help future researchers find the optimal spot between computational cost and accuracy of gradients for guidance problems.

\newrefcontext[sorting=nyt]
\printbibliography[heading=bibintoc]

\newpage
\section*{NeurIPS Paper Checklist}

\begin{enumerate}

\item {\bf Claims}
    \item[] Question: Do the main claims made in the abstract and introduction accurately reflect the paper's contributions and scope?
    \item[] Answer: \answerYes{} %
    \item[] Justification: The theoretical claims introduced in the abstract and introduction are verified in \cref{sec:greedy_perspective,sec:theory_guidance,sec:beyond_euler} and in the appendices.
        Likewise, the experimental claims made are verified in \cref{sec:experiments} and appropriate appendices.
    \item[] Guidelines:
    \begin{itemize}
        \item The answer NA means that the abstract and introduction do not include the claims made in the paper.
        \item The abstract and/or introduction should clearly state the claims made, including the contributions made in the paper and important assumptions and limitations. A No or NA answer to this question will not be perceived well by the reviewers. 
        \item The claims made should match theoretical and experimental results, and reflect how much the results can be expected to generalize to other settings. 
        \item It is fine to include aspirational goals as motivation as long as it is clear that these goals are not attained by the paper. 
    \end{itemize}

\item {\bf Limitations}
    \item[] Question: Does the paper discuss the limitations of the work performed by the authors?
    \item[] Answer: \answerYes{} %
    \item[] Justification: We address the limitations of this work in \cref{app:limitations}.
    \item[] Guidelines:
    \begin{itemize}
        \item The answer NA means that the paper has no limitation while the answer No means that the paper has limitations, but those are not discussed in the paper. 
        \item The authors are encouraged to create a separate "Limitations" section in their paper.
        \item The paper should point out any strong assumptions and how robust the results are to violations of these assumptions (e.g., independence assumptions, noiseless settings, model well-specification, asymptotic approximations only holding locally). The authors should reflect on how these assumptions might be violated in practice and what the implications would be.
        \item The authors should reflect on the scope of the claims made, e.g., if the approach was only tested on a few datasets or with a few runs. In general, empirical results often depend on implicit assumptions, which should be articulated.
        \item The authors should reflect on the factors that influence the performance of the approach. For example, a facial recognition algorithm may perform poorly when image resolution is low or images are taken in low lighting. Or a speech-to-text system might not be used reliably to provide closed captions for online lectures because it fails to handle technical jargon.
        \item The authors should discuss the computational efficiency of the proposed algorithms and how they scale with dataset size.
        \item If applicable, the authors should discuss possible limitations of their approach to address problems of privacy and fairness.
        \item While the authors might fear that complete honesty about limitations might be used by reviewers as grounds for rejection, a worse outcome might be that reviewers discover limitations that aren't acknowledged in the paper. The authors should use their best judgment and recognize that individual actions in favor of transparency play an important role in developing norms that preserve the integrity of the community. Reviewers will be specifically instructed to not penalize honesty concerning limitations.
    \end{itemize}

\item {\bf Theory assumptions and proofs}
    \item[] Question: For each theoretical result, does the paper provide the full set of assumptions and a complete (and correct) proof?
    \item[] Answer: \answerYes{} %
    \item[] Justification: The theoretical results from the main paper are proved with full assumptions in \cref{app:greedy_perspective,app:dynamics,app:beyond_euler}.
    \item[] Guidelines:
    \begin{itemize}
        \item The answer NA means that the paper does not include theoretical results. 
        \item All the theorems, formulas, and proofs in the paper should be numbered and cross-referenced.
        \item All assumptions should be clearly stated or referenced in the statement of any theorems.
        \item The proofs can either appear in the main paper or the supplemental material, but if they appear in the supplemental material, the authors are encouraged to provide a short proof sketch to provide intuition. 
        \item Inversely, any informal proof provided in the core of the paper should be complemented by formal proofs provided in appendix or supplemental material.
        \item Theorems and Lemmas that the proof relies upon should be properly referenced. 
    \end{itemize}

    \item {\bf Experimental result reproducibility}
    \item[] Question: Does the paper fully disclose all the information needed to reproduce the main experimental results of the paper to the extent that it affects the main claims and/or conclusions of the paper (regardless of whether the code and data are provided or not)?
    \item[] Answer: \answerYes{} %
    \item[] Justification: Yes, we provide details on our experimental procedures in \cref{app:impl}. Additionally, the random seeds we used are fixed in the configuration files in our code.
    \item[] Guidelines:
    \begin{itemize}
        \item The answer NA means that the paper does not include experiments.
        \item If the paper includes experiments, a No answer to this question will not be perceived well by the reviewers: Making the paper reproducible is important, regardless of whether the code and data are provided or not.
        \item If the contribution is a dataset and/or model, the authors should describe the steps taken to make their results reproducible or verifiable. 
        \item Depending on the contribution, reproducibility can be accomplished in various ways. For example, if the contribution is a novel architecture, describing the architecture fully might suffice, or if the contribution is a specific model and empirical evaluation, it may be necessary to either make it possible for others to replicate the model with the same dataset, or provide access to the model. In general. releasing code and data is often one good way to accomplish this, but reproducibility can also be provided via detailed instructions for how to replicate the results, access to a hosted model (e.g., in the case of a large language model), releasing of a model checkpoint, or other means that are appropriate to the research performed.
        \item While NeurIPS does not require releasing code, the conference does require all submissions to provide some reasonable avenue for reproducibility, which may depend on the nature of the contribution. For example
        \begin{enumerate}
            \item If the contribution is primarily a new algorithm, the paper should make it clear how to reproduce that algorithm.
            \item If the contribution is primarily a new model architecture, the paper should describe the architecture clearly and fully.
            \item If the contribution is a new model (e.g., a large language model), then there should either be a way to access this model for reproducing the results or a way to reproduce the model (e.g., with an open-source dataset or instructions for how to construct the dataset).
            \item We recognize that reproducibility may be tricky in some cases, in which case authors are welcome to describe the particular way they provide for reproducibility. In the case of closed-source models, it may be that access to the model is limited in some way (e.g., to registered users), but it should be possible for other researchers to have some path to reproducing or verifying the results.
        \end{enumerate}
    \end{itemize}

\item {\bf Open access to data and code}
    \item[] Question: Does the paper provide open access to the data and code, with sufficient instructions to faithfully reproduce the main experimental results, as described in supplemental material?
    \item[] Answer: \answerYes{} %
    \item[] Justification: The source code will be uploaded as a part of the supplementary material.
    \item[] Guidelines:
    \begin{itemize}
        \item The answer NA means that paper does not include experiments requiring code.
        \item Please see the NeurIPS code and data submission guidelines (\url{https://nips.cc/public/guides/CodeSubmissionPolicy}) for more details.
        \item While we encourage the release of code and data, we understand that this might not be possible, so “No” is an acceptable answer. Papers cannot be rejected simply for not including code, unless this is central to the contribution (e.g., for a new open-source benchmark).
        \item The instructions should contain the exact command and environment needed to run to reproduce the results. See the NeurIPS code and data submission guidelines (\url{https://nips.cc/public/guides/CodeSubmissionPolicy}) for more details.
        \item The authors should provide instructions on data access and preparation, including how to access the raw data, preprocessed data, intermediate data, and generated data, etc.
        \item The authors should provide scripts to reproduce all experimental results for the new proposed method and baselines. If only a subset of experiments are reproducible, they should state which ones are omitted from the script and why.
        \item At submission time, to preserve anonymity, the authors should release anonymized versions (if applicable).
        \item Providing as much information as possible in supplemental material (appended to the paper) is recommended, but including URLs to data and code is permitted.
    \end{itemize}

\item {\bf Experimental setting/details}
    \item[] Question: Does the paper specify all the training and test details (e.g., data splits, hyperparameters, how they were chosen, type of optimizer, etc.) necessary to understand the results?
    \item[] Answer: \answerYes{} %
    \item[] Justification: Yes, the full details for the experiments are provided in \cref{app:impl}.
    \item[] Guidelines:
    \begin{itemize}
        \item The answer NA means that the paper does not include experiments.
        \item The experimental setting should be presented in the core of the paper to a level of detail that is necessary to appreciate the results and make sense of them.
        \item The full details can be provided either with the code, in appendix, or as supplemental material.
    \end{itemize}

\item {\bf Experiment statistical significance}
    \item[] Question: Does the paper report error bars suitably and correctly defined or other appropriate information about the statistical significance of the experiments?
    \item[] Answer: \answerNo{} %
    \item[] Justification: Due to the heavy computational demands in running these experiments we did not report statistical significance which is inline with other works in this space.
    \item[] Guidelines:
    \begin{itemize}
        \item The answer NA means that the paper does not include experiments.
        \item The authors should answer "Yes" if the results are accompanied by error bars, confidence intervals, or statistical significance tests, at least for the experiments that support the main claims of the paper.
        \item The factors of variability that the error bars are capturing should be clearly stated (for example, train/test split, initialization, random drawing of some parameter, or overall run with given experimental conditions).
        \item The method for calculating the error bars should be explained (closed form formula, call to a library function, bootstrap, etc.)
        \item The assumptions made should be given (e.g., Normally distributed errors).
        \item It should be clear whether the error bar is the standard deviation or the standard error of the mean.
        \item It is OK to report 1-sigma error bars, but one should state it. The authors should preferably report a 2-sigma error bar than state that they have a 96\% CI, if the hypothesis of Normality of errors is not verified.
        \item For asymmetric distributions, the authors should be careful not to show in tables or figures symmetric error bars that would yield results that are out of range (e.g. negative error rates).
        \item If error bars are reported in tables or plots, The authors should explain in the text how they were calculated and reference the corresponding figures or tables in the text.
    \end{itemize}

\item {\bf Experiments compute resources}
    \item[] Question: For each experiment, does the paper provide sufficient information on the computer resources (type of compute workers, memory, time of execution) needed to reproduce the experiments?
    \item[] Answer: \answerYes{} %
    \item[] Justification: We provide the full details of the compute resources in \cref{app:hardware}.
    \item[] Guidelines:
    \begin{itemize}
        \item The answer NA means that the paper does not include experiments.
        \item The paper should indicate the type of compute workers CPU or GPU, internal cluster, or cloud provider, including relevant memory and storage.
        \item The paper should provide the amount of compute required for each of the individual experimental runs as well as estimate the total compute. 
        \item The paper should disclose whether the full research project required more compute than the experiments reported in the paper (e.g., preliminary or failed experiments that didn't make it into the paper). 
    \end{itemize}
    
\item {\bf Code of ethics}
    \item[] Question: Does the research conducted in the paper conform, in every respect, with the NeurIPS Code of Ethics \url{https://neurips.cc/public/EthicsGuidelines}?
    \item[] Answer: \answerYes{} %
    \item[] Justification: We conducted research conforming in every aspect with the NeurIPS Code of Ethics.
    \item[] Guidelines:
    \begin{itemize}
        \item The answer NA means that the authors have not reviewed the NeurIPS Code of Ethics.
        \item If the authors answer No, they should explain the special circumstances that require a deviation from the Code of Ethics.
        \item The authors should make sure to preserve anonymity (e.g., if there is a special consideration due to laws or regulations in their jurisdiction).
    \end{itemize}

\item {\bf Broader impacts}
    \item[] Question: Does the paper discuss both potential positive societal impacts and negative societal impacts of the work performed?
    \item[] Answer: \answerYes{} %
    \item[] Justification: Yes, we discuss such impacts in \cref{app:broader}.
    \item[] Guidelines:
    \begin{itemize}
        \item The answer NA means that there is no societal impact of the work performed.
        \item If the authors answer NA or No, they should explain why their work has no societal impact or why the paper does not address societal impact.
        \item Examples of negative societal impacts include potential malicious or unintended uses (e.g., disinformation, generating fake profiles, surveillance), fairness considerations (e.g., deployment of technologies that could make decisions that unfairly impact specific groups), privacy considerations, and security considerations.
        \item The conference expects that many papers will be foundational research and not tied to particular applications, let alone deployments. However, if there is a direct path to any negative applications, the authors should point it out. For example, it is legitimate to point out that an improvement in the quality of generative models could be used to generate deepfakes for disinformation. On the other hand, it is not needed to point out that a generic algorithm for optimizing neural networks could enable people to train models that generate Deepfakes faster.
        \item The authors should consider possible harms that could arise when the technology is being used as intended and functioning correctly, harms that could arise when the technology is being used as intended but gives incorrect results, and harms following from (intentional or unintentional) misuse of the technology.
        \item If there are negative societal impacts, the authors could also discuss possible mitigation strategies (e.g., gated release of models, providing defenses in addition to attacks, mechanisms for monitoring misuse, mechanisms to monitor how a system learns from feedback over time, improving the efficiency and accessibility of ML).
    \end{itemize}
    
\item {\bf Safeguards}
    \item[] Question: Does the paper describe safeguards that have been put in place for responsible release of data or models that have a high risk for misuse (e.g., pretrained language models, image generators, or scraped datasets)?
    \item[] Answer: \answerNA{} %
    \item[] Justification: While the theoretical concepts could be distilled for potential malicious downstream tasks it is not immediately applicable.
    \item[] Guidelines:
    \begin{itemize}
        \item The answer NA means that the paper poses no such risks.
        \item Released models that have a high risk for misuse or dual-use should be released with necessary safeguards to allow for controlled use of the model, for example by requiring that users adhere to usage guidelines or restrictions to access the model or implementing safety filters. 
        \item Datasets that have been scraped from the Internet could pose safety risks. The authors should describe how they avoided releasing unsafe images.
        \item We recognize that providing effective safeguards is challenging, and many papers do not require this, but we encourage authors to take this into account and make a best faith effort.
    \end{itemize}

\item {\bf Licenses for existing assets}
    \item[] Question: Are the creators or original owners of assets (e.g., code, data, models), used in the paper, properly credited and are the license and terms of use explicitly mentioned and properly respected?
    \item[] Answer: \answerYes{} %
    \item[] Justification: The details for the datasets are described in \cref{sec:experiments,app:impl}.
    \item[] Guidelines:
    \begin{itemize}
        \item The answer NA means that the paper does not use existing assets.
        \item The authors should cite the original paper that produced the code package or dataset.
        \item The authors should state which version of the asset is used and, if possible, include a URL.
        \item The name of the license (e.g., CC-BY 4.0) should be included for each asset.
        \item For scraped data from a particular source (e.g., website), the copyright and terms of service of that source should be provided.
        \item If assets are released, the license, copyright information, and terms of use in the package should be provided. For popular datasets, \url{paperswithcode.com/datasets} has curated licenses for some datasets. Their licensing guide can help determine the license of a dataset.
        \item For existing datasets that are re-packaged, both the original license and the license of the derived asset (if it has changed) should be provided.
        \item If this information is not available online, the authors are encouraged to reach out to the asset's creators.
    \end{itemize}

\item {\bf New assets}
    \item[] Question: Are new assets introduced in the paper well documented and is the documentation provided alongside the assets?
    \item[] Answer: \answerNA{} %
    \item[] Justification: No new assets are being released.
    \item[] Guidelines:
    \begin{itemize}
        \item The answer NA means that the paper does not release new assets.
        \item Researchers should communicate the details of the dataset/code/model as part of their submissions via structured templates. This includes details about training, license, limitations, etc. 
        \item The paper should discuss whether and how consent was obtained from people whose asset is used.
        \item At submission time, remember to anonymize your assets (if applicable). You can either create an anonymized URL or include an anonymized zip file.
    \end{itemize}

\item {\bf Crowdsourcing and research with human subjects}
    \item[] Question: For crowdsourcing experiments and research with human subjects, does the paper include the full text of instructions given to participants and screenshots, if applicable, as well as details about compensation (if any)? 
    \item[] Answer: \answerNA{} %
    \item[] Justification: We did not perform crowdsourcing. We used FFHQ, which contains subjects with human faces, but this is a publicly avaiable dataset.
    \item[] Guidelines:
    \begin{itemize}
        \item The answer NA means that the paper does not involve crowdsourcing nor research with human subjects.
        \item Including this information in the supplemental material is fine, but if the main contribution of the paper involves human subjects, then as much detail as possible should be included in the main paper. 
        \item According to the NeurIPS Code of Ethics, workers involved in data collection, curation, or other labor should be paid at least the minimum wage in the country of the data collector. 
    \end{itemize}

\item {\bf Institutional review board (IRB) approvals or equivalent for research with human subjects}
    \item[] Question: Does the paper describe potential risks incurred by study participants, whether such risks were disclosed to the subjects, and whether Institutional Review Board (IRB) approvals (or an equivalent approval/review based on the requirements of your country or institution) were obtained?
    \item[] Answer: \answerNA{} %
    \item[] Justification: Our research only used publicly available datasets and as such IRB approvals were not required.
    \item[] Guidelines:
    \begin{itemize}
        \item The answer NA means that the paper does not involve crowdsourcing nor research with human subjects.
        \item Depending on the country in which research is conducted, IRB approval (or equivalent) may be required for any human subjects research. If you obtained IRB approval, you should clearly state this in the paper. 
        \item We recognize that the procedures for this may vary significantly between institutions and locations, and we expect authors to adhere to the NeurIPS Code of Ethics and the guidelines for their institution. 
        \item For initial submissions, do not include any information that would break anonymity (if applicable), such as the institution conducting the review.
    \end{itemize}

\item {\bf Declaration of LLM usage}
    \item[] Question: Does the paper describe the usage of LLMs if it is an important, original, or non-standard component of the core methods in this research? Note that if the LLM is used only for writing, editing, or formatting purposes and does not impact the core methodology, scientific rigorousness, or originality of the research, declaration is not required.
    \item[] Answer: \answerNA{} %
    \item[] Justification: LLMs were not used in any part of this research.
    \item[] Guidelines:
    \begin{itemize}
        \item The answer NA means that the core method development in this research does not involve LLMs as any important, original, or non-standard components.
        \item Please refer to our LLM policy (\url{https://neurips.cc/Conferences/2025/LLM}) for what should or should not be described.
    \end{itemize}

\end{enumerate}

\newpage
\appendix
\crefalias{section}{appendix}

\section*{Organization of the appendix}
In \cref{app:related_works} we discuss previous approaches by exploring posterior guidance and end-to-end guidance in greater detail to provide a more comprehensive overview of how this greedy perspective connects these various works.
\cref{app:greedy_perspective} is devoted to the proofs and derivations from \cref{sec:greedy_perspective} in the main paper.
Likewise, \cref{app:dynamics,app:beyond_euler} is devoted to proofs and derivations from \cref{sec:theory_guidance,sec:beyond_euler} respectively.
In \cref{app:otd_issues} we discuss some important practical issues when using OTD for guidance, which we believe several to be useful background for the reader.
We provide some additional connections between posterior guidance and control signal optimization in \cref{app:control_signal_opt} that we were unable to include in the main paper.
\cref{app:inverse_problems} is devoted to providing a brief background on inverse problems.
Likewise, \cref{app:impl} is devoted to discussing the implementation details of the numerical experiments in \cref{sec:experiments} and providing a background for the experiments.
In \cref{app:add_results} we include additional results that we could not fit into the main paper.
Lastly, in \cref{app:disc} we discuss the limitations and broader impacts of this research.

\startcontents[]
\printcontents[]{l}{1}[3]{{\bfseries \large Appendices}}

\subsection*{Overview of theoretical results}
For convenience we provide a list of theorems to make navigating the theoretical results easier.
\listoftheorems[title={},numwidth={3em},ignoreall,onlynamed={theorem,proposition,lemma,corollary}]

\clearpage

\begin{figure*}[t]
    \centering

    \tikzset{
        basic/.style  = {draw, text width=20mm, font=\tiny, rectangle},
        root/.style = {basic, thin, align=center},
        tnode/.style = {basic, thin, align=center, font=\tiny\bfseries},
        xnode/.style = {basic, thin, align=left, text width=30mm},
        xnnode/.style = {basic, thin, align=left, text width=25mm},
        arrow/.style = {thick, dotted, shorten >=3, shorten <=3, ->}
    }

    \begin{forest} for tree={
        grow=east,
        growth parent anchor=west,
        parent anchor=east,
        child anchor=west,
        fork sep=6mm,
        l sep=12mm,
    },
    forked edges,
    [Training-free\\guided generation, root
        [Posterior\\guidance, tnode, name=posterior
            [\parencite{bansal2023universal}, xnnode]
            [\parencite{chung2023diffusion}, xnnode]
            [\parencite{wang2023zeroshot}, xnnode]
            [\parencite{kadkhodaie2021posterior}, xnnode]
            [\parencite{moufad2025variational}, xnnode]
            [\parencite{he2024manifold}, xnnode]
            [\parencite{yu2023freedom}, xnnode]
            [\parencite{song2023pseudoinverseguided}, xnnode]
            [\parencite{zhang2024improvingdiffusioninverseproblem}, xnnode]
        ]
        [End-to-end\\guidance, tnode, name=e2e
            [State\\optimization, tnode
                [\parencite{pan2024adjointdpm}, xnode]
                [\parencite{novack2024ditto}, xnode]
                [\parencite{karunratanakul2024optimizing}, xnode]
                [\parencite{clark2024directly}, xnode]
                [\parencite{blasingame2024adjointdeis}, xnode]
                [\parencite{ben-hamu2024dflow}, xnode]
                [\parencite{marion2024implicit}, xnode]
            ]
            [Control signal\\ optimization, tnode
                [\parencite{liu2023flowgrad}, xnode]
                [\parencite{wang2024training}, xnode]
            ]
        ]
    ]
    \draw[arrow] (e2e) to node[font=\scriptsize, anchor=west,align=center, yshift=4mm]{A greedy strategy} (posterior); 
    \end{forest}

    \caption[A taxonomy of training-free guided generation methods.]{A more detailed taxonomy of \textit{training-free guided generation} methods from \cref{fig:taxonomy_of_guided} from the main paper.}
    \label{fig:app:taxonomy_of_guided}
\end{figure*}

\section{Related works}
\label{app:related_works}
We provide a brief summary of previous work exploring either posterior guidance or end-to-end guidance strategies.
In \cref{fig:app:taxonomy_of_guided} we provide a more detailed taxonomy of training-free methods for gradient-based guided generation based on \cref{fig:taxonomy_of_guided} from the main paper.

\subsection{Posterior guidance}
Recent work in flow/diffusion models has explored the guidance using this strategy; we highlight a few notable examples.
Diffusion Posterior Sampling (DPS) \parencite{chung2023diffusion} is a guidance method that uses Tweedie's formula \parencite{stein1981estimation} to estimate the gradient of some guidance function defined in the output state \wrt the noisy state, \ie, $\ex[\bfX_1|\bfX_t = \bfx]$. 
Likewise, the work of \textcite{bansal2023universal,wang2023zeroshot,yu2023freedom,he2024manifold} explores similar concepts by employing Tweedie's formula for diffusion models.
Most of these works have explored using the SDE (or Markov chain) formulation of diffusion models rather than the ODE formulation, which is what we primarily focused on in our analysis.

\paragraph{Correcting the guidance trajectory.}
Several works have explored extensions to the DPS framework by using multiple steps of an SDE solver to correct \textit{errors} made by the guidance steps.
In particular, FreeDoM \parencite{yu2023freedom} explores the usage of a \textit{time-reversal} strategy repeated for a set number of times in each sampling step to correct possible guidance errors.
Likewise, recent work by \textcite{zhang2024improvingdiffusioninverseproblem} explored modeling Langevin dynamics on top of a diffusion ODE to correct measurement errors in inverse problems.
A significant number of the proposed methods which use posterior guidance arise from solving inverse problems \parencite[\cf][]{daras2024survey}.

\paragraph{Scheduled hyperparameters.}
Researchers realized that extra performance can be gained in such problems by scheduling hyperparameters like the learning rate (or guidance strength) at different timesteps in the numerical scheme \parencite{yu2023freedom,moufad2025variational}.

\paragraph{Beyond Euler.}
Recent work by \textcite{moufad2025variational} explores an extension to \parencite{chung2023diffusion} by using a two-step method to estimate the guidance gradient.
This is mostly closely related to the \textit{greedy (2-step Euler)} method from the main paper, although they use a stochastic sampling method, so it would be more akin to taking two Euler-Maruyama steps.

\subsection{End-to-end guidance}
Within the last year, many researchers have explored backpropagation through flow/diffusion models for controllable generation.
As mentioned in the main paper, the two main strategies for solving such a problem is a DTO or OTD scheme (\cf \cref{app:otd_issues}).

\paragraph{Discretize-then-optimize.}
FlowGrad proposed by \textcite{liu2023flowgrad} uses a DTO scheme to optimize an additional control signal (more details on this later) to perform guidance with flow models.
Although the analysis of \textcite{ben-hamu2024dflow} makes use of the continuous adjoint equations, in practice they use the \textit{generally} preferred approach of DTO with gradient checkpointing.\footnote{See \url{https://docs.kidger.site/diffrax/api/adjoints/} for an excellent summary of such design considerations and why DTO is generally preferable over OTD.}
Likewise, \textcite{novack2024ditto,karunratanakul2024optimizing,clark2024directly} all use gradient checkpointing with DTO to perform backpropagation through the flow/diffusion model.

\paragraph{Optimize-then-discretize.}
Another stream of work has explored the use of continuous adjoint equations to perform the backpropagation.
The advantage of such approaches is the $\mathcal O(1)$ memory cost, and we enumerate the drawbacks in \cref{app:otd_issues}, but suffice to say there are several.
To the best of our knowledge, the first work to explore this was \textcite{nie2022diffpure} which used OTD with SDEs for the adversarial purification task.
More general work came later by \textcite{pan2024adjointdpm,blasingame2024adjointdeis,ben-hamu2024dflow}.
More specifically, \textcite{pan2024adjointdpm,pan2023adjointsymplectic} explore bespoke solvers for the continuous adjoint equations of diffusion ODEs.
\textcite{blasingame2024adjointdeis} extends these works by developing bespoke solvers for diffusion ODEs and SDEs and performs more theoretical analysis of the problem in the SDE setting.
\textcite{marion2024implicit} explore using the continuous adjoint equations as a part of a larger bi-level optimization scheme for guided generation.
The work of \textcite{ben-hamu2024dflow} extends the analysis of continuous adjoint equations for diffusion models to flow-based models and provides an alternative perspective to the analysis performed in the earlier works.
Recent work by \textcite{wang2024training} explores an extension of \textcite{ben-hamu2024dflow} to Riemannian manifolds which incorporates a control signal to the vector field and optimizes both the solution state and \textit{co-state}, they call their approach OC-Flow.

Parallel to these works (conceptually) is the work of \textcite{wallace2023end} who uses EDICT \parencite{wallace2023edict}, an invertible formulation of diffusion models, to perform backpropagation through the diffusion model.
Although not presented or viewed this way in the original work, the later work by \textcite{blasingame2024adjointdeis} showed that this approach can be viewed as a specific discretization scheme of continuous adjoint equations.
We note that the EDICT solver, while reversible, is a zeroth-order solver and has poor convergence properties \parencite[\cf][]{wang2024belm}.

\paragraph{Control signal optimization.}
We discuss this in more detail in \cref{app:control_signal_opt}, but there are several works that explore the optimization of an additional control signal $\bfz(t)$ rather than the solution trajectory $\bfx(t)$; namely, \textcite{liu2023flowgrad,wang2024training}.

\section{A greedy perspective}
\label{app:greedy_perspective}
We present the proofs and derivations associated with \cref{sec:greedy_perspective}.

\subsection{Additional details on flow models}
\label{app:more_flow}
Applying this flow to the random variable $\bfX_0$ we define a \textit{continuous-time Markov process} $\{\bfX_t\}_{t \in [0,1]}$ with mapping $\bfX_t = \Phi_t(\bfX_0)$.
The \textit{goal}, then, is to learn a flow $\Phi_t$ such that $\bfX_1 = \Phi_1(\bfX_0) \sim q(\bfx)$.
This procedure amounts to learning a neural network parameterized vector field $\bfu^\theta \in \mathcal{C}^{1,r}([0,1] \times \R^d; \R^d)$; this learning procedure can be performed efficiently through a \textit{simulation-free} training process known as \textit{flow matching} \parencite{lipman2023flow} or more generally \textit{generator matching} \parencite{holderrieth2025generator}.

Throughout the rest of this paper we will assume a standard flow model trained to zero loss and we denote the parameterized flow model via $\Phi_t^\theta(\bfx)$.
We let $\Phi_{s,t}(\bfx) = (\Phi_t \circ \Phi_s^{-1})(\bfx)$ denote the flow from time $s$ to time $t$, $s, t \in [0, 1]$.

\paragraph{Affine probability paths.}
A special subset of flow models, are flows which model an \textit{affine probability path}, \ie, given a schedule $(\alpha_t, \sigma_t)$ the random process $\{\bfX_t\}$ is described via the affine equation
\begin{equation}
    \bfX_t = \alpha_t \bfX_1 + \sigma_t \bfX_0,
\end{equation}
where $\alpha_t, \sigma_t \in \C^\infty([0,1];[0,1])$ which satisfy
\begin{equation}
    \label{eq:schedule_boundaries}
    \alpha_0 = \sigma_1 = 0, \quad \alpha_1 = \sigma_0 = 1, \quad \forall t \in (0,1)\; [\dot\alpha_t > 0, \; \dot\sigma_t < 0].
\end{equation}
The \textit{marginal vector field} can then be expressed as the following conditional expectation:
\begin{equation}
    \label{eq:marginal_vec}
    \bfu_t(\bfx) = \ex[\dot\alpha_t \bfX_1 + \dot\sigma_t \bfX_0 | \bfX_t = \bfx].
\end{equation}
This \textit{nice} form of the marginal vector field enables use to rewrite the vector field in the forms of either source \parencite{ho2020denoising} or target \parencite{kingma2021variational} prediction as 
\begin{equation}
    \label{eq:vector_field_denoiser}
    \bfu_t(\bfx) = \underbrace{\frac{\dot\beta_t}{\beta_t}}_{= a_t} \bfx + \underbrace{\frac{\sigma_t\dot\alpha_t - \dot\sigma_t\alpha_t}{\beta_t}}_{= b_t} \bsf_t(\bfx),
\end{equation}
where $\beta_t = -\alpha_t$ for source prediction with $\bsf_t(\bfx) = \bfx_{0|t}(\bfx) = \ex[\bfX_0|\bfX_t = \bfx]$ and $\beta_t = \sigma_t$ for target prediction with $\bsf_t(\bfx) = \bfx_{1|t}(\bfx) = \ex[\bfX_1|\bfX_t  = \bfx]$; and $a_t, b_t$ are useful shorthands to be used later.

\begin{remark}
The probability flow ODE formulation of \textit{diffusion models} \parencite{song2021scorebased} is subsumed by flow models, and represents a model with an affine Gaussian probability paths (AGGP), \ie, $(\bfX_0, \bfX_1) \sim \pi_{0,1}(\bfx_0, \bfx_1) = p(\bfx_0)q(\bfx_1)$ with $p(\bfx) = \mathcal{N}(\bfx | \mathbf 0, \sigma^2 \mathbf I)$ \parencite{lipman2024flow-guide}.
Thus without loss of generality we consider flow models of affine probability paths.\footnote{Clearly, diffusion models which solve the reverse-time SDE are different and require a separate analysis.}
\end{remark}

\subsection{Assumptions}
Throughout the norm $\|\cdot\|$ corresponds to the Euclidean norm $\|\cdot\|_2$.
Additionally, we make the following (mild) regularity assumptions:
\begin{assumption}
    \label{assump:at_integrable}
    The function $a_t \coloneq \frac{\dot\sigma_t}{\sigma_t}$ is integrable in $[0, 1]$.
\end{assumption}
\begin{assumption}
    \label{assump:total_derivs}
    The total derivatives $\frac{\rmd^n}{\rmd \gamma^n}\left[\bfx_{1|\gamma}^\theta(\bfx)\right]$ exist and are continuous for $0 \leq n \leq k - 1$.
\end{assumption}

\cref{assump:at_integrable} is necessary for the simplification that we perform with exponential integrators and \textcite{ben-hamu2024dflow} make the same assumption in their analysis of the continuous adjoint equations for affine probability paths.
\cref{assump:total_derivs} is to ensure that we can take a Taylor expansion of $\bfx_{1|\gamma}^\theta(\bfx)$.

\subsection{Proof of \texorpdfstring{ \cref{prop:exact_sol_flow}}{Exact solution of affine probability paths}}
\label{proof:exact_sol_flow}

We restate \cref{prop:exact_sol_flow} below.
\begin{theorembox}\exactsolflow*\end{theorembox}

\begin{proof}
    Recall that we uniquely define a flow model through the vector field $\bfu \in \C^{1,1}([0,1]\times\R^d;\R^d)$.
    The vector field which models the affine conditional flow with schedule $(\alpha_t, \sigma_t)$, is defined as
    \begin{equation}
        \bfu_t^\theta(\bfx) = \ex[\dot \alpha_t \bfX_1 + \dot \sigma_t \bfX_0 | \bfX_t = \bfx].
    \end{equation}
    With some simple algebra, we can rewrite the vector field in terms of $\hat\bfx_{1|t}$,
    \begin{equation}
        \begin{aligned}
            \bfu_t^\theta(\bfx) &= a_t \bfx + b_t \bfx_{1|t}^\theta(\bfx),\\
            a_t &= \frac{\dot \sigma_t}{\sigma_t} \qquad b_t = \dot \alpha_t - \alpha_t \frac{\dot \sigma_t}{\sigma_t}.
        \end{aligned}
    \end{equation}
    Now using this definition we can rewrite the solution for $\bfx_t$ from $\bfx_s$ in terms of $\hat\bfx_{1|t}$,
    \begin{align}
        \bfx_t &= \bfx_s + \int_s^t \bfu_\tau^\theta(\bfx_\tau) \; \rmd \tau,\\
        \label{eq:app:exact_sol1}
        \bfx_t &= \bfx_s + \int_s^t a_\tau \bfx_\tau + b_\tau \bfx_{1|\tau}^\theta(\bfx_\tau)\; \rmd \tau.
    \end{align}
    Note the semi-linear form of the integral equation.
    We can exploit this structure using the technique of \textit{exponential integrators}, \parencite[see][]{lu2022dpmsolver,zhangfast,gonzalez2024seeds}, to simplify \cref{eq:app:exact_sol1}, under \cref{assump:at_integrable}, to
    \begin{equation}
        \label{eq:app:exact_sol2}
        \bfx_t = e^{\int_s^t a_u\;\rmd u}\bfx_s + \int_s^t e^{\int_\tau^t a_u\;\rmd u}b_\tau \bfx_{1|\tau}^\theta(\bfx_\tau)\;\rmd \tau.
    \end{equation}
    Now, the integrating factor simplifies quite nicely to
    \begin{equation}
        e^{\int_s^t a_u\;\rmd u} = e^{\int_s^t \frac{\dot\sigma_u}{\sigma_u}\;\rmd u} = e^{\int_{\sigma_s}^{\sigma_t} \frac 1\sigma \;\rmd \sigma} = \frac{\sigma_t}{\sigma_s},
    \end{equation}
    such that \cref{eq:app:exact_sol2} becomes
    \begin{equation}
        \label{eq:app:exact_sol3}
        \bfx_t = \frac{\sigma_t}{\sigma_s}\bfx_s + \sigma_t\int_s^t \frac{b_\tau}{\sigma_\tau} \bfx_{1|\tau}^\theta(\bfx_\tau)\;\rmd \tau.
    \end{equation}
    We can simplify $b_t / \sigma_t$ to find:
    \begin{equation}
        \label{eq:gamma_dot}
        \frac{b_t}{\sigma_t} = \frac{\dot\alpha_t \sigma_t - \alpha_t\dot\sigma_t}{\sigma_t^2} = \frac{\rmd}{\rmd t}\bigg(\frac{\alpha_t}{\sigma_t}\bigg) = \frac{\rmd}{\rmd t} \gamma_t,
    \end{equation}
    where $\gamma_t \coloneq \alpha_t / \sigma_t$, \ie, the signal-to-noise ratio.
    As such, we can rewrite \cref{eq:app:exact_sol3} with a change of variables $\bfx_\gamma = \bfx_{\gamma_t^{-1}(\gamma)} = \bfx_t$,
    \begin{align}
        \bfx_t &= \frac{\sigma_t}{\sigma_s}\bfx_s + \sigma_t\int_{\gamma_s}^{\gamma_t} \bfx_{1|\gamma}^\theta(\bfx_\gamma)\;\rmd \gamma,
    \end{align}
    concluding the proof.
\end{proof}

\subsection{Proof of \texorpdfstring{\cref{prop:greedy_is_explicit_euler}}{greedy as an explicit Euler scheme}}
\label{proof:greedy_is_explicit_euler}

\begin{theorembox}\greedyasexplicit*\end{theorembox}

\begin{proof}
From \cref{prop:exact_sol_flow} we see that using the target prediction model to estimate $\bfx_1$ is akin to taking a first-order approximation of the flow.
More specifically, under \cref{assump:total_derivs} we can construct a $(k-1)$-th Taylor expansion of \cref{eq:exact_sol_flow} with:
\begin{align}
    \bfx_t &= \frac{\sigma_t}{\sigma_s}\bfx_s + \sigma_t \sum_{n=0}^{k-1} \frac{\rmd^n}{\rmd \gamma^n}\bigg [ \bfx^\theta_{1|\gamma}(\bfx_\gamma)\bigg]_{\gamma = \gamma_s} \int_{\gamma_s}^{\gamma_t} \frac{(\gamma - \gamma_s)^n}{n!}\;\rmd \gamma + \mathcal{O}(h^{k+1}),\\
           &= \frac{\sigma_t}{\sigma_s}\bfx_s + \sigma_t \sum_{n=0}^{k-1} \frac{\rmd^n}{\rmd \gamma^n}\bigg [ \bfx^\theta_{1|\gamma}(\bfx_\gamma)\bigg]_{\gamma = \gamma_s} \frac{h^{n+1}}{(n+1)!}+ \mathcal{O}(h^{k+1}),
\end{align}
where $h \coloneq \gamma_t - \gamma_s$ is the step size.
Then it follows that for $k=1$ the first-order discretization of the flow, omitting high-order error terms becomes,
\begin{equation}
    \bfx_t \approx \tilde \bfx_t = \frac{\sigma_t}{\sigma_s}\bfx_s + (\alpha_t + \frac{\sigma_t\alpha_s}{\sigma_s})\bfx^\theta_{1|s}(\bfx_s).
\end{equation}
In the limit as $t \to 1$ we have $\tilde\bfx_t =\bfx^\theta_{1|s}(\bfx_s)$.\footnote{Note that despite $\sigma_t \to 0$ the asymptotic behavior is well-defined \parencite[see][]{ben-hamu2024dflow}.}
Thus, the greedy gradient is a DTO scheme with an explicit Euler discretization with step size $h = \gamma_1 - \gamma_t$.
\end{proof}

\subsection{Proof of \texorpdfstring{ \cref{thm:greedy_is_implicit_euler}}{greedy as an implicit Euler scheme}}
\label{proof:greedy_is_implicit_euler}
We restate \cref{thm:greedy_is_implicit_euler} below.

\begin{theorembox}\greedyasimplicit*\end{theorembox}

For clarity we restate the definition of the continuous adjoint equations.
Let $\bfu_\theta \in \mathcal{C}^{1,1}([0,1] \times \R^d;\R^d)$ be a model that models the vector field of some ODE and be Lipschitz continuous in its second argument. Let $\bfx: [0,1] \to \R^d$ be the solution to the ODE with the initial condition $\bfx_0 \in \R^d$, $\dot \bfx_t = \bfu_\theta(t, \bfx_t)$.
For some scalar-valued loss function $\mathcal{L} \in \mathcal{C}^2(\R^d)$ in $\bfx_1$, let $\bfa_\bfx \coloneq \partial\mathcal{L}/\partial \bfx_t$ denote the gradient.
Then $\bfa_\bfx$ and related quantity $\bfa_\theta \coloneq \partial\mathcal{L}/\partial \theta$ can be found by solving an augmented ODE of the form,
\begin{equation}
    \begin{aligned}
    \bfa_\bfx(1) &= \frac{\partial \mathcal{L}}{\partial \bfx_1}, \quad &&\frac{\rmd \bfa_\bfx}{\rmd t}(t) = -\bfa_\bfx(t)^\top \frac{\partial \bfu_\theta}{\partial \bfx}(t, \bfx_t),\\
    \bfa_\theta(1) &= \mathbf 0, \quad &&\frac{\rmd \bfa_\theta}{\rmd t}(t) = -\bfa_\bfx(t)^\top \frac{\partial \bfu_\theta}{\partial \theta}(t, \bfx_t).
    \end{aligned}
\end{equation}
Now we present the proof.

\begin{proof}
The adjoint state can be simplified by rewriting the vector field in terms of the target prediction model to find
\begin{equation}
    \frac{\rmd \bfa_\bfx}{\rmd t}(t) = -a_t\bfa_\bfx(t) - b_t \bfa_\bfx(t)^\top \frac{\partial \bfx_{1|t}^\theta(\bfx_t)}{\partial \bfx_t}.
\end{equation}
We can express this \textit{backwards-in-time} ODE as an integral equation in the form of
\begin{align}
    \bfa_\bfx(s) &= \bfa_\bfx(t) - \int_t^s a_\tau \bfa_\bfx(t) + b_\tau \bfa_\bfx(\tau)^\top \frac{\partial \bfx_{1|\tau}^\theta(\bfx_\tau)}{\bfx_\tau}\;\rmd \tau,\nonumber\\
    &= \bfa_\bfx(t) + \int_s^t a_\tau \bfa_\bfx(t) + b_\tau \bfa_\bfx(\tau)^\top \frac{\partial \bfx_{1|\tau}^\theta(\bfx_\tau)}{\partial \bfx_\tau}\;\rmd \tau. \qquad \textrm{(time-reversal)}
\end{align}
Using the technique of exponential integrators we rewrite the integral as
\begin{align}
    \bfa_\bfx(s) &= e^{\int_s^t a_u\;\rmd u}\bfa_\bfx(t) + \int_s^t e^{\int_\tau^t a_u\;\rmd u}b_\tau\bfa_\bfx(\tau)^\top\frac{\partial \bfx_{1|\tau}^\theta(\bfx_\tau)}{\partial \bfx_\tau}\;\rmd \tau,\nonumber\\
    &= \frac{\sigma_t}{\sigma_s}\bfa_\bfx(t) + \sigma_t\int_s^t \frac{b_\tau}{\sigma_\tau}\bfa_\bfx(\tau)^\top\frac{\partial \bfx_{1|\tau}^\theta(\bfx_\tau)}{\partial \bfx_\tau}\;\rmd \tau,\nonumber\\
    &= \frac{\sigma_t}{\sigma_s}\bfa_\bfx(t) + \sigma_t\int_{\gamma_s}^{\gamma_t} \bfa_\bfx(\gamma)^\top\frac{\partial \bfx_{1|\gamma}^\theta(\bfx_\gamma)}{\partial \bfx_\gamma}\;\rmd \gamma.
\end{align}
By \cref{assump:total_derivs} it follows that the vector-Jacobian product has $(k-1)$-th total derivatives, allowing us to define a first-order Taylor expansion around $\gamma_s$:
\begin{equation}
    \bfa_\bfx(s) = \frac{\sigma_t}{\sigma_s}\bfa_\bfx(t) + (\alpha_t - \frac{\sigma_t}{\sigma_s}\alpha_s)\bfa_\bfx(s)^\top\frac{\partial\hat\bfx_{1|s}(\bfx_s)}{\partial \bfx_s} + \mathcal{O}(h^2).
\end{equation}
Thus, the first-order approximation of the adjoint state at time $t$ with a step size of $h = \gamma_1 - \gamma_t$ is the implicit equation
\begin{equation}
    \bfa_\bfx(t) = \bfa_\bfx(t)^\top\frac{\partial\hat\bfx_{1|t}(\bfx_t)}{\partial \bfx_t}.
\end{equation}
Now to solve the implicit equation we can use the fixed-point iteration method. Let $\bfa_\bfx(t)^{(0)} = \bfa_\bfx(1)$, then the first iteration has
\begin{equation}
    \bfa_\bfx(t)^{(1)} = \bfa_\bfx(1)^\top\frac{\partial\hat\bfx_{1|t}(\bfx_t)}{\partial \bfx_t} = \nabla_{\bfx_t} \mathcal{L}(\hat\bfx_{1|t}(\bfx_t)).
\end{equation}
Thus, we have shown that the greedy gradients are equivalent to the first iteration of an implicit Euler discretization of the continuous adjoint equations.

\end{proof}

\section{Dynamics of guidance}
\label{app:dynamics}
In this section we detail some of the formalisms omitted in the main paper concerning the dynamics of the gradient flow and greedy gradients.

We begin by re-establishing some useful prior results.
\textcite[Proposition 4.1]{ben-hamu2024dflow} showed that the gradient of the target prediction model is proportional to the variance of the random variable defined by $p_{1|t}(\bfx_1|\bfx)$, we restate their result below.
\begin{lemma}[Gradient of target prediction model]
    \label{prop:gradient_target_pred}
    For affine Gaussian probability paths, the gradient of the target prediction model $\bfx_{1|t}^\theta(\bfx)$ \wrt $\bfx$ is proportional to the variance of $p_{1|t}(\bfx_1|\bfx)$, \ie,
    \begin{equation}
        \nabla_\bfx \bfx_{1|t}^\theta(\bfx) = \frac{\alpha_t}{\sigma_t^2} \var_{1|t}(\bfx),
    \end{equation}
    where
    \begin{equation}
        \var_{1|t}(\bfx) = \ex_{p_{1|t}(\bfx_1|\bfx)} \left[(\bfx_1 - \bfx_{1|t}^\theta(\bfx)) (\bfx_1 - \bfx_{1|t}^\theta(\bfx))^\top\right].
    \end{equation}
\end{lemma}

\begin{remark}
    This can be written more generally in terms of the (pushforward) differential $D_x \bfx_{1|t}^\theta(\bfx)$ where the underlying spaces are smooth manifolds and $\bfx_{1|t}^\theta$ is a smooth map between them \parencite{ben-hamu2024dflow}.
    In this section, we only consider flow models defined in Euclidean spaces, and so we opt not to elaborate on this generalization.
\end{remark}

We restate a well-known result below in \cref{lemma:forward_sensitivities_ode} regarding the continuous-time analogue to forward-mode autodifferentiation, or in other words, forward sensitivity.
\begin{lemma}[Dynamics of Jacobian matrices for flows]
    \label{lemma:forward_sensitivities_ode}
    Let $\bfx_0 \in \R^d$ and let $\bsf \in \C^{1,1}([0,T] \times \R^d;\R^d)$ be uniformly Lipschitz in $\bfx$.
    Let $\bfx: [0,T] \to \R^d$ be the unique solution to
    \begin{equation}
        \label{eq:app:proto_ode_lemma}
        \bfx(0) = \bfx_0,\qquad \frac{\rmd \bfx}{\rmd t}(t) = \bsf(t, \bfx(t)).
    \end{equation}
    Let $\Phi_{s,t}(\bfx)$, $s, t \in [0,T]$ denote the flow associated with \cref{eq:app:proto_ode_lemma}.
    Then let $\bm J_s(t) \coloneq \nabla_\bfx \Phi_{s,t}(\bfx)$ denote the Jacobian matrices, where $\bm J_s: [s, T] \to \R^{d \times d}$ solve the differential equation
    \begin{equation}
        \bm J_s(s) = \bm I, \qquad \frac{\rmd \bm J_s}{\rmd t}(t) = \nabla_\bfx \bsf(t, \Phi_{s,t}(\bfx(s))) \bm J_s(t),
    \end{equation}
    where $\nabla_\bfx \bsf(t, \cdot)$ refers to the gradient \wrt the second argument.
\end{lemma}

\begin{remark}
    This result is well known and has been extended to \textit{controlled differential equations} \parencite[Theorem 4.4]{friz2010multidimensional} and \textit{rough differential equations} \parencite[Theorem 11.3]{friz2010multidimensional}.
    \textcite[Theorem 5.8]{kidger_thesis} discusses this result for neural ODEs.
\end{remark}

\subsection{Proof of \texorpdfstring{\cref{thm:jacobians_aggp}}{Jacobian matrices of affine Gaussian probability paths}}
\label{proof:jacobians_aggp}

We restate \cref{thm:jacobians_aggp} below.
\begin{theorembox}\jacobiansaggp*\end{theorembox}

This proof follows a similar technique to that used by \textcite{blasingame2024adjointdeis} to simplify adjoint equations for diffusion models using exponential integrators.

\begin{proof}
Now recall \cref{lemma:forward_sensitivities_ode} which discusses the dynamics of Jacobian matrices for flows, rewriting this as an integral equation yields:
\begin{equation}
    \label{eq:adjoint_flow_sensitivities}
    \nabla_\bfx \Phi_{s,t}^\theta(\bfx) = \bm I + \int_s^t \nabla_{\bfx_u} \bfu_u^\theta(\Phi_{s,u}^\theta(\bfx)) \nabla_{\bfx} \Phi_{s,u}^\theta(\bfx)\;\rmd u.
\end{equation}
Now recall the definition of the marginal vector field in terms of the target prediction model (\cf \cref{eq:vector_field_denoiser}) which we use to rewrite \cref{eq:adjoint_flow_sensitivities} as
\begin{align}
    \nabla_\bfx \Phi_{s,t}^\theta(\bfx) &= \bm I + \int_s^t \nabla_{\bfx_u} a_u\Phi_{s,u}^\theta(\bfx) \nabla_{\bfx} \Phi_{s,u}^\theta(\bfx) + \nabla_{\bfx_u} b_u \bfx_{1|u}^\theta(\Phi_{s,u}^\theta(\bfx)) \nabla_\bfx \Phi_{s,u}^\theta(\bfx)\;\rmd u,\nonumber\\
    \label{eq:adjoint_flow_sensitivities2}
                                        &\stackrel{(i)}= \bm I + \int_s^t a_u \nabla_\bfx\Phi_{s,u}^\theta(\bfx) + b_u\nabla_{\bfx_u} \bfx_{1|u}^\theta(\Phi_{s,u}^\theta(\bfx)) \nabla_\bfx \Phi_{s,u}^\theta(\bfx)\;\rmd u,
\end{align}
where (i) holds by $\nabla_{\bfx_u} \Phi_{s,u}^\theta(\bfx) = \bm I$.
Next we can make use of the popular technique of \textit{exponential integrators} to simplify \cref{eq:adjoint_flow_sensitivities} in combination with \cref{eq:vector_field_denoiser}.
Thus, the integral equation in \cref{eq:adjoint_flow_sensitivities2} becomes
\begin{equation}
    \label{eq:adjoint_step1}
    \nabla_\bfx \Phi_{s,t}^\theta(\bfx) = \Lambda_a(s,t)\bm I + \int_s^t \Lambda_a(u,t)b_u\nabla_{\bfx_u} \bfx_{1|u}^\theta(\Phi_{s,u}^\theta(\bfx))\nabla_{\bfx} \Phi_{s,u}^\theta(\bfx)\;\rmd u,
\end{equation}
where $\Lambda_a(s,t) \coloneq \exp \int_s^t a_u\; \rmd u$ is the integrating factor.
This simplifies to $\Lambda_a(s,t) = \sigma_t / \sigma_s$.
Using this, \cref{eq:adjoint_step1} can be simplified to 
\begin{equation}
    \label{eq:adjoint_step2}
    \nabla_\bfx \Phi_{s,t}^\theta(\bfx) = \frac{\sigma_t}{\sigma_s}\bm I + \sigma_t\int_s^t \frac{b_u}{\sigma_u}\nabla_{\bfx_u} \bfx_{1|u}^\theta(\Phi_{s,u}^\theta(\bfx))\nabla_{\bfx} \Phi_{s,u}^\theta(\bfx)\;\rmd u.
\end{equation}
Now we can apply \cref{prop:gradient_target_pred} to further simplify \cref{eq:adjoint_step2} to find
\begin{equation}
    \label{eq:adjoint_step3}
    \nabla_\bfx \Phi_{s,t}^\theta(\bfx) = \frac{\sigma_t}{\sigma_s}\bm I + \sigma_t\int_s^t \frac{\alpha_u}{\sigma_u^3}b_u \var_{1|u}(\Phi_{s,u}^\theta(\bfx))\nabla_{\bfx} \Phi_{s,u}^\theta(\bfx)\;\rmd u.
\end{equation}
Next we simplify the coefficient $\alpha_u b_u / \sigma_u^3$ in the integral term.
Let $\gamma_t \coloneq \alpha_t / \sigma_t$ equal the signal-to-noise-ratio.
Then we observe
\begin{align}
    b_t \frac{\alpha_t}{\sigma_t^3} &= \left(\dot\alpha_t - \alpha_t \frac{\dot\sigma_t}{\sigma_t}\right) \frac{\alpha_t}{\sigma_t^3},\nonumber\\
                                    &= \frac{\dot\alpha_t\sigma_t - \dot\sigma_t\alpha_t}{\sigma_t^3} \frac{\alpha_t}{\sigma_t^2},\nonumber\\
                                    &\stackrel{(i)}=  \frac{\rmd}{\rmd t}\left[\frac{\alpha_t}{\sigma_t}\right] \frac{\alpha_t}{\sigma_t} \frac{1}{\sigma_t},\nonumber\\
                                    &\stackrel{(ii)}= \dot\gamma_t \frac{\gamma_t}{\sigma_t},
\end{align}
where (i) holds by the quotient rule and (ii) holds by definition of $\gamma_t$.
Using this simplification we can perform a change-of-variables to simplify the gradient resulting in
\begin{equation}
    \nabla_\bfx \Phi_{s,t}^\theta(\bfx) = \frac{\sigma_t}{\sigma_s}\bm I + \sigma_t\int_s^t \dot\gamma_u \frac{\gamma_u}{\sigma_u} \var_{1|u}(\Phi_{s,u}^\theta(\bfx))\nabla_{\bfx} \Phi_{s,u}^\theta(\bfx)\;\rmd u.
\end{equation}
\end{proof}

\begin{remark}
    \label{remark:diff_in_jacobian_thms}
    Readers familiar with the work of \textcite{ben-hamu2024dflow} may notice some similarities between our result \cref{thm:jacobians_aggp} and \textcite[Theorem 4.2]{ben-hamu2024dflow}.
    The difference between the two is that the former is a simplified integral equation; whereas, the latter is the exact solution and no longer requires solving an ODE.
    However, this later solution does require solving a time-ordered exponential which requires a formal truncated series expansion, \eg, Magnus expansion.
\end{remark}

\cref{thm:jacobians_aggp} is closely related to \textcite[Theorem 4.2]{ben-hamu2024dflow} which we restate below within the context of our notational conventions.\footnote{With abuse of notation let $\dot\gamma_t^2$ denote the time derivative of $\gamma_t^2$.}
\begin{theorem}
    \label{thm:benhamu_jacobians}
    For the standard affine Gaussian probability path, the differential of $\Phi_{0,1}^\theta(\bfx)$ as of function of $\bfx$ is
    \begin{equation}
        \nabla_{\bfx} \Phi_{0,1}^\theta(\bfx) = \sigma_1 \mathcal T \exp \left[\int_0^1 \frac 12 \dot\gamma_t^2 \var_{1|t}(\bfx) \;\rmd t\right],
    \end{equation}
    where $\mathcal T \exp$ denotes the time-ordered exponential.
\end{theorem}

The time-ordered exponential\footnote{This is closely related to the \textit{Peano-Baker series} \parencite[see][Section 7.5]{frazer1938elementary}.} \parencite{grossman1972non} is defined as
\begin{equation}
    \begin{aligned}
        \mathcal T \exp \left[\int_t^1 \bfA(s)\; \rmd s\right] &= \sum_{n=0}^\infty \frac{1}{n!} \int_t^1 \rmd s_1 \cdots \int_t^1 \rmd s_n\;\; \mathcal T \{\bfA(s_1) \ldots \bfA(s_n)\},\\
                                                               &= \sum_{n=0}^\infty \int_t^1 \rmd s_1 \int_t^{s_1} \rmd s_2 \cdots \int_t^{s_{n-1}} \rmd s_n\;\; \bfA(s_1) \bfA(s_2)\ldots \bfA(s_n),
    \end{aligned}
\end{equation}
and the solution can be found the Dyson series \parencite{sakurai2020modern} or Magnus expansion \parencite{magnus1954exponential}, which are truncated in practice.
The meta-operator $\mathcal T$ denotes the time-ordering \parencite{dyson1949radiation}, \eg, consider the time-ordering of two operators $\bm A, \bm B$:
\begin{equation}
    \mathcal T \{\bfA(s_1)\bm B(s_2)\} \coloneq \left \{ 
        \begin{aligned}
            &\bfA(s_1)\bm B(s_2) \qquad &\textrm{if } s_1 > s_2,\\
            \pm &\bm B(s_2)\bm A(s_1) & \textrm{otherwise}.
        \end{aligned}
    \right.
\end{equation}
For more details we refer the reader to \textcite{weinberg1995quantum}.

\subsection{Dynamics of gradient guidance}
\label{app:dyn_grad_guidance}

We state this more formally below in \cref{prop:dyn_grad_guidance}.
\begin{theorembox}
    \begin{restatable}[Dynamics of gradient guidance]{proposition}{dyngradguidance}
        \label{prop:dyn_grad_guidance}
        Consider the standard affine Gaussian probability paths model trained to zero loss.
        The Gateaux differential of $\bfx$ at some time $t \in [0, 1]$ in the direction of the gradient $\nabla_\bfx \mathcal L\left(\Phi_{t,1}^\theta(\bfx)\right)$ is given by
        \begin{equation}
            \label{eq:sensitivity_gradient}
            \delta_\bfx \Phi_{t,1}^\theta(\bfx) = -\nabla_\bfx\Phi_{t,1}^\theta(\bfx) \nabla_\bfx\Phi_{t,1}^\theta(\bfx)^\top \nabla_{\bfx_1} \mathcal L(\bfx_1).
        \end{equation}
    \end{restatable}
\end{theorembox}

Thus the behavior of $\bfx_1$ when guided by $\mathcal L$ is determined by the operator $\nabla_\bfx \Phi_{t,1}^\theta(\bfx)$ which iteratively projects the gradient of the loss function by the covariance matrix $\var_{1|t}(\bfx)$.
Put another way:
\begin{standoutbox}
    Performing gradient guidance with $\mathcal L$ at time $t < 1$ amounts to guidance which follows the target distribution $p(\bfX_1)$ by projecting $\nabla_{\bfx_1} \mathcal L(\bfx_1)$ onto to the target distribution via the local covariance matrix.
\end{standoutbox}

It is for this reason that it is undesirable to simply perform guidance in the data space as we are likely to deviate from this target distribution.
From \cref{eq:sensitivity_gradient} we know that applying the gradient at earlier timesteps causes the initial gradient $\nabla_{\bfx_1}\mathcal L(\bfx_1)$ to be projected into high-variance directions of the target distribution causing the guided sample to stay closer to the true target distribution.

The next question is: how does $\bfx_1$ change when $\bfx$ is updated with our greedy guidance strategy?

\subsection{Proof of \texorpdfstring{\cref{prop:dyn_grad_guidance}}{dynamics of gradient guidance}}
\label{proof:dyn_grad_guidance}

We restate \cref{prop:dyn_grad_guidance} below.

\begin{theorembox}\dyngradguidance*\end{theorembox}

\begin{proof}
    This can be shown from a straightforward derivation:
    \begin{align}
        \delta_\bfx \Phi_{t,1}^\theta(\bfx) &\stackrel{(i)}= \frac{\rmd}{\rmd \eta} \bigg |_{\eta = 0} \Phi_{t,1}^\theta\left(\bfx - \eta \nabla_\bfx \mathcal L\left(\Phi_{t,1}^\theta(\bfx)\right)\right),\nonumber\\
                                               &\stackrel{(ii)}= -\nabla_\bfx\Phi_{t,1}^\theta\left(\bfx - \eta \nabla_\bfx \mathcal L\left(\Phi_{t,1}^\theta(\bfx)\right)\right) \nabla_\bfx \mathcal L\left(\Phi_{t,1}^\theta(\bfx)\right)\bigg |_{\eta = 0}, \nonumber\\
                                               &= -\nabla_\bfx\Phi_{t,1}^\theta(\bfx) \nabla_\bfx \mathcal L\left(\Phi_{t,1}^\theta(\bfx)\right), \nonumber\\
                                               &\stackrel{(iii)}= -\nabla_\bfx\Phi_{t,1}^\theta(\bfx) \nabla_\bfx\Phi_{t,1}^\theta(\bfx)^\top \nabla_{\bfx_1} \mathcal L(\bfx_1),
    \end{align}
    where (i) holds by the definition of the Gateaux differential, (ii) holds by the chain rule, and (iii) holds by a substitution of \cref{eq:chain_rule_gradient_flow} with the simplification of $\bfx_1 = \Phi_{t,1}^\theta(\bfx)$.
\end{proof}

\subsection{Proof of \texorpdfstring{\cref{prop:dyn_greedy_guidance}}{dynamics of greedy gradient guidance}}
\label{proof:dyn_greedy_guidance}

We restate \cref{prop:dyn_greedy_guidance} below.

\begin{theorembox}\dyngreedyguidance*\end{theorembox}
\begin{proof}
    This can be shown from a straightforward derivation:
    \begin{align}
        \delta_\bfx^{\mathcal G} \Phi_{t,1}^\theta(\bfx) &\stackrel{(i)}= \frac{\rmd}{\rmd \eta} \bigg |_{\eta = 0} \Phi_{t,1}^\theta\left(\bfx - \eta \nabla_\bfx \mathcal L\left(\bfx_{1|t}^\theta(\bfx)\right)\right),\nonumber\\
                                               &\stackrel{(ii)}= -\nabla_\bfx\Phi_{t,1}^\theta\left(\bfx - \eta \nabla_\bfx \mathcal L\left(\Phi_{t,1}^\theta(\bfx)\right)\right) \nabla_\bfx \mathcal L\left(\bfx_{1|t}^\theta(\bfx)\right)\bigg |_{\eta = 0}, \nonumber\\
                                               &= -\nabla_\bfx\Phi_{t,1}^\theta(\bfx) \nabla_\bfx \mathcal L\left(\bfx_{1|t}^\theta(\bfx)\right), \nonumber\\
                                               &\stackrel{(iii)}= -\nabla_\bfx\Phi_{t,1}^\theta(\bfx) \nabla_\bfx\bfx_{1|t}^\theta(\bfx)^\top \nabla_{\bfx_1} \mathcal L(\bfx_1),
    \end{align}
    where (i) holds by the definition of the Gateaux differential, (ii) holds by the chain rule, and (iii) holds by the chain rule.
\end{proof}

We note an interesting corollary below.
\begin{corollary}[Dynamics of gradient vs greedy guidance]
    \label{corr:dynamics_greed_vs_gradient}
    The difference between the dynamics of gradient guidance in \cref{prop:dyn_grad_guidance} and greedy gradient guidance in \cref{prop:dyn_greedy_guidance} for a point $\bfx$ at time $t$ with guidance function $\mathcal L \in \C^1(\R^d)$ is
    \begin{equation}
        \left\| \delta_\bfx \Phi_{t,1}^\theta(\bfx) - \delta_\bfx^{\mathcal G} \Phi_{t,1}^\theta(\bfx)\right\| = \left\|\nabla_\bfx \Phi_{t,1}^\theta(\bfx) \left(\nabla_\bfx \Phi_{t,1}^\theta(\bfx) - \nabla_\bfx \bfx_{1|t}^\theta(\bfx) \right)^\top \nabla_{\bfx_1}\mathcal L(\bfx_1)\right\|.
    \end{equation}
\end{corollary}

\subsection{Proof of \texorpdfstring{\cref{thm:grad_vs_greed}}{dynamics of gradient vs greedy guidance}}
\label{proof:grad_vs_greed}

We restate \cref{thm:grad_vs_greed} below.

\begin{theorembox}\gradvsgreed*\end{theorembox}

\begin{proof}
    From \cref{corr:dynamics_greed_vs_gradient} it is clear that the difference between $\delta_\bfx \Phi_{t,1}^\theta(\bfx)$ and $\delta_\bfx^{\mathcal G} \Phi_{t,1}^\theta(\bfx)$ amounts to the difference between the true gradient and gradient of the target prediction model.
    Recall \cref{thm:jacobians_aggp} which enables to write the gradient as the solution to an integral equation:
    \begin{equation}
        \nabla_\bfx \Phi_{t,1}^\theta(\bfx) = \frac{\sigma_1}{\sigma_t}\bm I + \sigma_1\int_t^1 \dot\gamma_u \frac{\gamma_u}{\sigma_u} \var_{1|u}(\Phi_{s,u}^\theta(\bfx))\nabla_{\bfx} \Phi_{t,u}^\theta(\bfx)\;\rmd u.
    \end{equation}
    Now as $\sigma_t \to 0$ as $t \to 1$, we can simplify the integral equation
    \begin{equation}
        \nabla_\bfx \Phi_{t,1}^\theta(\bfx) = \sigma_1\int_t^1 \dot\gamma_u \frac{\gamma_u}{\sigma_u} \var_{1|u}(\Phi_{t,u}^\theta(\bfx))\nabla_{\bfx} \Phi_{t,u}^\theta(\bfx)\;\rmd u,
    \end{equation}
    and then by rewriting the integral in terms of $\rmd \gamma = \dot \gamma_u \rmd u$ we find
    \begin{equation}
        \label{eq:app:grad_flow1}
        \nabla_\bfx \Phi_{t,1}^\theta(\bfx) = \sigma_1\int_{\gamma_t}^{\gamma_1} \frac{\gamma}{\sigma_\gamma} \var_{1|\gamma}(\Phi_{\gamma_t,\gamma}^\theta(\bfx))\nabla_{\bfx} \Phi_{\gamma_t,\gamma}^\theta(\bfx)\;\rmd \gamma.
    \end{equation}
    Next we take a first-order Taylor expansion of $\frac{1}{\sigma_\gamma} \var_{1|\gamma}(\Phi_{\gamma_t,\gamma}^\theta(\bfx))\nabla_{\bfx} \Phi_{\gamma_t,\gamma}^\theta(\bfx)$ centered at $\gamma_t$ which yields:
    \begin{equation}
        \label{eq:app:taylor_grad_flow}
        \frac{\gamma}{\sigma_\gamma} \var_{1|\gamma}(\Phi_{\gamma_t,\gamma}^\theta(\bfx))\nabla_{\bfx} \Phi_{\gamma_t,\gamma}^\theta(\bfx) = \frac{\gamma_t}{\sigma_t} \var_{1|t}(\bfx) + \mathcal O(\gamma - \gamma_t).
    \end{equation}
    For this analysis, it is actually more convenient to include the $\gamma$ term as part of the Taylor expansion rather than computing it in closed form in the integral.
    Now plugging \cref{eq:app:taylor_grad_flow} into \cref{eq:app:grad_flow1} yields
    \begin{align}
        \nabla_\bfx \Phi_{t,1}^\theta(\bfx) &= \sigma_1\int_{\gamma_t}^{\gamma_1} \gamma \frac{1}{\sigma_t} \var_{1|t}(\bfx) + \mathcal O(\gamma - \gamma_t) \;\rmd \gamma,\nonumber\\
                                            &\stackrel{(i)}= \sigma_1 \frac{\gamma_t}{\sigma_t} \var_{1|t}(\bfx) \int_{\gamma_t}^{\gamma_1}\;\rmd \gamma + \mathcal O(h^2),\nonumber\\
                                            &= \sigma_1 \frac{\gamma_t}{\sigma_t} \var_{1|t}(\bfx) \left(\gamma_1 - \gamma_t\right) + \mathcal O(h^2),
    \end{align}
    where (i) holds with $h \coloneq \gamma_1 - \gamma_t$.
    Then, with a little algebra we have
    \begin{align}
        \nabla_\bfx \Phi_{t,1}^\theta(\bfx) &= \sigma_1 \frac{\alpha_t}{\sigma_t^2}  \left(\gamma_1 - \gamma_t\right)  \var_{1|t}(\bfx)+ \mathcal O(h^2),\nonumber\\
                                            &= \sigma_1 \frac{\alpha_t}{\sigma_t^2}  \left(\frac{\alpha_1}{\sigma_1} - \frac{\alpha_t}{\sigma_t}\right)  \var_{1|t}(\bfx)+ \mathcal O(h^2),\nonumber\\
                                            &=\frac{\alpha_t}{\sigma_t^2}  \left(\alpha_1 - \sigma_1\frac{\alpha_t}{\sigma_t}\right)  \var_{1|t}(\bfx)+ \mathcal O(h^2),\nonumber\\
                                            \label{eq:app:expanded_grad_flow}
                                            &\stackrel{(i)}=\frac{\alpha_t}{\sigma_t^2} \var_{1|t}(\bfx)+ \mathcal O(h^2),
    \end{align}
    where (i) holds by the boundary conditions of the schedule (\cf \cref{eq:schedule_boundaries}).
    Now recall \cref{prop:gradient_target_pred} which states:
    \begin{equation}
        \label{eq:app:grad_target}
        \nabla_\bfx \bfx_{1|t}^\theta(\bfx) = \frac{\alpha_t}{\sigma_t^2}\var_{1|t}(\bfx).
    \end{equation}
    Thus from \cref{eq:app:expanded_grad_flow} and \cref{eq:app:grad_target} it is easy to see that
    \begin{equation}
        \left\| \nabla_\bfx \Phi_{t,1}^\theta(\bfx) -  \nabla_\bfx \bfx_{1|t}^\theta(\bfx)\right\| = \mathcal O (h^2),
    \end{equation}
    holds and thus
    \begin{equation}
        \left\|\delta_\bfx \Phi_{t,1}^\theta(\bfx) - \delta_\bfx^{\mathcal G} \Phi_{t,1}^\theta(\bfx)\right\| = \mathcal O (h^2).
    \end{equation}
\end{proof}

\subsection{Proof of \texorpdfstring{ \cref{thm:convergence}}{greedy convergence}}
\label{proof:convergence}

We restate \cref{thm:convergence} below.

\begin{theorembox}\greedyconverge*\end{theorembox}

\begin{proof}
    By \cref{assump:total_derivs}, we can take a $(k-1)$-th order Taylor expansion around $\gamma_t$ of the flow in \cref{eq:exact_sol_flow} to obtain
    \begin{align}
        \Phi_{1|t}^\theta(\bfx_t) &= \frac{\sigma_1}{\sigma_t}\bfx_t + \sigma_1 \int_{\gamma_t}^{\gamma_1}\sum_{n=0}^{k-1} \frac{\rmd^n}{\rmd \gamma^n}\bigg [ \bfx_{1|\gamma}^\theta(\bfx_\gamma)\bigg]_{\gamma = \gamma_t} \frac{(\gamma - \gamma_t)^{n}}{n!}\;\rmd \gamma + \mathcal{O}(h^{k+1}),\nonumber\\
        &= \frac{\sigma_1}{\sigma_t}\bfx_t + \sigma_1 \sum_{n=0}^{k-1} \frac{\rmd^n}{\rmd \gamma^n}\bigg [ \bfx_{1|\gamma}^\theta(\bfx_\gamma)\bigg]_{\gamma = \gamma_t} \int_{\gamma_t}^{\gamma_1}\frac{(\gamma - \gamma_t)^{n}}{n!}\;\rmd \gamma + \mathcal{O}(h^{k+1}),\nonumber\\
        &= \frac{\sigma_1}{\sigma_t}\bfx_t + \sigma_1 \sum_{n=0}^{k-1} \frac{\rmd^n}{\rmd \gamma^n}\bigg [ \bfx_{1|\gamma}^\theta(\bfx_\gamma)\bigg]_{\gamma = \gamma_t} \frac{h^{n+1}}{(n+1)!} + \mathcal{O}(h^{k+1}),
    \end{align}
    where $h \coloneq \gamma_1 - \gamma_t$ is the stepsize.
    Let $k=1$, then we have:
    \begin{align}
        \Phi_{1|t}^\theta(\bfx_t) &= \frac{\sigma_1}{\sigma_t}\bfx_n + \sigma_1\hat\bfx_{1|t}(\bfx_t)h + \mathcal{O}(h^2),\\
                                  &= \frac{\sigma_1}{\sigma_t}\bfx_n + (\alpha_1 - \frac{\sigma_1\alpha_t}{\sigma_t})\hat\bfx_{1|t}(\bfx_t) + \mathcal{O}(h^2).
    \end{align}
    By definition $\sigma_1 = 0$ and $\alpha_1=1$, then
    \begin{equation}
        \Phi_{1|t}^\theta(\bfx_t) = \hat\bfx_{1|t}(\bfx_t) + \mathcal{O}(h^2),
    \end{equation}
    which is equivalent to
    \begin{equation}
        \left\|\Phi_{1|t}^\theta(\bfx_t) - \hat\bfx_{1|t}(\bfx_t)\right\| \leq C_1 h^2,
    \end{equation}
    for some constant $C_1 > 0$.
    Since $\bfx_{1|t}^\theta(\bfx_t^{(n)}) \to \bfx_1^*$ we know that for any $\epsilon > 0$ there exists some $n \geq N$ such that $\|\bfx_1^* - \bfx_{1|t}^\theta(\bfx_t^{(n)})\| < \epsilon$.
    Thus,
    \begin{equation}
        \left\|\Phi_{1|t}^\theta(\bfx_t^{(n)}) - \bfx_1^*\right\| \leq \left\|\Phi_{1|t}^\theta(\bfx_t^{(n)}) - \bfx_{1|t}^\theta(\bfx_t^{(n)})\right\| + \left\|\bfx_1^* - \bfx_{1|t}^\theta(\bfx_t^{(n)})\right\| < \underbrace{\epsilon + C_1h^2}_{\coloneq C_2}.
    \end{equation}
    Therefore, $\Phi_{1|t}(\bfx_t^{(n)})$ converges to a point inside a neighborhood centered at $\bfx_1^*$ with radius $\mathcal O(h^2)$.
\end{proof}

\section{Beyond Euler}
\label{app:beyond_euler}
In this section we provide the full proofs and derivations for \cref{sec:beyond_euler} in the main paper.

\subsection{Proof of \texorpdfstring{ \cref{thm:grad_vs_single_step}}{Local truncation error of discretize-then-optimize gradients}}

Before showing \cref{thm:grad_vs_single_step} we show a more general version below.
\begin{theorembox}
\begin{theorem}[Local truncation error of discretize-then-optimize gradients]
    \label{thm:lte_dto_grads}
    Let $\bm \Phi$ be an explict Runge-Kutta solver of order $\alpha > 0$ to the ODE
    \begin{equation}
        \bfx(0) = \bfx_0, \qquad \frac{\rmd \bfx}{\rmd t}(t) = \bfu_\theta(t, \bfx(t)),
    \end{equation}
    on $[0, T]$ which satisfies the regularity conditions for the Picard-Lindel\"of theorem.
    Let $\Phi_{s,t}^\theta(\bfx)$ denote the flow from $s$ to $t$, for any $s,t \in [0,T]$ admitted by the ODE.
    Then,
    \begin{equation}
        \left\|\nabla_\bfx \Phi_{s,t}^\theta(\bfx) - \nabla_\bfx \bm \Phi_{s,t}(\bfx) \right\| = \mathcal O (h^{\alpha + 1}).
    \end{equation}
\end{theorem}
\end{theorembox}

\begin{proof}
    Consider an explicit $k$-stage Runge-Kutta method given by
    \begin{align}
        \bfu_{n,j} &= \bfu_\theta \left(t_n + c_j h, \bfx_n + h \sum_{i=1}^j a_{j,i}\bfu_{n, i}\right), \qquad j = 1,2,\ldots,k\\
        \bfx_{n+1} &= \bfx_n + h \sum_{j=1}^k b_j \bfu_{n, j},
    \end{align}
    where $a_{j,i}, b_j, c_j$ are all given via the \textit{Butcher Tableau} \parencite[Section 6.1.4]{stewart2022numerical}.
    Now, we consider a single step from time $s$ to time $t$ with initial value $\bfx$ and step size $h \coloneq t - s$.
    Then, the gradient is
    \begin{align}
        \nabla_\bfx \bm \Phi_{s,t}(\bfx) &= \nabla_\bfx \bfx + h \sum_{j=1}^k b_j \nabla_\bfx \bfu_\theta \left(s + c_j h, \bfx + h \sum_{i=1}^j a_{j,i}\bfu_{i}\right),\nonumber\\
                                            \label{eq:app:grad_dto}
                                         &= \bm I + h \sum_{j=1}^k b_j \left[\nabla_{\hat\bfx_j}\bfu_\theta(s + c_j h, \hat\bfx_j) \left(\bm I + h \sum_{i=1}^j a_{j,i}\nabla_\bfx\bfu_{i}\right)\right],
    \end{align}
    where we let
    \begin{equation}
        \hat\bfx_j = \bfx + h \sum_{i=1}^j a_{j,i} \bfu_i.
    \end{equation}

    Next, recall \cref{lemma:forward_sensitivities_ode} which gives the following ODE
    \begin{equation}
        \bm J_s(s) = \bm I, \qquad \frac{\rmd \bm J_s}{\rmd t}(t) = \nabla_\bfx \bfu_\theta(t, \Phi_{s,t}(\bfx)) \bm J_s(t).
    \end{equation}
    Next, we augmented the ODE above with the underyling ODE for the solution state, $\dot \bfx(t) = \bfu_\theta(t, \bfx(t))$.
    We now apply the same Runge-Kutta solver to this augmented ODE for the Jacobian matrices which yields
    \begin{equation}
        \label{eq:app:grad_otd_forward}
        \bm U_j = \bm I + h \sum_{j=1}^k b_j \left[\nabla_{\hat\bfx_j} \bfu_\theta\left(s + c_jh, \hat\bfx_j\right)\left(\bm I + h\sum_{i=1}^j a_{j,i} \nabla_\bfx\bfu_i\right)\right].
    \end{equation}
    Clearly, \cref{eq:app:grad_otd_forward} and \cref{eq:app:grad_dto} are equivalent.
    Now as the underlying numerical solver has local truncation error $\mathcal O(h^{\alpha+1})$ we find that
    \begin{equation}
        \left\|\nabla_\bfx \Phi_{s,t}^\theta(\bfx) - \nabla_\bfx \bm \Phi_{s,t}(\bfx) \right\| = \mathcal O (h^{\alpha + 1}).
    \end{equation}

\end{proof}

\begin{remark}
    This result is intuitive as differentiation is a linear operator.
    However simple, we believe the insight is useful on the discussion of using DTO/OTD/posterior methods for guidance and thus include it here.
\end{remark}

\begin{remark}
    \cref{thm:lte_dto_grads} shows that DTO and OTD are really just two sides of the same coin and that one of the main differences is the choice of end points when discretizing.
\end{remark}

\begin{remark}
    \textcite[Appendix A]{onken2020discretize} made similar observations; however, it is for only of the case of Euler.
\end{remark}

\begin{theorembox}\gradsinglestepdto*\end{theorembox}
\begin{proof}
    This follows as a corollary of \cref{thm:lte_dto_grads}.
\end{proof}

\begin{corollary}[Convergence of a $\alpha$-th order posterior gradient]
    For affine probability paths, if there exists a sequence of states $\bfx_t^{(n)}$ at time $t$ such that it converges to the locally optimal solution $\bm \Phi_{t,1}^\theta(\bfx_t^{(n)}) \to \bfx_{1}^*$.
    Then solution, $\Phi_{1|t}^\theta(\bfx_t^{(n)})$, converges to a neighborhood of size $\mathcal O(h^{\alpha+1})$ centered at $\bfx_1^*$.
\end{corollary}
\begin{proof}
    This follows as a straightforward derivation from \cref{thm:lte_dto_grads}.
\end{proof}

\begin{corollary}[Dynamics of $\alpha$-th order posterior gradient]
    \label{corr:higher_orderdynamics}
    Consider the standard affine Gaussian probability paths model trained to zero loss.
    Let $\bm \Phi$ be an explicit Runga-Kutta solver of order $\alpha > 0$ of a flow model with flow $\Phi_{s,t}^\theta(\bfx)$.
    The Gateaux differential of $\bfx$ at some time $t \in [0,1]$ in the direction of the gradient $\nabla_\bfx \mathcal L\left(\bm\Phi_{t,1}(\bfx)\right)$ is given by
    \begin{equation}
        \delta_\bfx^{\bm\Phi}(\bfx) = -\nabla_\bfx \Phi_{t,1}^\theta(\bfx) \nabla_\bfx \bm \Phi_{t,1}(\bfx)^\top \nabla_{\bfx_1}\mathcal L(\bfx_1).
    \end{equation}
\end{corollary}
\begin{proof}
    This follows straightforwardly from \cref{prop:dyn_greedy_guidance,thm:grad_vs_single_step}.
\end{proof}
    
\subsection{A useful reparameterization of the flow model}
We present a useful reparameterization of the flow model, which is a parallel result to \cref{prop:exact_sol_flow}.
\begin{theorembox}
    \begin{proposition}[Reparameterized for the target prediction model of affine probability paths]
        \label{prop:reparam_ode}
        The ODE governed by the vector field in \cref{eq:marginal_vec} can be reparameterized as
        \begin{equation}
            \frac{\rmd \bfy_\gamma}{\rmd \gamma} = \sigma_0\bfx_{1|\zeta}^\theta\left(\frac{\sigma_\gamma}{\sigma_0}\bfy_\zeta\right),
        \end{equation}
        where $\bfy_t = \frac{\sigma_0}{\sigma_t} \bfx_t$.
    \end{proposition}
\end{theorembox}

\begin{proof}
    The ODE governed by the vector field in \cref{eq:marginal_vec} can be written as
    \begin{equation}
        \frac{\rmd \bfx_t}{\rmd t} = a_t\bfx_t + b_t \bfx_{1|t}^\theta(\bfx_t).
    \end{equation}
    Now we can use the technique of exponential integrators to rewrite the ODE as
    \begin{equation}
        \frac{\rmd}{\rmd t}\left[e^{\int_0^t -a_u \;\rmd u}\bfx_t\right] = e^{\int_0^t -a_u \;\rmd u} b_t \bfx_{1|t}^\theta(\bfx_t).
    \end{equation}
    The exponential term can be simplified to
    \begin{equation}
        e^{\int_0^t -a_u\;\rmd u} = \frac{\sigma_0}{\sigma_t}.
    \end{equation}
    We introduce a \textit{change-of-variables}, $\bfy_t = \frac{\sigma_0}{\sigma_t} \bfx_t$.
    Thus, the ODE becomes
    \begin{equation}
        \frac{\rmd \bfy_t}{\rmd t} = \frac{\sigma_0}{\sigma_t} b_t \bfx_{1|t}^\theta\left(\frac{\sigma_t}{\sigma_0}\bfy_t\right).
    \end{equation}
    Next, recall that $b_t/\sigma_t = \dot\gamma_t$ (\cf \cref{eq:gamma_dot}) which enables a change of integration variable:
    \begin{equation}
        \frac{\rmd \bfy_\gamma}{\rmd \gamma} = \sigma_0\bfx_{1|\gamma}^\theta\left(\frac{\sigma_\gamma}{\sigma_0} \bfy_\gamma\right).
    \end{equation}
\end{proof}

\begin{remark}
    Recall that, often, for affine probability paths we let $\sigma_0 = 1$, further simplifying \cref{prop:reparam_ode} to
    \begin{equation}
        \frac{\rmd \bfy_\gamma}{\rmd \gamma} = \bfx_{1|\gamma}^\theta\left(\sigma_\gamma \bfy_\gamma\right).
    \end{equation}
\end{remark}

\begin{remark}
    \cref{prop:reparam_ode} is a tangential result to the prior result of \textcite[Equation (11)]{pan2024adjointdpm} which was for diffusion models and was developed \wrt the source prediction model rather than the target prediction model and was solved in reverse-time.\footnote{Technically forward-time due to the conventions of diffusion models.}
\end{remark}

This parameterization in \cref{prop:reparam_ode} can be combined with \cref{thm:lte_dto_grads} to construct a DTO approximation of the gradient with truncation error $(\gamma_t - \gamma_s)^{\alpha+1}$.

\section{Notes on using OTD in practice}
\label{app:otd_issues}
While the OTD approach has become quite popular after the work of \textcite{chen2018neural}, several later works have noticed several key issues that we wish to note for ML practitioners.

Recall our prototypical neural ODE (or flow model) of the form
\begin{equation}
    \frac{\rmd \bfx}{\rmd t}(t) = \bfu_\theta(t, \bfx(t)),
\end{equation}
and assume it is defined on the interval $[0, T]$ and the flow model statifies the usual regularity conditions.
Then, the continuous adjoint equations \parencite[Theorem 5.2]{kidger_thesis} are:
\begin{equation}
    \begin{aligned}
    \bfa_\bfx(T) &= \frac{\partial \mathcal{L}}{\partial \bfx_T}, \quad &&\frac{\rmd \bfa_\bfx}{\rmd t}(t) = -\bfa_\bfx(t)^\top \frac{\partial \bfu_\theta}{\partial \bfx}(t, \bfx(t)),\\
    \bfa_\theta(T) &= \bm 0, \quad &&\frac{\rmd \bfa_\theta}{\rmd t}(t) = -\bfa_\bfx(t)^\top \frac{\partial \bfu_\theta}{\partial \theta}(t, \bfx(t)),
    \end{aligned}
\end{equation}
where $\bfa_\bfx(t) \coloneq \partial \mathcal L / \partial \bfx(t)$ and $\bfa_\theta(0) \coloneq \partial \mathcal L / \partial \theta$.

\paragraph{Truncation errors.}
One area of concern is the potential mismatch between the forward trajectory $\{\bfx_{t_i}\}_{i=1}^N$ and the backward trajectory $\{\tilde \bfx_{t_i}\}_{i=1}^N$ when performing the backwards solve.
 \Eg, consider an explicit Euler scheme
 \begin{equation}
     \bfx_{t_{i+1}} = \bfx_{t_i} + (t_{i+1} - t_i)\bfu_\theta(t_i, \bfx_{t_i}).
 \end{equation}
 The same scheme when applied to solving the backward trajectory would yield,
 \begin{equation}
     \tilde\bfx_{t_i} = \tilde\bfx_{t_{i+1}} + (t_i - t_{i+1})\bfu_\theta(t_{i+1}, \tilde\bfx_{t_{i+1}}).
 \end{equation}

Clearly, there is no guarantee that these two trajectories match during the forward and backward solve introducing a source of error.
One potential solution is to use an \textit{algebraically reversible solver} \parencite[see][]{kidger2021efficient,mccallum2024efficient,blasingame2025reversible} which guarantees that the forward and backward trajectory match \textit{perfectly}.
Another option is to store the forward trajectory $\{\bfx_{t_i}\}_{i=1}^N$ in memory and use \textit{interpolated adjoints} if the backward timesteps do not perfectly align with the forward timesteps \parencite[see][]{kim2021stiff}.

\paragraph{Stability concerns.}
Consider the simple ODE, $\dot y(t) = \lambda y(t)$ defined on $t \in [0,T]$ with $y(0) = y_0$ and $\lambda < 0$.
Clearly, most ODE solvers with a non-trivial region of stability \parencite[see][Definition 2.1]{hairer2002stiffodes} will solve this ODE without an issue, as the errors will decrease exponentially with $\lambda < 0$.
However, in the backwards in time solve from $y(T)$ the errors will \textit{grow exponentially}.
It can be shown that the adjoint state suffers from similar stability issues.
The local behavior of a differential equation is described through the eigenvalues of the Jacobian of the vector field \parencite[see][]{butcher2016numerical}.
For $\bfx_t$ this is given by $\frac{\partial \bfu_\theta}{\partial \bfx}$ and for $\bfa_\bfx$ this is given by
\begin{equation}
    \frac{\partial}{\partial \bfa_\bfx}\bigg(-\bfa_\bfx(t)^\top \frac{\partial \bfu_\theta}{\partial \bfx}(t, \bfx(t))\bigg) = - \frac{\partial \bfu_\theta}{\partial \bfx}(t, \bfx(t)).
\end{equation}
Clearly, the Jacobians for $\bfa_\bfx$ and $\bfx_t$ solved in reverse-time are identical, meaning the stability of the backward solve is pushed onto the solve for the adjoint state \parencite[see][Section 5.1.2.4]{kidger_thesis} for more details.
Reversible solvers eliminate truncation errors, but tend to suffer from poor stability, \eg, the region of stability for reversible Heun applied to neural ODEs is the complex interval $[-i, i]$ \parencite{kidger2021efficient}.
Recent work by \textcite{mccallum2024efficient}, however, has shown a strategy for constructing reversible solvers with a non-trivial region of stability.

\begin{table}[t!]
    \centering
    \caption{Comparison of different strategies for performing backpropagation through flow models. For the complexity analysis $n$ denotes the number of discretization steps and $d$ the dimensionality of the state. Note, for accuracy we mean there are no truncation errors. Note that whilst in general the stability of reversible solvers is quite poor, there are \textit{some} solvers which have a non-trival region of stability.}
    \label{tab:dto_vs_otd}
    \begin{tabular}{l ll cc}
        \toprule
        \textbf{Method} & \textbf{Time} & \textbf{Memory} & \textbf{Accurate gradients} & \textbf{Stability}\\
        \midrule
        DTO & $\mathcal O(n)$ & $\mathcal O(nd^2)$ & \cmark & -\\
        DTO + recursive checkpointing & $\mathcal O(n\log n)$ & $\mathcal O(d^2\log n)$ & \cmark & -\\
        OTD + stored trajectory & $\mathcal O(n)$ & $\mathcal O(nd + d^2)$ & \cmark & -\\
        OTD + reversible solver & $\mathcal O(n)$ & $\mathcal O(d^2)$ & \cmark & ?\\
        OTD & $\mathcal O(n)$ & $\mathcal O(d^2)$ & \xmark & \xmark\\
        \bottomrule
    \end{tabular}
\end{table}

\paragraph{Recommendations.}
In light of these concerns we propose we consider to be best practices for deciding what scheme to use.

Generally, the best choice is DTO when memory allows as it is the most \textit{accurate} in terms of the forward discretization.
If memory is an issue then using a clever checkpointing scheme \parencite{griewank2000revolve,griewank1992achieving,stumm2010new} can help alleviate such issues in exchange for additional compute time.
The recursive checkpointing strategy in combination with DTO is actually the default (and recommended) implementation in the \href{https://docs.kidger.site/diffrax/api/adjoints/}{\texttt{Diffrax}} library.
Alternatively, one could store the forward trajectory in memory and then apply the OTD scheme on these stored states (not activations).
This strategy of caching the forward trajectory is quite popular and was used by \textcite{blasingame2024adjointdeis,domingo-enrich2025adjoint} in practice when solving the continuous adjoint equations.
Another option is to use an algebraically reversible solver in conjunction with OTD.
Lastly, one could use vanilla OTD, which we should mention can actually work reasonably well depending on the application despite the concerns listed above.

In \cref{tab:dto_vs_otd} we summarize the discussion of this section and hope it is helpful to the reader.%

\section{On control signal optimization}
\label{app:control_signal_opt}
Rather than optimizing the trajectory of the solution or the initial condition, several works \parencite{liu2023flowgrad,wang2024training} have explored the guidance from the perspective of optimal control \parencite{kirk2004optimal}.
In essence this technique first injects an additional control signal,
$\bfz \in \C^{1}(\R;\R^d)$, to the vector field, $\bfu_t^\theta$, such that
\begin{equation}
    \label{eq:flow_plus_control}
    \frac{\rmd \bfx_t}{\rmd t} = \bfu_t^\theta(\bfx_t) + \bfz(t).
\end{equation}
Thus, instead of optimizing $\{\bfx_t\}_{t \in [0,t]}$ directly, this control signal can instead be optimized, serving as one of the key insights in \parencite{liu2023flowgrad,wang2024training}.
\Ie, suppose we have a neural ODE with vector field $\bfu_t^\theta(\bfx)$, then we can write the optimization problem as
\begin{equation}
    \begin{aligned}
        \min_{\bm z} \quad & \mathcal L(\bfx_T) + \lambda \int_0^T \|\bm z(t)\|\; \rmd t,\\
        \textrm{s.t.} \quad & \bfx_T = \bfx_0 + \int_0^T \bfu_t^\theta(\bfx_t) + \bm z(t)\;\rmd t.
    \end{aligned}
\end{equation}

The next natural question then is to ask about the behavior of a greedy strategy applied to $\bfz(t)$.
To simplify the analysis, we now consider a control signal applied to the posterior model $\bfx^\theta_{1|t}$ such that it is replaced by $\bfx^\theta_{1|t}(\bfx_t) + \bfz(t)$ which amounts to simply rescaling $\bfz(t)$ from \cref{eq:flow_plus_control} with $b_t$.
From this construction, it should be clear that the greedy gradient for the control signal is merely $\nabla_{\tilde\bfx_1} \mathcal{L}(\tilde\bfx_1)$.
If using the original formulation where the control signal is applied to the vector field, rather than the denoiser, the gradient is simply scaled by a weighting function dependent on time.
Note that this approach is similar to the greedy approach taken by \textcite{blasingame2024greedydim}; however, they inject the control signal to the source prediction model rather than the target prediction model.

\subsection{Continuous adjoint equations for control signals}
We can model the gradient for this signal by augmenting the continuous adjoint equations with the adjoint state $\bfa_\bfz(t) \coloneq \partial \mathcal{L} / \partial \bfz(t)$.
In \cref{thm:cae_for_control} we show that this gradient is simply an integral of the adjoint state $\bfa_\bfx(t)$.

\begin{theorembox}
\begin{restatable}[Continuous adjoint equations for the control term]{theorem}{caeforcontrol}
    \label{thm:cae_for_control}
    Let $\bfu_t^\theta \in \C^{1,1}([0,T]\times \R^{d_x}; \R^{d_x})$ be a parameterization of some time-dependent vector field of a neural ODE that is Lipschitz continuous in its second argument, and let $\bfz \in \C^{1}([0,1];\R^d)$ be an additional control signal such that the new dynamics are given by \cref{eq:flow_plus_control}.
    Let $\bfa_\bfz(t) \coloneq \partial\mathcal{L}/\partial \bfz(t)$, then
    \begin{equation}
        \bfa_\bfz(t) = - \int_T^t \bfa_\bfx(s) \; \rmd s.
    \end{equation}
\end{restatable}
\end{theorembox}

Our proof follows the structure of the modern proof of Pontryagin's original result \parencite{pontryagin1963} presented by \parencite{chen2018neural}; and is similar to the form used by \textcite[Theorem 2.2]{blasingame2024adjointdeis}.

\begin{proof}
    For notational clarity, we use the notation $\bfx(t) = \bfx_t$.
    We define the augmented state on $[0, T]$ as
    \begin{equation}
        \frac{\rmd}{\rmd t} \begin{bmatrix}
            \bfx\\
            \bfz
        \end{bmatrix}(t) = \bsf_{\text{aug}} = \begin{bmatrix}
            \bfu_\theta(t, \bfx(t)) + \bfz(t)\\
            \frac{\rmd \bfz}{\rmd t}(t)
        \end{bmatrix},
    \end{equation}
    and the augmented adjoint state as
    \begin{equation}
        \bfa_{\text{aug}}(t) \coloneq \begin{bmatrix}
            \bfa_\bfx\\
            \bfa_\bfz
        \end{bmatrix}(t).
    \end{equation}
    The Jacobian of $\bsf_{\text{aug}}$ has form
    \begin{equation}
        \frac{\partial \bsf_{\text{aug}}}{\partial [\bfx, \bfz]} = \begin{bmatrix}
            \frac{\partial \bfu_\theta(t, \bfx(t))}{\partial \bfx} & \mathbf 1\\
            \mathbf 0 & \mathbf 0
        \end{bmatrix}.
    \end{equation}
    The evolution of the adjoint state is given by
    \begin{equation}
        \frac{\rmd \bfa_{\text{aug}}}{\rmd t}(t) = -\begin{bmatrix}
            \bfa_\bfx & \bfa_\bfz
        \end{bmatrix}(t) \frac{\partial \bsf_{\text{aug}}}{\partial [\bfx, \bfz]}(t).
    \end{equation}
    Therefore, $\bfa_\bfu(t)$ evolves with
    \begin{equation}
        \bfa_\bfu(T) = \mathbf 0, \qquad \frac{\rmd \bfa_\bfu}{\rmd t}(t) = -\bfa_\bfx(t),
    \end{equation}
    thereby finishing the proof.
\end{proof}

\section{Implementation details}
\label{app:impl_details}
We discuss how to implement the greedy strategy.

\subsection{The construction of the greedy guidance schemes}

Recall that the general Butcher tableau for a $k$-stage explicit RK scheme \parencite[Section 6.1.4]{stewart2022numerical} is written as
\begin{equation}
    \renewcommand\arraystretch{1.2}
    \begin{array}{c|ccccc}
        c_1 &\\
        c_2 & a_{2,1}\\
        c_3 & a_{3,1} & a_{3,2}\\
        \vdots & \vdots & \vdots & \ddots\\
        c_k & a_{k,1} & a_{k,2} & \cdots & a_{(k-1),k}\\
        \hline
        & b_1 & b_2 & \cdots & b_{k-1} & b_k
    \end{array} = 
    \begin{array}{c|c}
        c & a\\
        \hline
        & b
    \end{array}.
\end{equation}

Thus a single-step is given by
\begin{align}
    \bfu_{n,j} &= \bfu_\theta \left(t_n + c_j h, \bfx_n + h \sum_{i=1}^j a_{j,i}\bfu_{n, i}\right), \qquad j = 1,2,\ldots,k\\
    \bfx_{n+1} &= \bfx_n + h \sum_{j=1}^k b_j \bfu_{n, j},
\end{align}
where $a_{j,i}, b_j, c_j$ are all given via the \textit{Butcher Tableau} \parencite[Section 6.1.4]{stewart2022numerical}.
Now, we consider a single step from time $s$ to time $t$ with initial value $\bfx$ and step size $h \coloneq t - s$.
Then, the gradient is
\begin{align}
    \nabla_\bfx \bm \Phi_{s,t}(\bfx) &= \nabla_\bfx \bfx + h \sum_{j=1}^k b_j \nabla_\bfx \bfu_\theta \left(s + c_j h, \bfx + h \sum_{i=1}^j a_{j,i}\bfu_{i}\right),\nonumber\\
                                     &= \bm I + h \sum_{j=1}^k b_j \left[\nabla_{\hat\bfx_j}\bfu_\theta(s + c_j h, \hat\bfx_j) \left(\bm I + h \sum_{i=1}^j a_{j,i}\nabla_\bfx\bfu_{i}\right)\right],
\end{align}
where we let
\begin{equation}
    \hat\bfx_j = \bfx + h \sum_{i=1}^j a_{j,i} \bfu_i.
\end{equation}
Which can easily be found through standard reverse-mode autodifferentiation frameworks; likewise, the gradients for multiple Euler steps can be found.

\section{A brief introduction to inverse problems}
\label{app:inverse_problems}
Inverse problems cover a large class of scientific problems \parencite{chung2023diffusion} that encompass scenarios where a partial measurement $\bfy$ is made of $\bfx$.
When the mapping $\bfx \mapsto \bfy$ is not an injection, recovering $\bfx$ from $\bfy$ becomes an ill-posed inverse problem.
Generally, the relationship between the underlying sample $\bfx$ and the measurement $\bfy$ is given by
\begin{equation}
    \bfy = \mathcal A(\bfx) + \bm \eta, \qquad \bfy, \bm \eta \in \R^{d_y}, \bfx \in \R^{d_x},
\end{equation}
where $\mathcal A: \R^{d_x} \to \R^{d_y}$ is the forward measurement operator and $\bm \eta \sim \mathcal(0, \beta_\bfy^2 \bm I)$ is the measurement noise.
\begin{standoutbox}
The inverse problem then is to find $p(\bfx|\bfy)$.
\end{standoutbox}
More details on these types of problems can be found in \textcite{chung2023diffusion,moufad2025variational,zhang2024improvingdiffusioninverseproblem}.

\subsection{Inverse problems and diffusion models}
Recall that the ODE formulation of diffusion models is just a particular type of affine Gaussian probability path \parencite{lipman2024flow-guide}.
Following the conventions of the EDM model \parencite{karras2022elucidating} we write this ODE formulation, known in the literature as the \textit{probability flow ODE}, below in
\begin{equation}
    \rmd \bfx_t = -\dot\sigma_t\sigma_t \nabla_{\bfx_t} \log p(\sigma_t, \bfx_t)\; \rmd t,
\end{equation}
where $p(\sigma_t, \bfx_t)$ is the joint distribution of $\bfx_t$ at noise level $\sigma_t$.\footnote{This $\sigma_t$ is not the same as the $\sigma_t$ from the scheduler $(\alpha_t, \sigma_t)$ used in the main paper.}
\NB, for diffusion models $\rmd t$ is a \textit{negative} timestep and we integrate in reverse-time from $T$ to $0$.
These models are also called \textit{score-based generative models} due to learning the score function $\nabla_{\bfx_t} \log p(\sigma_t, \bfx_t)$.

One of the insights of \textcite{song2021scorebased,chung2023diffusion} is to apply Bayes' theorem for inverse problems to score-based generative models, \ie,
\begin{align}
    p(\bfx|\bfy) &= \frac{p(\bfy|\bfx)p(\bfx)}{p(\bfy)},\\
    \nabla_\bfx \log p(\bfx|\bfy) &= \nabla_\bfx \log p(\bfx) + \nabla_\bfx \log p(\bfy|\bfx).
\end{align}
Adapting this for diffusion models, assuming $\mathcal A$ is defined on $\bfx_0$ (the output), we have
\begin{equation}
    \nabla_{\bfx_t} \log p(\sigma_t, \bfx_t | \bfy) = \nabla_{\bfx_t} \log p(\sigma_t, \bfx_t) + \nabla_{\bfx_t} \log p(\bfy | \bfx_t, \sigma_t).
\end{equation}
The unconditional score term is the regular score function learned by diffusion models and thus is \textit{appropriately} learned; however, the other term is much more difficult to work with.
The approach of \textcite{chung2023diffusion} is to use an approximation of
\begin{equation}
    p(\bfy|\bfx_t, \sigma_t) = \ex_{\bfx_0 \sim p(\bfx_0|\bfx_t)}[p(\bfy|\bfx_0,\sigma_0)],
\end{equation}
via Tweedie's formula \parencite{stein1981estimation} to write
\begin{equation}
    p(\bfy|\bfx_t, \sigma_t) \approx p(\bfy | \ex[\bfx_0|\bfx_t], \sigma_0).
\end{equation}
The approximation error can be quantified by the Jensen gap \parencite[Theorem 1]{chung2023diffusion}.

\section{Experimental details}
\label{app:impl}
We provide additional details of the experiments performed in \cref{sec:experiments}.
\NB, for all experiments we used fixed random seeds between the different software components to ensure a fair comparison.

\subsection{Inverse image problems}
\label{app:inv_images}

\paragraph{Inverse problems.}
The inverse problems are implemented in the same way as in \textcite{zhang2024improvingdiffusioninverseproblem}.
We reiterate some of the important settings below.
For Gaussian and motion deblurring we made use of kernels of size $61 \times 61$ with standard deviations of 3.0 and 0.5 respectively.
The box inpainting task makes use of a random box of size $128 \times 128$ to mask the original images, while the random inpainting task randomly masks each pixel with a probability of 70\% following \parencite{song2024solving}.
The measure for the high dynamic range reconstruction problem is defined as
\begin{equation}
    \bfy \sim \mathcal{N}(\mathrm{clip}(\alpha\bfx_0, -1, 1), \beta_\bfy^2\bm I),
\end{equation}
with $\alpha = 2$.

\paragraph{Diffusion model.}
We make use of the pre-trained diffusion model from \textcite{chung2023diffusion}, trained on the FFHQ $256\times 256$ dataset.
We focus on the probability flow ODE formulation popularized by \textcite{karras2022elucidating} known as EDM described as
\begin{equation}
    \rmd \bfx_t = -\dot\sigma_t\sigma_t \nabla_{\bfx_t} \log p(\sigma_t, \bfx_t)\;\rmd t.
\end{equation}
Following \textcite{ben-hamu2024dflow}, we employ a midpoint scheme to solve this ODE in \textit{reverse-time} with $N = 20$ steps.
We use the noise schedule $\sigma_t = t$ which means $\dot\sigma_t = 1$.
The discretized noise schedule $\{\sigma_n\}_{n=1}^N$ is given by the following polynomial interpolation
\begin{equation}
    \sigma_n = \left(\sigma_{\max}^{\frac 1\rho} + \frac{n}{N-1}\left(\sigma_{\min}^{\frac 1\rho} - \sigma_{\max}^{\frac 1\rho}\right)\right)^\rho.
\end{equation}
We use $\rho = 7$, $T = \sigma_{\max} = 100$, and $\epsilon = \sigma_{\min} = 0.01$ for all experiments and integrate over $[\epsilon, T]$.
\NB, truncating the integration domain at $\epsilon$ rather than $0$ is quite common in diffusion models \parencite{song2023consistency}.

\paragraph{Hyperparameters.}
Unlike previous works \parencite{zhang2024improvingdiffusioninverseproblem} we did not adjust the hyperparameters per task and left them the same throughout.
The learning rate was set at $\eta = 1$ for all experiments, and we performed $n_{\textrm{opt}} = 50$ optimization steps with the stock implementation of the \href{https://docs.pytorch.org/docs/stable/generated/torch.optim.SGD.html}{\texttt{torch.optim.SGD}} method for each step of the ODE solve.
We set $\beta_\bfy = 0.05$ for all tasks.

\paragraph{Ablation study.}
For the ablation study in \cref{tab:ablation_hdr} we used the L-BGFS optimizer over the standard SGD optimizer used in the main experiments (for the greedy guidance runs). For full DTO we used SGD as it provided better performance over L-BGFS in that scenario.
\NB, due to compute limitations we couldn't run DTO for step sizes larger than 8.
Importantly, we fix the maximum number of optimization steps between the greedy and DTO strategies; for greedy we take 5 optimization step per step in the ODE solver, so a 100 in total.
Likewise, for DTO we take a 100 optimization steps in total.

\subsection{Molecule generation for QM9}
\label{app:qm9}
We follow the experimental methodology taken in previous work \parencite{ben-hamu2024dflow,wang2024training} and follow the conditional generation pipeline used by \textcite{hoogeboom2022equivariant}.
An equivariant \textit{graph neural network} (GNN) was trained for each property on half of the QM9 dataset, serving as a classifier---this model was then used as a guidance function during the experiments.
The EquiFM \parencite{song2023equivariant} model was trained on the whole QM9 training set and was used as the underlying flow model for the experiments.
Following \textcite{wang2024training}, the test time properties were sampled from the whole training set; in contrast to \textcite{ben-hamu2024dflow}.

Following \textcite{ben-hamu2024dflow} we used the L-BFGS algorithm \parencite{liu1989limited} with 5 optimizer steps and 5 inner steps with a linear search, in particular we used the stock PyTorch implementation \href{https://docs.pytorch.org/docs/stable/generated/torch.optim.LBFGS.html}{\texttt{torch.opt.LBFGS}}.
For the DTO experiment we used a learning rate of $\eta = 1$.
We tried this for the posterior guidance experiments but encountered severe instability.
We found that a learning rate of $\eta = 0.001$ seemed to work better.

Recall that \cref{prop:dyn_greedy_guidance} states that the greedy gradient is scaled by the covariance projection.
This effect is lessened as $t \to 1$, thus in later timesteps the greedy gradient is more likely to push samples off the data manifold.
We observed this, with exploding losses even at small learning rates.
To remedy this, we took inspiration from other works \parencite{chung2023diffusion,yu2023freedom,moufad2025variational} and annealed the learning rate.
We chose the following simple scheduler:
\begin{equation}
    \eta_t = \begin{cases}
        \eta (1 - t) & t > 0.5\\
        0 & t \leq 0.5
    \end{cases},
\end{equation}
where $\eta = 0.001$ is the base learning rate.

\paragraph{Runge-Kutta 4.}
Additionally, we ran some experiments using RK4 but ran into insurmountable stability issues.
Recall that RK4 is given by 
\begin{align}
    \bm k_1 &= \bfu_\theta\left(t_n, \bfx_n\right),\\
    \bm k_2 &= \bfu_\theta\left(t_n + \frac h2, \bfx_n + \frac h2 \bm k_1\right),\\
    \bm k_3 &= \bfu_\theta\left(t_n + \frac h2, \bfx_n + \frac h2 \bm k_2\right),\\
    \bm k_4 &= \bfu_\theta\left(t_n + h, \bfx_n + h \bm k_3\right),\\
    \bfx_{n+1} &= \bfx_n + \frac h6 (\bm k_1 + 2\bm k_2 + 2\bm k_3 + \bm k_4).
\end{align}
Using the step size $h = 1 - t$ we encountered large stability issues with the $\bm k_4$ term due to being evaluated at the endpoint of the flow model trajectory.
We tried a mixed-solver scheme were we would start with Euler and then switch to RK4, but that did not help.
We also tried the common diffusion trick of truncated the time interval to $[0, 1 - \epsilon]$ for some small $\epsilon > 0$, but this did not solve the stability issues either.
Ultimately, we abandoned it for this work and left such explorations for future work.
It seems reasonable to suppose that schemes which don't evaluate on the endpoint, \eg, Ralston's method, Heun's third-order method, or Ralton's third-order method may fair better.

\subsection{Numerical schemes}
We detail the numerical schemes used for posterior guidance beyond Euler.

\paragraph{Midpoint.}
The midpoint scheme used in both experiments is implemented as
\begin{equation}
    \bfx_{1} = \bfx_t + h \bfu_\theta\left(t + \frac h2, \bfx_t + \frac h2 \bfu_\theta(t, \bfx_t)\right)
\end{equation}
with step size $h = 1 - t$.\footnote{This is appropriately adjusted for diffusion models with a terminal time of $0$.}

\paragraph{2-step Euler.}
This scheme used in both experiments is implemented as
\begin{align}
    \bfx_{t+\frac h2} &= \bfx_{t} + \frac h2 \bfu_\theta(t, \bfx_t),\\
    \bfx_{1} &= \bfx_{t + \frac h2} + \frac h2 \bfu_\theta\left(t + \frac h2, \bfx_{t + \frac h2}\right),
\end{align}
with step sizes $h = 1-t$.

\subsection{Hardware and compute cost}
\label{app:hardware}

\paragraph{Inverse image problems.}
The inverse image problem experiments were run on a single NVIDIA H100 80GB GPU.
It took roughly 4 minutes and 78 GB of VRAM to generate 10 images for each inverse problem.
As such each experiment took about an 40--50 minutes.
Experiments which used the midpoint method, unsurprisingly ran about 90\% slower.

\paragraph{Molecule generation.}
The molecule generation experiments were run on a single NVIDIA V100 16GB GPU.
It took about 3 minutes and 1.5 GB of VRAM to generate 1 molecule leading to the experiments taking on the order of 300 minutes to complete.
Experiments which used the midpoint method, unsurprisingly ran about 90\% slower.

\section{Further experimental results}
\label{app:add_results}

We present additional experimental results that we could not include in the main paper for the sake of space.

\subsection{Molecule generation for QM9}
\label{app:add_mol}

In \cref{tab:stability} we present the \textit{atom stability percentage} (ASP) and \textit{molecule stability percentage} (MSP) per property for each guided generation model.
Interestingly, despite their poor quantitative performance in MAE (\cf \cref{tab:mol_results}) the greedy (midpoint) and (2-step Euler) strategies have slightly better stability than DTO.

\begin{table}[htpb]
    \centering
    \caption{Stability reported in ASP/MSP per property.}
    \label{tab:stability}
    \scriptsize

    \begin{tabular}{l ccc ccc}
        \toprule
        Property & $\alpha$ & $\Delta \varepsilon$ & $\varepsilon_{\mathrm{HOMO}}$ & $\varepsilon_{\mathrm{LUMO}}$ & $\mu$ & $C_v$\\
        \midrule
        DTO & 94.90/65.00 & 96.20/74.00 & 95.90/67.00 & 96.00/65.00 & 94.60/61.00 & 95.00/67.00\\
        Greedy (Euler) & 94.70/68.80 & 96.40/76.00 & 97.40/79.00 & 98.40/84.80 & 97.60/84.00 & 85.55/21.20\\
        Greedy (midpoint) & 97.46/80.00 & 97.51/83.00 & 97.91/81.00 & 97.77/83.00 & 97.70/81.00 & 97.09/80.00\\
        Greedy (2-step Euler) & 97.67/82.00 & 96.95/74.00 & 98.18/84.00 & 96.29/72.00 & 97.40/93.00 & 97.75/84.00 \\
        \midrule
        EquiFM & \multicolumn{6}{c}{98.88/89.00} \\
        \bottomrule
    \end{tabular}
\end{table}

\begin{table}[t]
    \centering
    \caption{Additional results for inverse image problems on FFHQ $256 \times 256$.}
    \label{tab:add_inv_images}
    \scriptsize
    \begin{tabular}{ll cccc}
        \toprule
        \textbf{Task} & \textbf{Method} & \textbf{PSNR} ($\uparrow$) & \textbf{SSIM} ($\uparrow$) & \textbf{LPIPS} ($\downarrow$) & \textbf{FID} ($\downarrow$)\\
        \midrule
        \multirow{10}{*}{Super resolution 4$\times$} & Greedy (Euler)  & 27.94 & 0.728 & 0.217 & 66.64\\
        & Greedy (midpoint)  & 27.98 & 0.727 & 0.224 & 70.96\\
        & Greedy (2-step Euler)  & 27.95 & 0.728 & 0.220 & 68.93\\

        & DAPS & 29.07 & 0.818 & 0.177 & 51.44\\
        & DPS & 25.86 & 0.753 & 0.269 & 81.07 \\
        & DDRM & 26.58 & 0.782 & 0.282 & 79.25 \\
        & DDNM & 28.03 & 0.795 & 0.197 & 64.62 \\
        & DCDP & ${28.66}$ & 0.807 & ${0.178}$ & 53.81\\
        & FPS-SMC & 28.42 & ${0.813}$ & 0.204 & 49.25\\ 
        & DiffPIR & 26.64 & - & 0.260 & 65.77\\

        \midrule
        \multirow{9}{*}{Inpaint (box)} & Greedy (Euler)  & 23.74 & 0.732 & 0.187 & 46.87\\
        & Greedy (midpoint)  & 24.08 & 0.724 & 0.186 & 44.55\\
        & Greedy (2-step Euler)  & 23.88 & 0.720 & 0.188 & 44.09\\

        & DAPS & 24.07 & 0.814 & 0.133 & 43.10\\
        & DPS & 22.51 & 0.792 & 0.209 & 61.27\\
        & DDRM & 22.26 & 0.801 & 0.207 & 78.62\\
        & DDNM & ${24.47}$ & 0.837 & 0.235 & 46.59\\
        & DCDP & 23.89 & 0.760 & 0.163 & ${45.23}$\\
        & FPS-SMC & 24.86 & ${0.823}$ & ${0.146}$ & 48.34\\

        \midrule
        \multirow{8}{*}{Inpaint (random)} & Greedy (Euler)  & 30.87 & 0.823 & 0.141 & 40.73\\
        & Greedy (midpoint)  & 31.03 & 0.816 & 0.139 & 38.80\\
        & Greedy (2-step Euler)  & 30.80 & 0.811 & 0.144 & 39.23\\

        & DAPS & 31.12 & 0.844 & 0.098 & 32.17\\
        & DPS & 25.46 & 0.823 & 0.203 & 69.20\\
        & DDNM & 29.91 & 0.817 & 0.121 & 44.37\\
        & DCDP & ${30.69}$ & ${0.842}$ & 0.142 & 52.51\\
        & FPS-SMC & 28.21 & 0.823 & 0.261 & 61.23\\

        \midrule
        \multirow{10}{*}{Gaussian deblurring} & Greedy (Euler)  & 28.01 & 0.766 & 0.182 & 57.04\\
        & Greedy (midpoint)  & 28.36 & 0.776 & 0.185 & 58.55\\
        & Greedy (2-step Euler)  & 28.18 & 0.774 & 0.181 & 57.18\\

        & DAPS & 29.19 & 0.817 & 0.165 & 53.33\\
        & DPS & 25.87 & 0.764 & 0.219 & 79.75\\
        & DDRM & 24.93 & 0.732 & 0.239 & 92.43\\
        & DDNM & ${28.20}$ & ${0.804}$ & ${0.216}$ & ${57.83}$\\
        & DCDP & 27.50 & 0.699 & 0.304 & 86.43\\
        & FPS-SMC & 26.54 & 0.773 & 0.253 & 67.45 \\
        & DiffPIR & 27.36 & - & 0.236 & 59.65\\

        \midrule
        \multirow{8}{*}{Motion deblurring} & Greedy (Euler)  & 29.35 & 0.748 & 0.207 & 63.05\\
        & Greedy (midpoint)  & 29.73 & 0.762 & 0.207 & 66.21\\
        & Greedy (2-step Euler)  & 29.64 & 0.764 & 0.203 & 63.99\\

        & DAPS & 29.66 & 0.847 & 0.157 & 39.49\\
        & DPS & 24.52 & 0.801 & 0.246 & 65.23\\
        & DCDP & 25.08 & 0.512 & 0.364 & 125.13\\
        & FPS-SMC & ${27.39}$ & ${0.826}$ & ${0.227}$ & ${48.32}$\\
        & DiffPIR & 26.57 & - & 0.255 & 65.78\\

        \midrule
        \multirow{7}{*}{Phase retrieval} & Greedy (Euler)  & 15.10 & 0.282 & 0.598 & 298.06\\
        & Greedy (midpoint)  & 15.10 & 0.286 & 0.595 & 299.45\\
        & Greedy (2-step Euler)  & 15.07 & 0.284 & 0.598 & 304.60\\

        & DAPS & ${3 0 . 6 3}_{ \pm 3.13}$ & ${0 . 8 5 1}_{ \pm 0.072}$ & ${6 . 1 3 9}_{ \pm 0.060}$ & 42.71\\
        & DPS & $17.64_{ \pm 2.97}$ & $0.441_{\pm 0.129}$ & $0.410_{ \pm 0.090}$ & 104.52\\
        & RED-diff & $15.60_{ \pm 4.48}$ & $0.398_{ \pm 0.195}$ & $0.596_{ \pm 0.092}$ & 167.43\\
        & DCDP & ${{28.65}} \pm 8.09$ & ${{0.781}}_{ \pm 0.217}$ & ${{0.203}}_{ \pm 0.196}$ & ${68.13}$\\

        \midrule
        \multirow{7}{*}{Nonlinear deblur} & Greedy (Euler)  & 24.767 & 0.551 & 0.327 & 79.06\\
        & Greedy (midpoint)  & 25.09 & 0.558 & 0.332 & 76.73\\
        & Greedy (2-step Euler)  & 24.81 & 0.547 & 0.330 & 76.26\\
        
        & DAPS & ${28.29}_{ \pm 1.77}$ & ${0.783}_{ \pm 0.036}$ & ${0 . 1 5 5}_{ \pm 0.032}$ & ${49.38}$\\
        & DPS & $23.39_{ \pm 2.01}$ & $0.623_{ \pm 0.082}$ & $0.278_{ \pm 0.060}$ & 91.31\\
        & RED-diff & ${3 0 . 8 6}_{ \pm 0.51}$ & ${0. 7 9 5}_{ \pm 0.028}$ & ${0.160}_{\pm 0.034}$ & 43.84\\
        & DCDP & $27.92_{ \pm 2.64}$ & $0.779_{ \pm 0.067}$ & $0.183_{ \pm 0.051}$ & 51.96\\

        \midrule
        \multirow{6}{*}{High dynamic range} & Greedy (Euler)  & 24.16 & 0.767 & 0.181 & 43.59\\
        & Greedy (midpoint)  & 26.62 & 0.809 & 0.160 & 37.86\\
        & Greedy (2-step Euler)  & 25.70 & 0.797 & 0.165 & 37.97\\

        & DAPS & ${2 7 . 1 2}_{ \pm 3.53}$ & ${0 . 7 5 2}_{ \pm 0.041}$ & ${0 . 1 6 2}_{ \pm 0.072}$ & 42.97\\
        & DPS & ${22.73}_{ \pm 6.07}$ & ${0.591}_{ \pm 0.141}$ & $0.264_{ \pm 0.156}$ & 112.82\\
        & RED-diff & $22.16_{ \pm 3.41}$ & $0.512_{ \pm 0.083}$ & ${0.258}_{ \pm 0.089}$ & ${108.32}$\\

        \bottomrule
    \end{tabular}
\end{table}

\subsection{Further results on inverse image problems}
\label{app:add_inv}
To put the results from \cref{sec:inv_problems} into context we present some detailed comparisons to other works from the domain of inverse problems with diffussion models, namely:
\begin{enumerate}
    \item DAPS \parencite{zhang2024improvingdiffusioninverseproblem},
    \item DPS \parencite{chung2023diffusion},
    \item DDRM \parencite{kawar2022denoising},
    \item DDNM \parencite{wang2023zeroshot},
    \item DCDP \parencite{li2024decoupled},
    \item FPS-SMC \parencite{dou2024diffusion},
    \item DiffPIR \parencite{zhu2023diffpir}, and
    \item RED-diff \parencite{mardani2024a}.
\end{enumerate}

We present the full comparison in \cref{tab:add_inv_images}.

\clearpage

\subsection{Sampling trajectories for inverse problems}
We present the solution trajectories the different guidance algorithms below for solving the HDR inverse problem.
Note that the midpoint and 2-step Euler, unsurprisingly, have better approximations of $\bfx_1$.

\begin{figure}[h]
    \includegraphics[width=\textwidth]{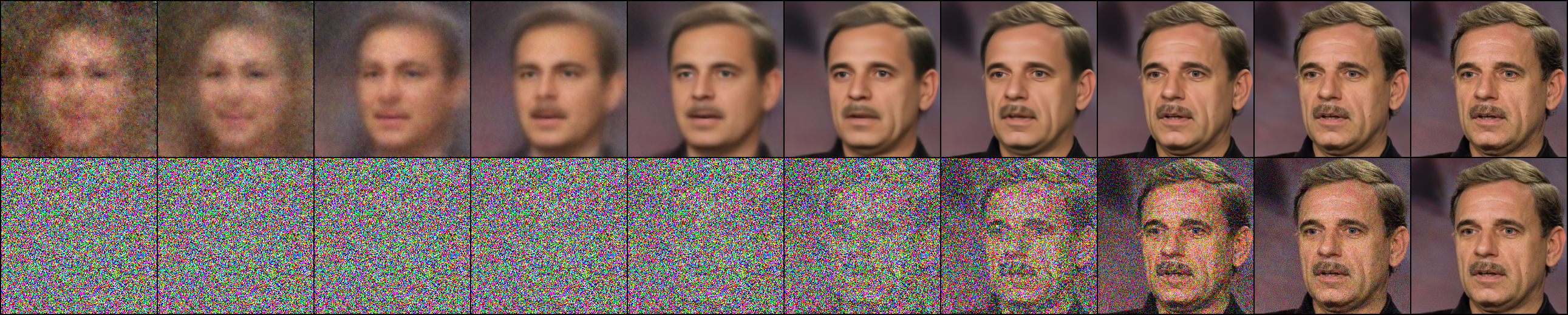}
    \caption{Sampling trajectory for greedy (Euler) solving the HDR inverse problem. Top row is $\bfx_{1|t}^\theta(\bfx_t)$ and the bottom row is $\bfx_t$.}
\end{figure}

\begin{figure}[h]
    \includegraphics[width=\textwidth]{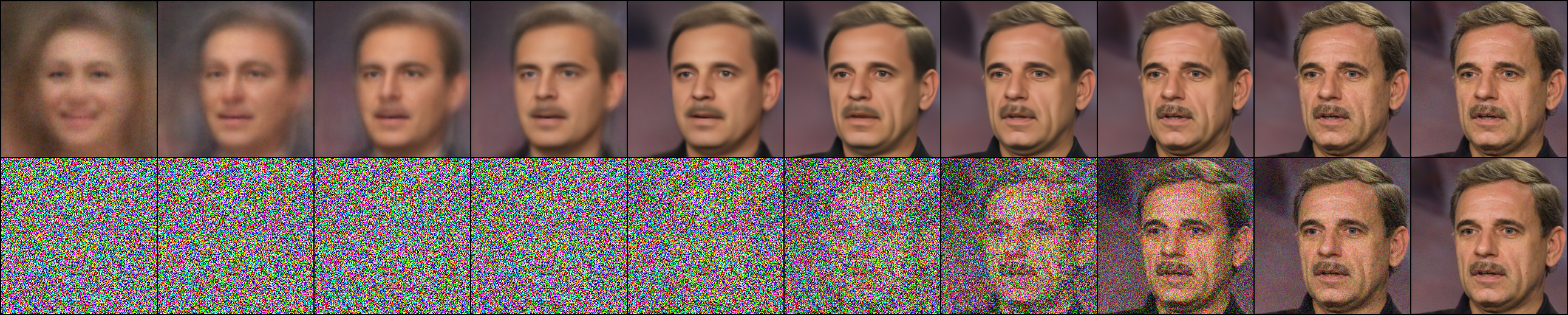}
    \caption{Sampling trajectory for greedy (midpoint) solving the HDR inverse problem. Top row is midpoint estimate and the bottom row is $\bfx_t$.}
\end{figure}

\begin{figure}[h]
    \includegraphics[width=\textwidth]{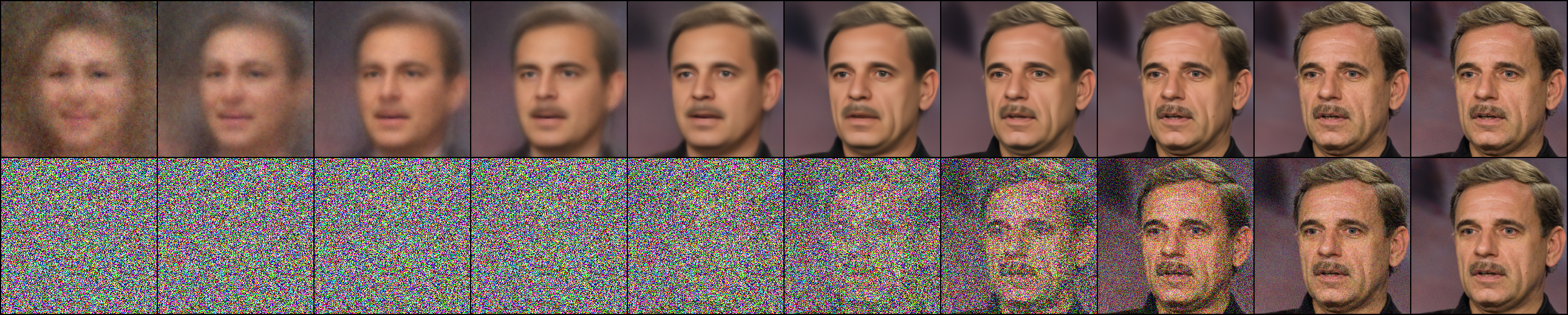}
    \caption{Sampling trajectory for greedy (2-step Euler) solving the HDR inverse problem. Top row is 2-step Euler estimate and the bottom row is $\bfx_t$.}
\end{figure}

\subsection{More qualitative samples for inverse problems}
We showcase some examples generated by the greedy gradient strategy (Euler) on the different inverse problems.

\begin{figure}[h]
    \centering
    \includegraphics[width=\textwidth]{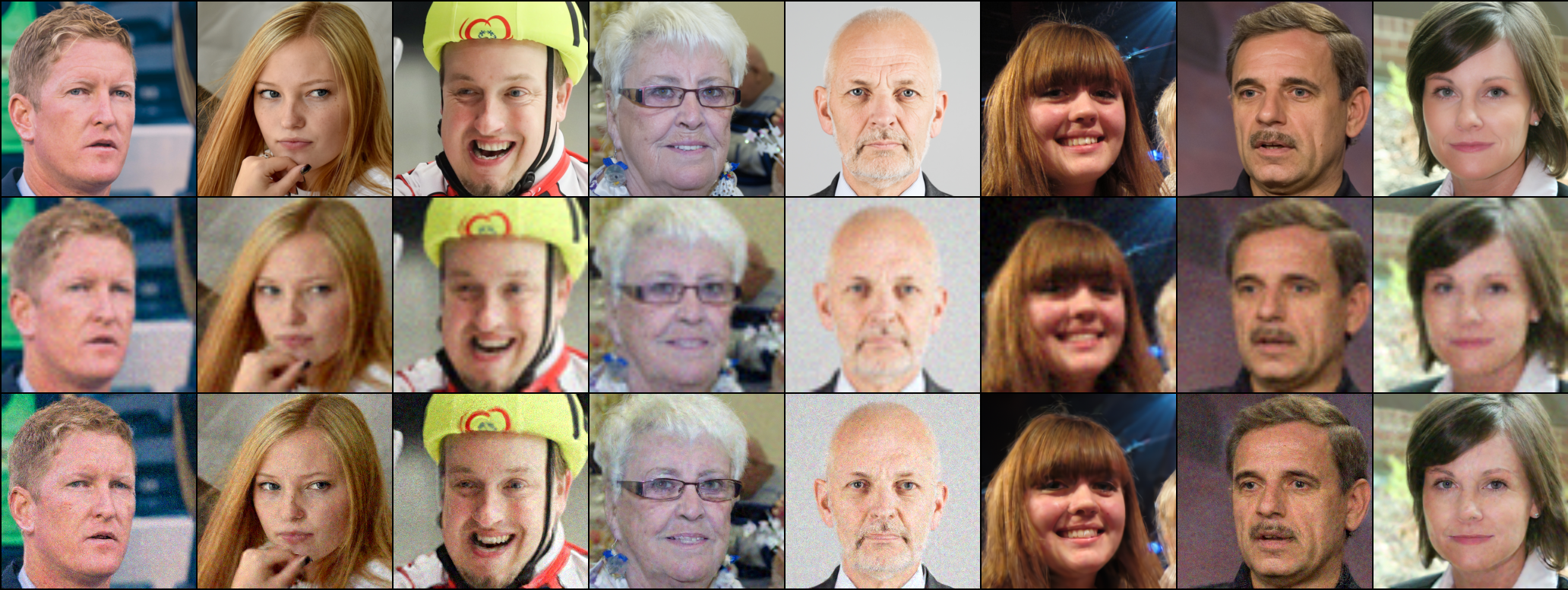}
    \caption{Qualitative visualization of using greedy guidance to solve the super resolution $4\times$ inverse problem. Top row is the ground truth, middle row is the measurement, and the bottom row is the reconstruction.}
\end{figure}

\begin{figure}[h]
    \centering
    \includegraphics[width=\textwidth]{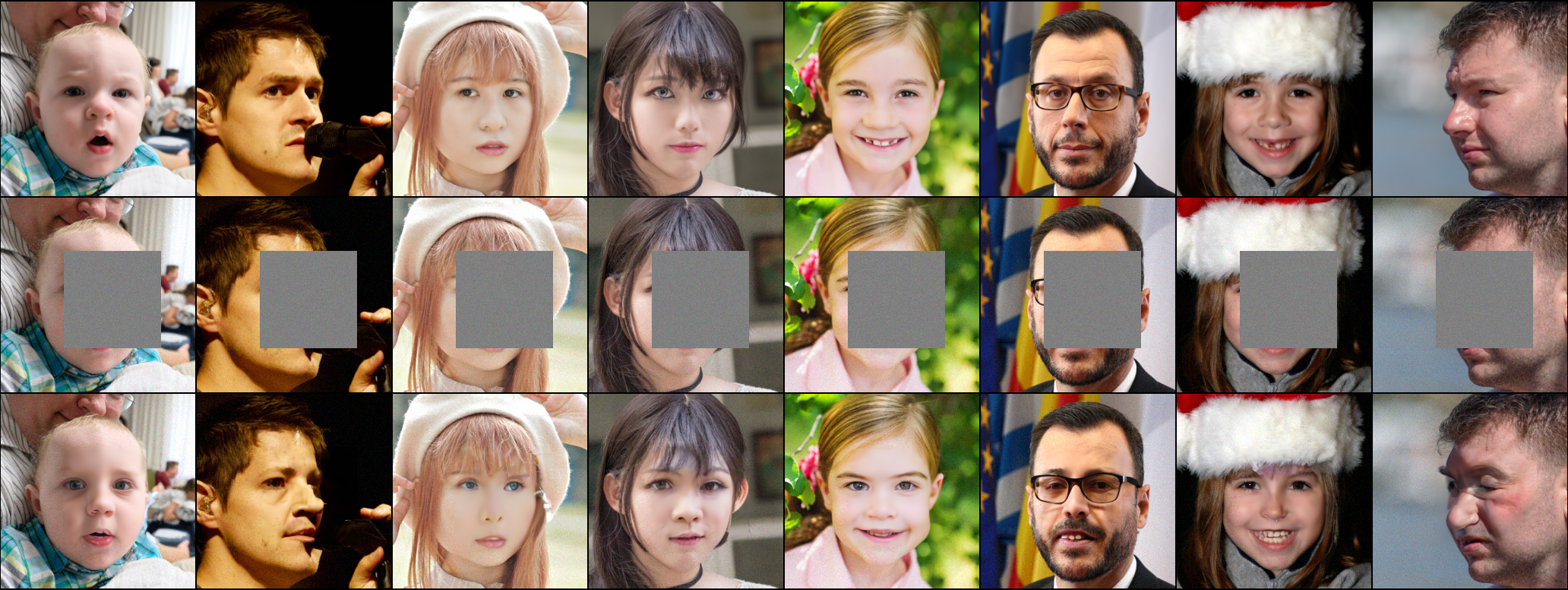}
    \caption{Qualitative visualization of using greedy guidance to solve the super resolution $4\times$ inverse problem. Top row is the ground truth, middle row is the measurement, and the bottom row is the reconstruction.}
\end{figure}

\begin{figure}[h]
    \centering
    \includegraphics[width=\textwidth]{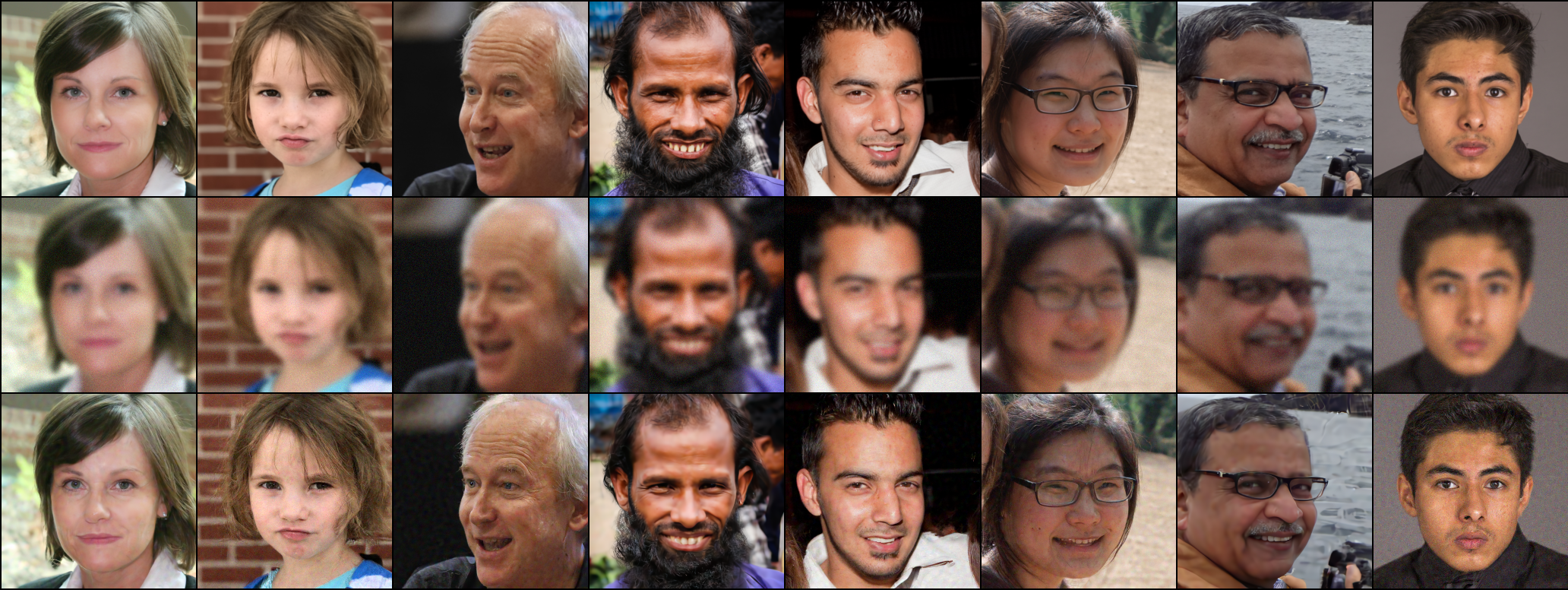}
    \caption{Qualitative visualization of using greedy guidance to solve the Gaussian deblurring inverse problem. Top row is the ground truth, middle row is the measurement, and the bottom row is the reconstruction.}
\end{figure}

\begin{figure}[h]
    \centering
    \includegraphics[width=\textwidth]{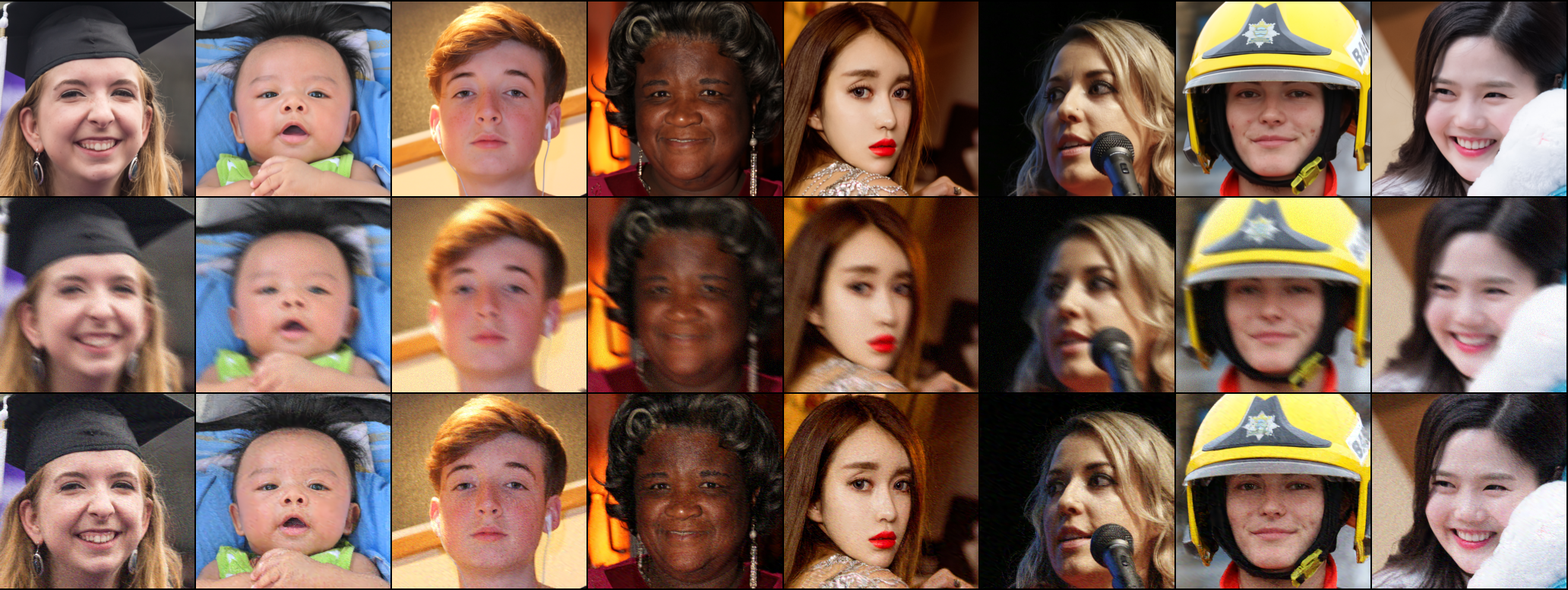}
    \caption{Qualitative visualization of using greedy guidance to solve the motion deblurring inverse problem. Top row is the ground truth, middle row is the measurement, and the bottom row is the reconstruction.}
\end{figure}

\begin{figure}[h]
    \centering
    \includegraphics[width=\textwidth]{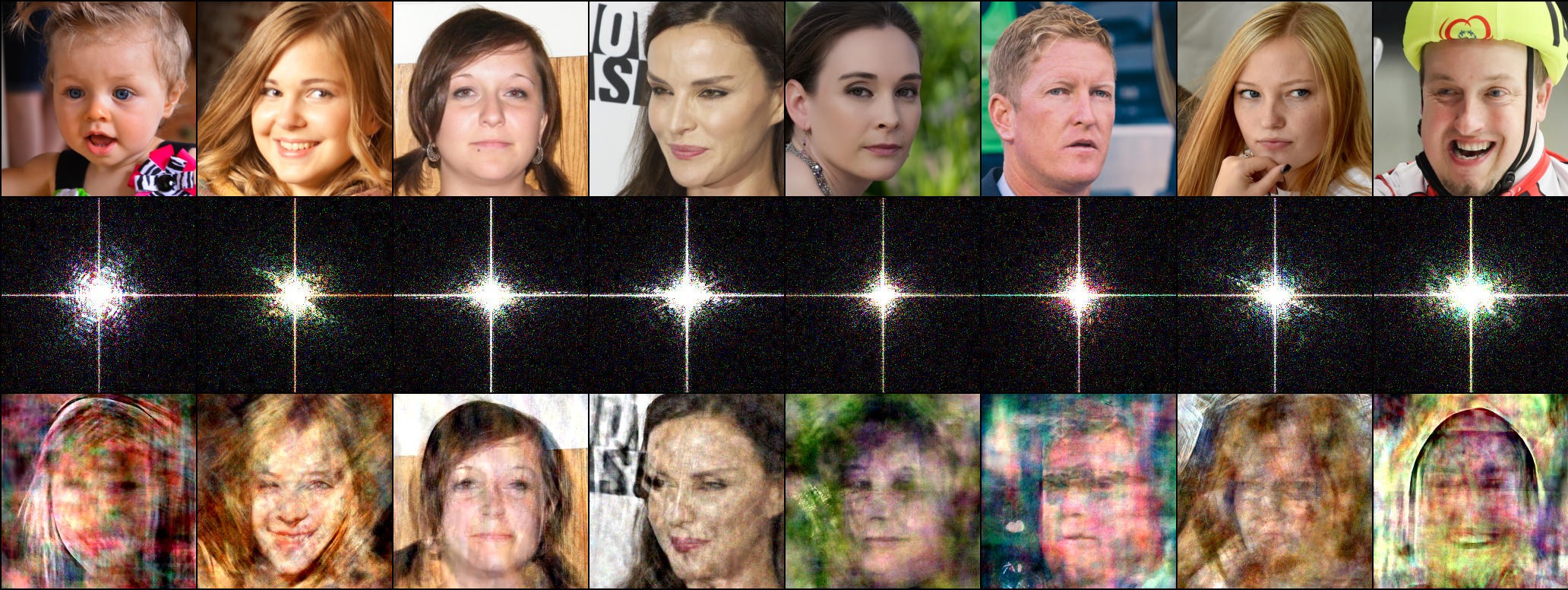}
    \caption{Qualitative visualization of using greedy guidance to solve the Phase retrieval inverse problem. Top row is the ground truth, middle row is the measurement, and the bottom row is the reconstruction.}
\end{figure}

\begin{figure}[h]
    \centering
    \includegraphics[width=\textwidth]{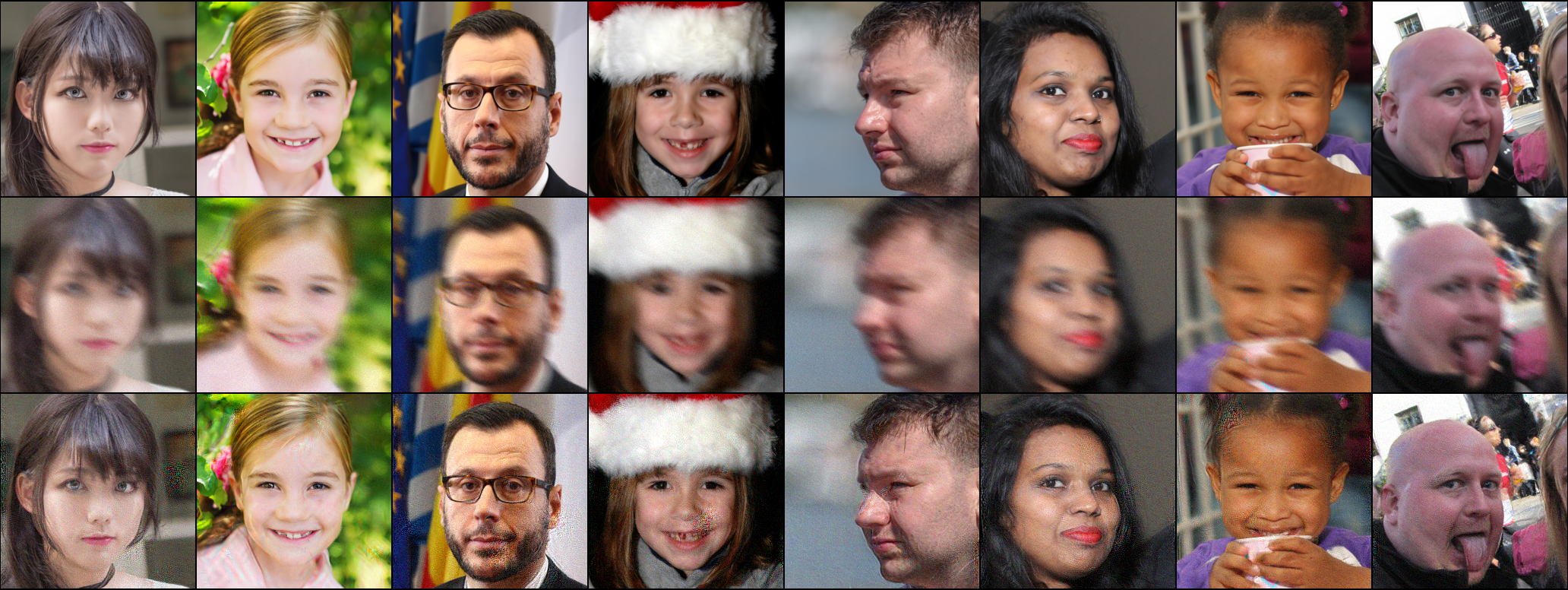}
    \caption{Qualitative visualization of using greedy guidance to solve the nonlinear deblurring inverse problem. Top row is the ground truth, middle row is the measurement, and the bottom row is the reconstruction.}
\end{figure}

\begin{figure}[h]
    \centering
    \includegraphics[width=\textwidth]{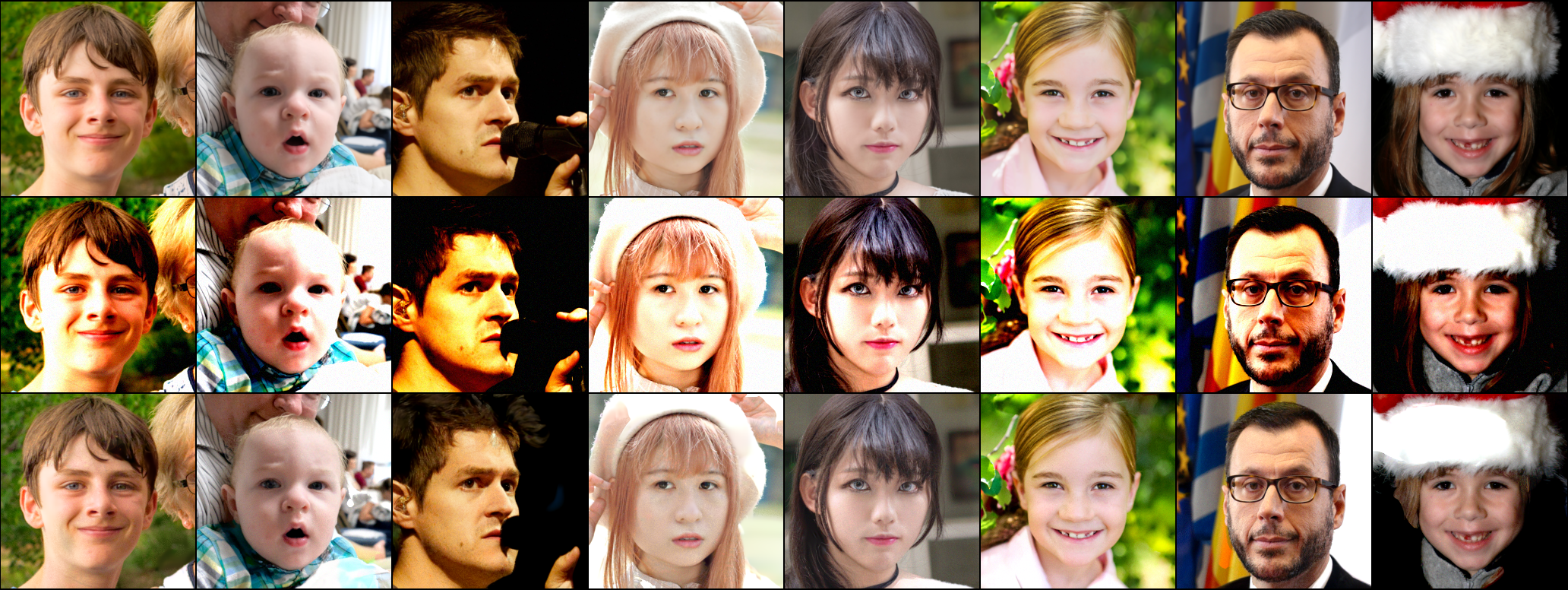}
    \caption{Qualitative visualization of using greedy guidance to solve the HDR inverse problem. Top row is the ground truth, middle row is the measurement, and the bottom row is the reconstruction.}
\end{figure}

\clearpage

\begin{figure}[t]
    \centering
    \begin{subfigure}{0.125\textwidth}
        \centering\includegraphics[width=\textwidth]{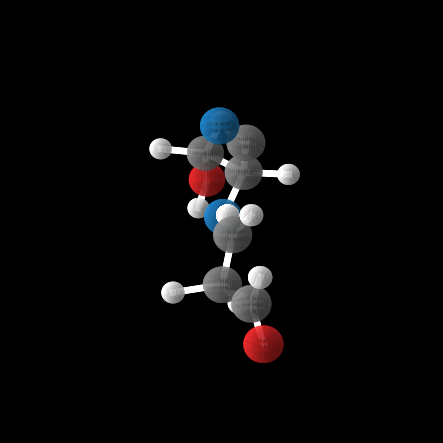}
    \end{subfigure}%
    \begin{subfigure}{0.125\textwidth}
        \centering\includegraphics[width=\textwidth]{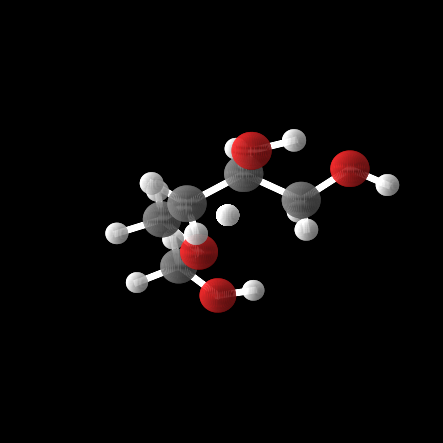}
    \end{subfigure}%
    \begin{subfigure}{0.125\textwidth}
        \centering\includegraphics[width=\textwidth]{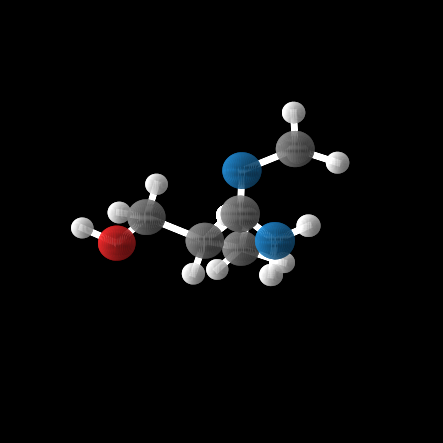}
    \end{subfigure}%
    \begin{subfigure}{0.125\textwidth}
        \centering\includegraphics[width=\textwidth]{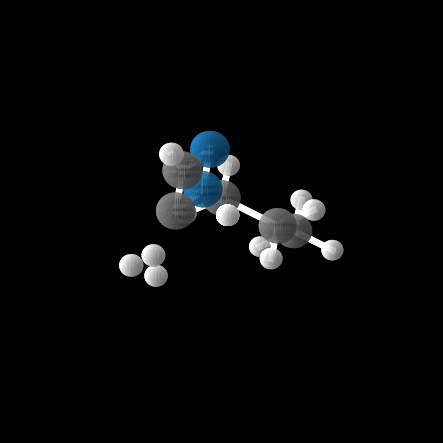}
    \end{subfigure}%
    \begin{subfigure}{0.125\textwidth}
        \centering\includegraphics[width=\textwidth]{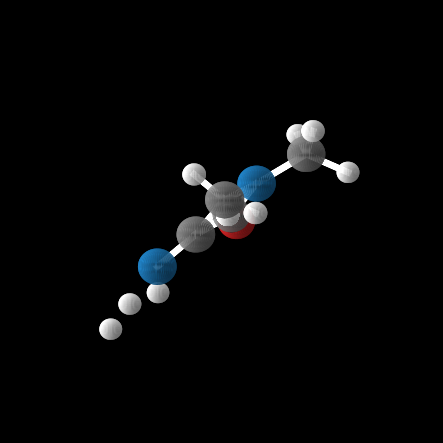}
    \end{subfigure}%
    \begin{subfigure}{0.125\textwidth}
        \centering\includegraphics[width=\textwidth]{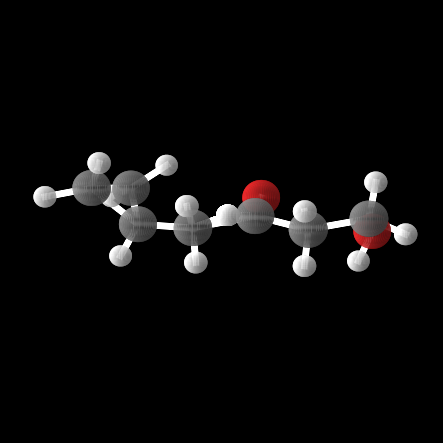}
    \end{subfigure}%
    \begin{subfigure}{0.125\textwidth}
        \centering\includegraphics[width=\textwidth]{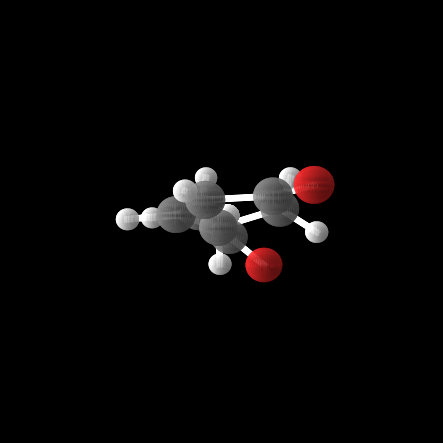}
    \end{subfigure}%
    \begin{subfigure}{0.125\textwidth}
        \centering\includegraphics[width=\textwidth]{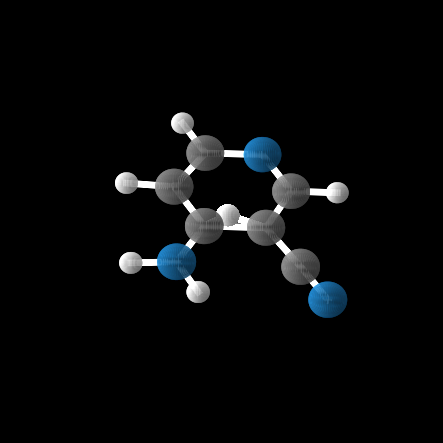}
    \end{subfigure}

    \begin{subfigure}{0.125\textwidth}
        \centering\includegraphics[width=\textwidth]{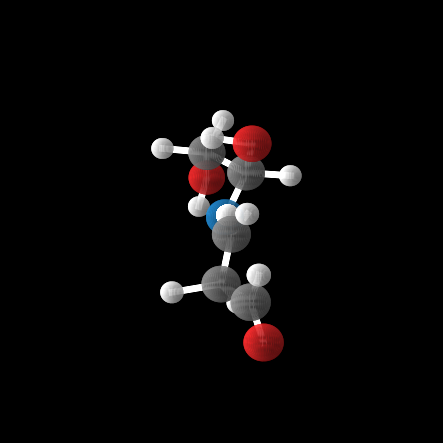}
        \caption*{0.014}
    \end{subfigure}%
    \begin{subfigure}{0.125\textwidth}
        \centering\includegraphics[width=\textwidth]{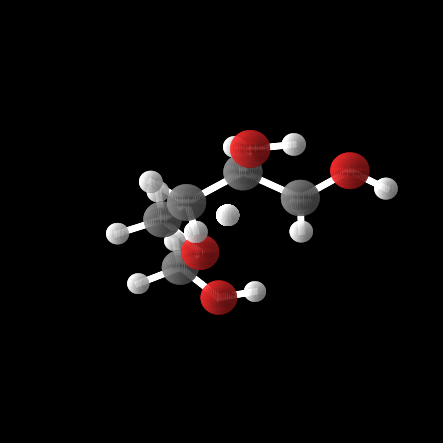}
        \caption*{0.166}
    \end{subfigure}%
    \begin{subfigure}{0.125\textwidth}
        \centering\includegraphics[width=\textwidth]{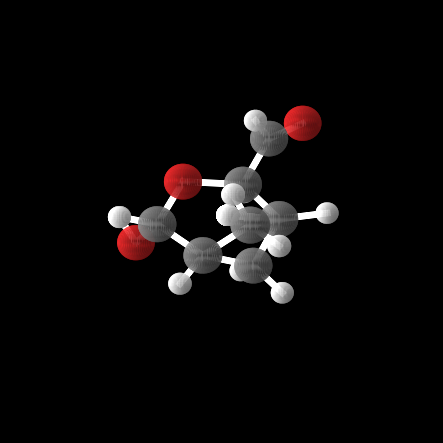}
        \caption*{1.388}
    \end{subfigure}%
    \begin{subfigure}{0.125\textwidth}
        \centering\includegraphics[width=\textwidth]{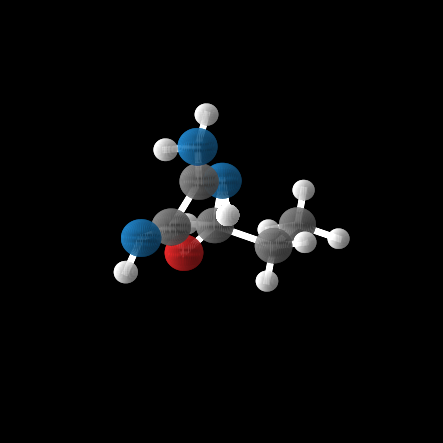}
        \caption*{2.322}
    \end{subfigure}%
    \begin{subfigure}{0.125\textwidth}
        \centering\includegraphics[width=\textwidth]{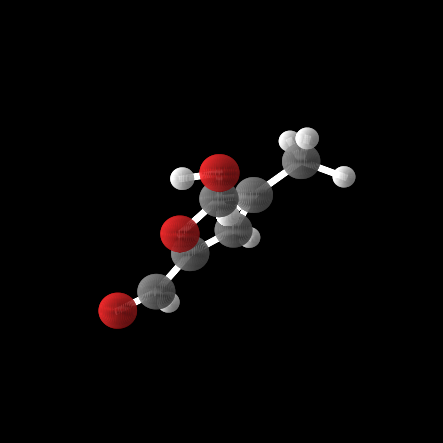}
        \caption*{2.804}
    \end{subfigure}%
    \begin{subfigure}{0.125\textwidth}
        \centering\includegraphics[width=\textwidth]{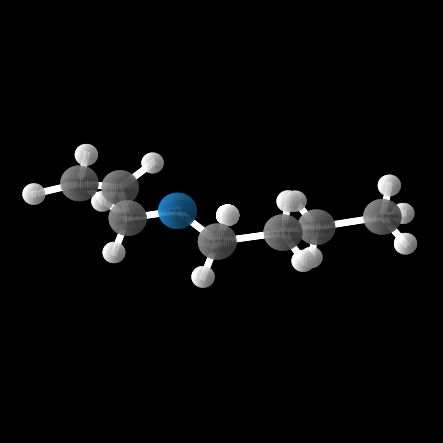}
        \caption*{3.878}
    \end{subfigure}%
    \begin{subfigure}{0.125\textwidth}
        \centering\includegraphics[width=\textwidth]{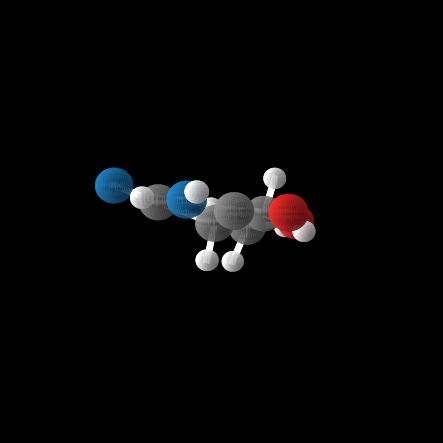}
        \caption*{4.922}
    \end{subfigure}%
    \begin{subfigure}{0.125\textwidth}
        \centering\includegraphics[width=\textwidth]{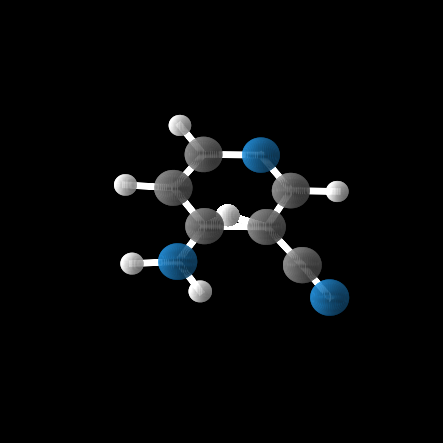}
        \caption*{7.951}
    \end{subfigure}
    \label{fig:molecules_vis2}
    \caption{Qualitative visualization of controlled generated molecules for various dipole moments $(\mu)$. Top row is generated using a end-to-end guidance with a DTO scheme and the bottom row is generated using posterior guidance.}
\end{figure}

\subsection{More qualitative samples for controlled molecule generation}
In \cref{fig:molecules_vis2} we present some qualitative results for property-guided molecule generation.
In particular, we target different dipole moments.

\section{Discussions}
\label{app:disc}

\subsection{Broader Impacts}
\label{app:broader}
Controllable generation can be used for many tasks both benign and malicious.
The insights from this paper could be used to develop more effective adversarial attacks, generation of harmful content, or other malicious applications.

\subsection{Limitations}
\label{app:limitations}
As this work is mostly theoretical, our experimental illustrations are limited, serving more to illustrate the key concepts rather than advancing the state-of-the-art within the particular problem.
We believe that future work can use these insights to make informed design choices when developing solutions to guided generation problems.

In our controllable molecule generation experiments, we take a na\"ive strategy for annealing the learning rate leaving performance on the table.
Moreover, we don't consider mixed accuracy schemes, \ie, using Euler for certain steps closer to the target and midpoint for steps further away \parencite[\cf][]{moufad2025variational}.

\end{document}